\documentclass[letterpaper]{article}
\usepackage{aaai2027}  

\usepackage[hyphens]{url}
\usepackage{graphicx}
\usepackage{natbib}
\usepackage{caption}
\usepackage{amsmath}
\usepackage{amssymb}
\usepackage{amsthm}
\usepackage{mathtools}
\usepackage{bm}

\usepackage{booktabs}
\usepackage{multirow}
\usepackage{xspace}
\usepackage{pifont}
\usepackage{enumitem}

\usepackage{algorithm}
\usepackage{algorithmic}

\usepackage{newfloat}
\usepackage{listings}
\DeclareCaptionStyle{ruled}{labelfont=normalfont,labelsep=colon,strut=off}
\floatstyle{ruled}
\newfloat{listing}{tb}{lst}{}
\floatname{listing}{Listing}

\nocopyright

\newtheorem{theorem}{Theorem}
\newtheorem{proposition}[theorem]{Proposition}
\newtheorem{lemma}[theorem]{Lemma}
\newtheorem{definition}[theorem]{Definition}
\newtheorem{assumption}[theorem]{Assumption}

\DeclareMathOperator{\KL}{KL}
\DeclareMathOperator{\softplus}{softplus}

\title{Crossing the Margin Cliff:\\
Toward Relearn-Robust LLM Unlearning via Margin Calibration}

\author{
    Xiangyu Yin\textsuperscript{1},
    Jiaxu Liu\textsuperscript{2},
    Zhen Chen\textsuperscript{3},
    Chih-Hong Cheng\textsuperscript{1,4}
}

\affiliations{
    \textsuperscript{1}Chalmers University of Technology, Gothenburg, Sweden\\
    \textsuperscript{2}Imperial College London, London, United Kingdom\\
    \textsuperscript{3}University of Liverpool, Liverpool, United Kingdom\\
    \textsuperscript{4}Carl von Ossietzky University of Oldenburg, Oldenburg, Germany
}

\begin{document}

\maketitle

\begin{abstract}
Large language model unlearning is consistently fragile under relearn attacks. On TOFU, fine-tuning on twenty forget examples substantially recovers held-out forget-set ROUGE for every method we evaluate, and we trace this fragility to optimization geometry. The per-token answer margin of fourteen post-hoc methods spanning gradient, preference, and distillation families converges into a narrow band above the retain reference in 41 of 42 method--size cells, a regularity we call the margin cliff. We prove that this cliff follows whenever the retain coupling holds the diagnostic log-odds of forget content above a floor, a condition that token-saturating losses induce at stationarity and that we verify directly on 34 of 42 cells. Margin Calibration (\textsc{MC}) is a plug-in polish adding a non-saturating margin hinge anchored at the reference's per-token margin plus a KL probe on a disjoint instruction corpus, restoring forget-side pressure where the native loss saturates. Under a stated gradient-dominance condition, whose on-trajectory gradient signature we measure by instrumenting the polish, its stationary set lies on the cliff-crossing side, yielding an attack-budget upper bound on the relearn margin lift. Across TOFU (three Llama-3 sizes, three forget tiers), MUSE-News on Llama-2-7B-hf, and a Phi-3.5 panel, a single frozen configuration wins all 14 head-to-head forget aggregates and all populated relearn cells (panel-mean post-attack ROUGE-L $0.41$ to $0.18$) and lowers raw membership AUC on 13/14, with reduced retain-side utility as the main cost. A deployment variant matches these gains without a retain-trained reference.
\end{abstract}

\section{Introduction} \label{sec:intro} Machine unlearning aims to remove the influence of a forget set $\mathcal{D}_f$ from a trained LLM without full retraining \citep{eldan2023whoshp,maini2024tofu}, motivated by legal, privacy, and deletion requirements. Under a realistic relearn-attack threat, an adversary fine-tunes the released model on a small auxiliary subset of $\mathcal{D}_f$ to recover held-out forgotten content. Across the fourteen published unlearning methods we evaluate, spanning gradient-based \citep{maini2024tofu,li2024wmdp}, preference-based \citep{zhang2024npo,fan2024simnpo}, and distillation-based losses \citep{dong2024undial,singh2025jensun}, a twenty-sample LoRA attacker (hereafter K20-LoRA) substantially restores this held-out behavior on TOFU. Recent defenses harden the canonical \textsc{NPO} objective with perturbation, curvature, noise, or invariance terms \citep{sheshadri2024lat,tamirisa2024tamper,fan2025sam,dang2025rna,wang2025ilu}, yet on our panel none eliminates the K20-LoRA vulnerability, suggesting a structural limitation of the underlying forget objective rather than an implementation issue. We characterize this failure through a per-token answer margin diagnostic, the log-probability gap between the gold and strongest alternative token at each answer's maximum-entropy position, compared against a retain-only reference $\theta_{\text{ref}}$. Across 14 methods and three Llama-3 sizes, the converged models lie in a narrow margin band above $\theta_{\text{ref}}$ in 41 of 42 cells (Figure~\ref{fig:cliff}), a positive \emph{cliff gap} $\Delta(\hat\theta)$ we call the \emph{margin cliff}, and where a method stops on this axis predicts how easily it relearns. Geometrically, once a token-saturating objective has suppressed a forget token its gradient there vanishes, so the retain regularizer fixes the stationary margin above $\theta_{\text{ref}}$'s, and the defenses above add terms without removing this saturation. To cross the cliff, the forget-side gradient structure must change. Margin Calibration (\textsc{MC}) adds a non-saturating forget term, a softplus hinge anchored at $\theta_{\text{ref}}$'s per-token margin whose gradient scales with the margin gap and so survives suppression, plus a KL probe on a disjoint instruction corpus that limits drift on non-forget behavior. \textsc{MC} runs as a short LoRA polish on any already-unlearned model, keeping its own loss running (a strict plug-in), and a deployment-phase variant instead anchors at the pre-unlearning model $\theta_0$, needing no retain-trained reference. \paragraph{Contributions.} \begin{enumerate}[leftmargin=1.5em,topsep=2pt,itemsep=2pt] \item We expose the \emph{margin cliff}, a consistent failure mode across fourteen published unlearning methods from all three loss families, and explain it as a stationarity property of retain-regularized unlearning driven by a diagnostic log-odds floor, with token saturation one mechanism inducing that floor. \item We introduce Margin Calibration (\textsc{MC}), a non-saturating margin-anchored objective with cliff-crossing conditions and an attack-budget upper bound on the post-attack forget margin that empirically predicts relearn recovery. \item Under a single hyperparameter configuration with no per-method tuning, \textsc{MC} improves post-attack robustness on every populated cell of a 97-cell cross-axis stress matrix covering 14 methods, seeds, forget tiers, three model sizes, the MUSE-News benchmark, and a Phi-3.5 panel, with a deployment-phase variant achieving comparable gains without any retain-trained reference at polishing time. \end{enumerate}

\section{Related Work} \paragraph{LLM unlearning methods.} Post-hoc LLM unlearning broadly falls into three families. Gradient-based methods directly suppress forget-set likelihood~\citep{eldan2023whoshp,maini2024tofu,li2024wmdp,entesari2025pdu}. Preference-based methods adapt RLHF-style objectives via \textsc{NPO}~\citep{zhang2024npo} and reference-free \textsc{SimNPO}~\citep{fan2024simnpo}. Distillation-based methods steer outputs toward a fixed anchor (\textsc{UNDIAL}~\citep{dong2024undial}, \textsc{JensUn}~\citep{singh2025jensun}), while representation-guided approaches have also been explored~\citep{hu2025falcon}. Our 14-method panel covers the three loss families above, which \S\ref{sec:m-kkt} places under a single cliff bound. \paragraph{Documented fragility and robustness defenses.} Post-hoc unlearning is substantially reversible under small-data relearning \citep{lynch2024eight,hu2024jogging,ren2025openunlearning}. Similar recovery-driven fragility has been observed in neighboring privacy defenses~\citep{chen2026fragile}. Recent defenses augment the unlearning loss with adversarial perturbations~\citep{sheshadri2024lat,tamirisa2024tamper}, curvature regularization~\citep{fan2025sam}, hidden-state noise~\citep{dang2025rna}, or invariance regularization~\citep{wang2025ilu}, with a separate line at pretraining time~\citep{hu2024goldfish}. These loss-augmenting defenses preserve the original saturating component, so they remain subject to cliff-side stationarity in our analysis and experiments, while \textsc{MC}'s non-saturating hinge changes the stationary set rather than re-weighting it. \paragraph{Theory and margin-based objectives.} Prior LLM unlearning theory is largely method-specific, including \textsc{NPO}'s DPO connection~\citep{zhang2024npo,rafailov2024dpo} and \textsc{RMU}'s representation-distance analysis~\citep{li2024wmdp}. We are not aware of prior work that characterizes relearning vulnerability through the stationarity (Karush--Kuhn--Tucker, KKT) structure of retain-regularized unlearning or connects post-unlearning margin geometry to an attack-budget bound on the post-attack margin. Our per-token answer margin relates to the logit margin in adversarial robustness~\citep{carlini2017towards} and decision margin in calibration~\citep{pereyra2017regularizing}. Broader robustness work has also studied logit-oriented regularization and representation-space sensitivity~\citep{yin2024loat,yin2024rerogcrl}. \textsc{MC}'s reference anchor operates on this per-token margin rather than a full output distribution, exposing the cliff and enabling a tractable KKT analysis.
\section{Method}
\label{sec:method}

We define the margin diagnostic and cliff gap (\S\ref{sec:m-setup}), document the cliff (\S\ref{sec:m-cliff}), explain it through KKT stationarity (\S\ref{sec:m-kkt}), introduce \textsc{MC} as a non-saturating correction (\S\ref{sec:m-v3}), and bound the attacker's margin lift (\S\ref{sec:m-thm2}).

\subsection{Problem Setup and Notation}
\label{sec:m-setup}

\paragraph{Unlearning task.}
Let $\theta_0$ be the target model, fine-tuned from an instruction-tuned LLM on $\mathcal{D}_f \cup \mathcal{D}_r$ (disjoint forget / retain corpora of $(x, y)$ pairs, $y = (y_1, \dots, y_T)$). Following TOFU~\citep{maini2024tofu}, $\theta_0$ bounds retain-side utility and the retain reference $\theta_{\text{ref}}$ (trained on $\mathcal{D}_r$ alone) is the forget-side oracle. An unlearning algorithm maps $\theta_0$ to $\hat\theta$ in time $\ll$ retraining, and we evaluate \emph{forget strength} (closeness on $\mathcal{D}_f$ to $\theta_{\text{ref}}$), \emph{utility} (closeness on $\mathcal{D}_r$ to $\theta_0$), and \emph{robustness} (survival of both under a relearn attacker), our focus being the third, the one axis that routinely fails.

\begin{figure*}[t]
  \centering
  \includegraphics[width=\textwidth]{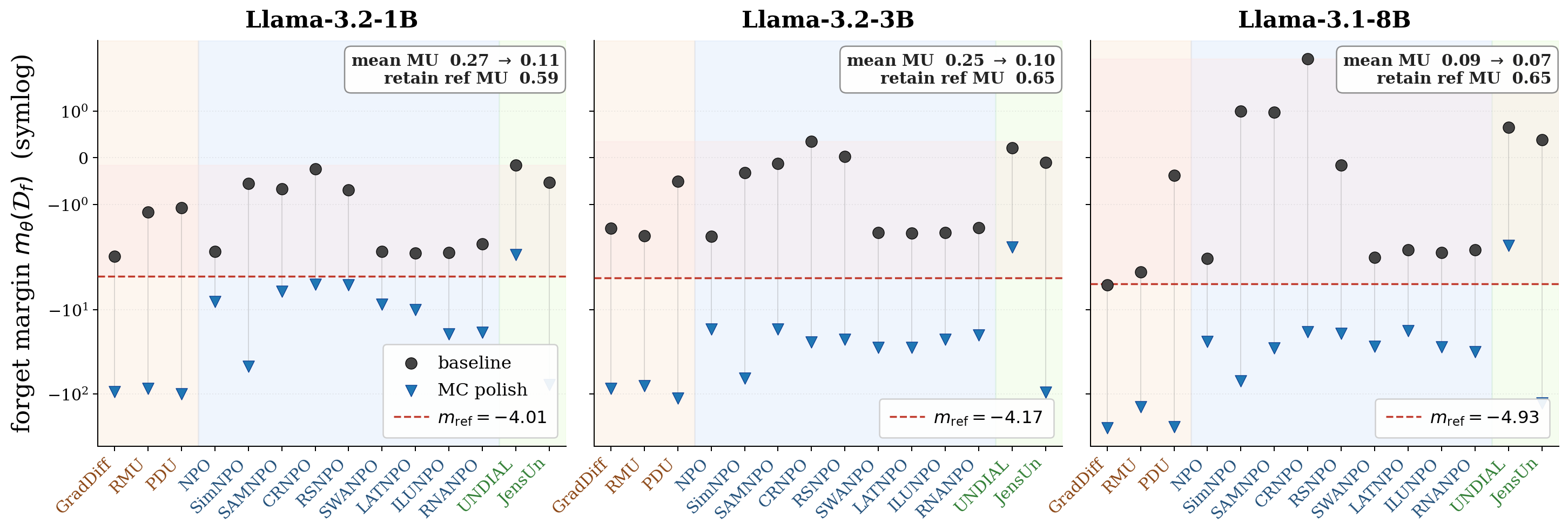}
  \caption{Forget-set margin diagnostic $m_\theta(\mathcal{D}_f)$ (Eq.~\eqref{eq:margin-agg}) for 14 baselines (gray) and \textsc{MC} (blue) at three Llama-3 sizes, ordered by loss family. Red dashed line is $m_{\mathrm{ref}}$, shaded band is the cliff gap. Mean TOFU MU annotated.}
  \label{fig:cliff}
\end{figure*}

\paragraph{Per-token answer margin.}
For a model $\theta$ on prompt $x$ and answer prefix $y_{<t}$, the per-token \emph{answer margin} of the gold token $y_t$ is
\begin{equation}
\label{eq:margin-def}
\small
m_\theta(t; x, y)  :=  \log p_\theta(y_t \mid x, y_{<t}) - \log \max_{v \neq y_t} p_\theta(v \mid x, y_{<t}),
\end{equation}
the log-probability gap between gold and its strongest competitor, positive iff gold is the argmax and unbounded below as a wrong token gains mass. A uniform average over answer positions is dominated by easy continuation tokens on which even $\theta_{\text{ref}}$ is confident (App.~\ref{app:logodds-verify}), so we aggregate at each sample's maximum-entropy answer position, where competitor mass is largest and the margin is least saturated. This is a proxy for the answer's most fragile position, not a claim about where content resides, validated operationally by its predictive power for post-attack recovery (App.~\ref{app:bridge}). Let $\hat t_\theta(x, y) := \arg\max_t \mathcal{H}\!\left(p_\theta(\cdot \mid x, y_{<t})\right)$ be the maximum-entropy answer position ($\mathcal{H}$ the Shannon entropy). The dataset diagnostic is
\begin{equation}
\label{eq:margin-agg}
m_\theta(\mathcal{D})  :=  \mathbb{E}_{(x,y)\sim\mathcal{D}} \left[ m_\theta(\hat t_\theta; x, y)\right],
\end{equation}
and we write $m_{\text{ref}} := m_{\theta_{\text{ref}}}$. (The rule $\hat t_\theta$ enters only this diagnostic, never the training loss or the operational metrics.) We define the \emph{cliff gap} of an unlearned model $\hat\theta$ as
\begin{equation}
\label{eq:cliff-gap}
\Delta(\hat\theta)  :=  m_{\hat\theta}(\mathcal{D}_f) - m_{\text{ref}}(\mathcal{D}_f).
\end{equation}
$\Delta \le 0$ means $\hat\theta$ has pressed the forget margin down to the retain reference's level and crossed the cliff.

Unlike NLL, the margin records whether gold remains the argmax, distinguishing diffuse uncertainty from confident wrong predictions, and is therefore our comparison axis and the anchor for a non-saturating forget loss (\S\ref{sec:m-v3}). We write $\tilde\theta = \mathcal{R}(\hat\theta, K)$ for the post-attack model with forget-sample budget $K$ (and $\tilde\theta_N$ when parametrized by attacker steps $N$), and use $\mathcal{D}_a$ for an auxiliary instruction corpus (Alpaca) disjoint from $\mathcal{D}_f \cup \mathcal{D}_r$.

\subsection{The Margin Cliff Observation}
\label{sec:m-cliff}

We measure $m_{\hat\theta}(\mathcal{D}_f)$ on TOFU forget10 (Llama-3.2-1B-Instruct) for fourteen methods spanning the gradient, preference, and distillation families (full list in \S\ref{sec:exp-setup}). Gradient ascent is treated separately (App.~\ref{app:gradascent}). After convergence, the methods stop in a narrow band while $\theta_{\text{ref}}$ sits well below it, with $\Delta(\hat\theta_{\text{base}}) > 0$ in 41 of 42 (method, size) cells (Figure~\ref{fig:cliff}, per-tier confirmation in App.~\ref{app:cliff-figure}). The single sub-reference cell, GradDiff at 8B ($\Delta \approx -0.5$), crosses only degenerately with collapsed general utility (App.~\ref{app:gradascent}), by destruction rather than calibrated forgetting. Where each method stops on this axis predicts relearnability (Fig.~\ref{fig:kcurve}, Fig.~\ref{fig:attack_bars}). The structural question is why every non-degenerate cell stops on the same side.

\subsection{A KKT Account of the Cliff}
\label{sec:m-kkt}

The preference family and JensUn share a structural property, in that their per-token forget gradient vanishes as that token's probability is suppressed (Def.~\ref{def:forget-saturating}). The gradient family (GradDiff, RMU, PDU) and the flattened-teacher distillation loss UNDIAL have bounded but non-vanishing forget gradients. For them the saturation mechanism does not apply, and the same cliff bound follows by assuming the diagnostic floor of Asm.~\ref{asm:overlap} directly (App.~\ref{app:loss-saturation}). All cases reduce to first-order stationarity of the joint unlearning problem with forget loss $\mathcal{L}_f$, retain NLL $\mathcal{L}_r$, and tradeoff weight $\lambda > 0$,
\begin{equation}
\label{eq:standard-unlearn}
\min_\theta    \mathcal{L}_f(\theta;\mathcal{D}_f) + \lambda \mathcal{L}_r(\theta;\mathcal{D}_r).
\end{equation}
Problem~\eqref{eq:standard-unlearn} is the Lagrangian of constrained forgetting, so its stationary points are KKT points of that problem, and we use the two views interchangeably.

\begin{definition}[Token-saturating forget loss]
\label{def:forget-saturating}
Write $\mathcal{L}_f(\theta;\mathcal{D}_f) = \mathbb{E}_{(x,y)}\left[\ell_f(\theta; x, y)\right]$ with
$\ell_f$ the per-sample loss. $\mathcal{L}_f$ is \emph{token-saturating} if there is a nondecreasing
$\rho:[0,1]\to\mathbb{R}_{\ge 0}$ with $\rho(p)\to 0$ as $p\to 0^+$ such that, for every answer token
of every sample,
\[
\left|\frac{\partial \ell_f}{\partial z_{y_t}}\right|
 \le
\rho\left(p_\theta(y_t \mid x, y_{<t})\right),
\]
where $z_{y_t}$ is the gold-token logit at position $t$.
The loss's marginal incentive to suppress a gold token further thus vanishes as that token's
probability is driven toward zero, \emph{token by token}, regardless of how confident the model
remains on other (easy) answer tokens. (No per-token decomposition of $\ell_f$ is required.)
\end{definition}

\paragraph{Examples.}
NPO satisfies Def.~\ref{def:forget-saturating} on the unlearning regime observed at every checkpoint we probe, with $\rho(p) \propto p^\beta$, and JensUn satisfies it unconditionally. UNDIAL does \emph{not}, its cross-entropy toward a flattened teacher keeping an $O(1)$ pull (per-loss derivations in App.~\ref{app:loss-saturation}).

Assumption~\ref{asm:overlap} (formal in App.~\ref{app:formal}) quantifies prompt-level coupling between $\mathcal{D}_r$ and $\mathcal{D}_f$ by a scalar $\epsilon \in [0, 1]$, where $\epsilon=0$ means disjoint prompt n-grams and $\epsilon=1$ means complete overlap, and states that at stationarity of \eqref{eq:standard-unlearn} the retain regularizer holds the \emph{diagnostic log-odds level} $\ell(\theta) := \mathbb{E}_{\mathcal{D}_f}[\log\frac{p_{\hat t}}{1 - p_{\hat t}}]$ (average gold log-odds at the diagnostic positions of Eq.~\eqref{eq:margin-agg}) above a floor $\ell_\star$ that is non-decreasing in $\epsilon$ and in $\lambda$. On TOFU forget10 we measure $\epsilon \approx 0.44$ (word-level bigram coverage of forget Q\&A by retain Q\&A), and $\ell(\theta)$ is measured directly per checkpoint (App.~\ref{app:logodds-verify}).

\begin{theorem}[Margin cliff]
\label{thm:kkt-cliff}
Suppose that, with $\epsilon > 0$ and $\lambda > 0$, the retain regularizer enforces the diagnostic
log-odds floor $\ell_\star$ of Asm.~\ref{asm:overlap} at a strict local minimum $\theta^\star$ of
\eqref{eq:standard-unlearn}. Then
\[
\Delta(\theta^\star)  \ge  \ell(\theta^\star) - m_{\text{ref}}(\mathcal{D}_f)  \ge  \ell_\star - m_{\text{ref}}(\mathcal{D}_f)  =:  C ,
\]
and $C > 0$ whenever $\ell_\star > m_{\text{ref}}(\mathcal{D}_f)$. The bound uses only the floor and an
unconditional competitor-mass inequality, with no property of $\mathcal{L}_f$ beyond stationarity
entering its proof (App.~\ref{app:proofs}).
\end{theorem}

\begin{proposition}[When the floor holds]
\label{prop:floor-mechanism}
The floor of Asm.~\ref{asm:overlap} is a property of the loss at stationarity, not a blanket assumption.
For a token-saturating loss (Def.~\ref{def:forget-saturating}) with $\epsilon > 0$ and $\lambda > 0$, a
first-order stationarity argument (App.~\ref{app:proofs}, Steps~1--2) shows the retain regularizer
holds $\ell(\theta^\star)$ up while the opposing forget push decays like $\rho(p_t)$, motivating the
floor mechanistically for that class. For the bounded-gradient family (GradDiff, RMU, PDU) and the
anchor-saturating UNDIAL, no such derivation is available and the floor is assumed directly. Token
saturation is thus a sufficient mechanism for the floor, not a hypothesis of
Theorem~\ref{thm:kkt-cliff}.
\end{proposition}

The proof is short (App.~\ref{app:proofs}). Every per-token margin is bounded below by its gold log-odds, so averaging at the diagnostic positions gives
$m_{\hat\theta}(\mathcal{D}_f) \ge \ell(\theta^\star)$, which the floor lifts above the (more negative)
retain reference. Because the bound passes entirely through the measurable
$\ell(\theta^\star)$, the measured value certifies the conclusion directly, with $\ell(\theta^\star) > m_{\text{ref}}$ and hence $\Delta(\theta^\star) > 0$ holding with no unverified assumption
on $34$ of $42$ (method, size) cells. This certification route also needs no idealized convergence,
since it evaluates the measured $\ell$ at the released checkpoint whether or not training reached an
exact stationary point. In the remaining cells the log-odds relaxation is too loose to
certify, and the measured $\Delta$ itself stays positive in all of them except the degenerate
GradDiff-8B cell of \S\ref{sec:m-cliff} (App.~\ref{app:logodds-verify}).

\paragraph{Longer training is not a substitute.}
The floor is a property of stationarity, not of budget. Further epochs of the same objective leave the
stationary set, and hence the terminal margin, unchanged, which is why the $14$ baselines land in the
same narrow band despite heterogeneous objectives. Descending past the floor requires changing the
forget-side gradient structure, which is what \textsc{MC} does.

\begin{table*}[t]
  \centering
  \small
  \setlength{\tabcolsep}{4pt}
  \caption{\textbf{14-method head-to-head on TOFU \texttt{forget10}} (Llama-3.2-1B). Each cell shows \emph{baseline} $\to$ \textbf{\textsc{MC}}, \textbf{bold} marks \textsc{MC} improvement. F.agg averages four forget metrics; MIA.agg averages five raw AUCs (per-detector distance-from-chance advantage improves $6/14$ due to overshoot past chance, App.~\ref{app:mia-breakdown}). KS-$p$ is the \textsc{MC} checkpoint's forget-quality KS $p$-value (single column; bold when $\geq 0.05$). Baseline columns: official evaluator; \textsc{MC} columns: pipeline evaluator, so the tabulated MU drop overstates the cost (panel mean $0.27 \to 0.11$ under a single evaluator, \S\ref{sec:exp-setup}, App.~\ref{app:cliff-numbers}).}
  \label{tab:t1-main}
  \resizebox{\textwidth}{!}{%
  \begin{tabular}{ll cccc cc}
    \toprule
    Base & Family & F.agg ($\downarrow$) & MIA.agg ($\downarrow$) & MU ($\uparrow$) & KS-$p$ & K20-LoRA ($\downarrow$) & K20-FPFT ($\downarrow$) \\
    \midrule
    \texttt{GradDiff} & GradDiff & 0.357$\!\to\!$\textbf{0.018} & 0.791$\!\to\!$\textbf{0.228} & 0.470$\!\to\!$0.109 & 0.0000 & 0.429$\!\to\!$\textbf{0.158} & 0.408$\!\to\!$\textbf{0.088} \\
    \texttt{RMU} & GradDiff & 0.279$\!\to\!$\textbf{0.002} & 0.394$\!\to\!$0.466 & 0.571$\!\to\!$0.000 & 0.0021 & 0.471$\!\to\!$\textbf{0.029} & 0.398$\!\to\!$\textbf{0.031} \\
    \texttt{PDU} & GradDiff & 0.024$\!\to\!$\textbf{0.009} & 0.472$\!\to\!$\textbf{0.226} & 0.000$\!\to\!$\textbf{0.127} & 0.0000 & 0.428$\!\to\!$\textbf{0.152} & 0.437$\!\to\!$\textbf{0.027} \\
    \midrule
    \texttt{SimNPO} & SimNPO & 0.533$\!\to\!$\textbf{0.010} & 0.889$\!\to\!$\textbf{0.129} & 0.594$\!\to\!$0.179 & 0.0000 & 0.458$\!\to\!$\textbf{0.025} & 0.450$\!\to\!$\textbf{0.026} \\
    \midrule
    \texttt{NPO} & NPO & 0.296$\!\to\!$\textbf{0.041} & 0.716$\!\to\!$\textbf{0.165} & 0.300$\!\to\!$0.211 & 0.0000 & 0.348$\!\to\!$\textbf{0.246} & 0.342$\!\to\!$\textbf{0.246} \\
    \texttt{SAMNPO} & NPO & 0.762$\!\to\!$\textbf{0.104} & 0.980$\!\to\!$\textbf{0.241} & 0.599$\!\to\!$0.031 & 0.0085 & 0.414$\!\to\!$\textbf{0.201} & 0.401$\!\to\!$\textbf{0.146} \\
    \texttt{ILUNPO} & NPO & 0.300$\!\to\!$\textbf{0.027} & 0.723$\!\to\!$\textbf{0.106} & 0.337$\!\to\!$0.054 & 0.0001 & 0.343$\!\to\!$\textbf{0.205} & 0.335$\!\to\!$\textbf{0.122} \\
    \texttt{LATNPO} & NPO & 0.298$\!\to\!$\textbf{0.043} & 0.723$\!\to\!$\textbf{0.160} & 0.336$\!\to\!$0.162 & 0.0010 & 0.342$\!\to\!$\textbf{0.247} & 0.337$\!\to\!$\textbf{0.219} \\
    \texttt{RNANPO} & NPO & 0.277$\!\to\!$\textbf{0.005} & 0.735$\!\to\!$\textbf{0.114} & 0.323$\!\to\!$0.150 & 0.0000 & 0.366$\!\to\!$\textbf{0.165} & 0.357$\!\to\!$\textbf{0.041} \\
    \texttt{RSNPO} & NPO & 0.528$\!\to\!$\textbf{0.114} & 0.948$\!\to\!$\textbf{0.254} & 0.500$\!\to\!$0.070 & \textbf{0.2705} & 0.442$\!\to\!$\textbf{0.195} & 0.427$\!\to\!$\textbf{0.154} \\
    \texttt{SWANPO} & NPO & 0.300$\!\to\!$\textbf{0.052} & 0.724$\!\to\!$\textbf{0.268} & 0.338$\!\to\!$0.110 & 0.0006 & 0.349$\!\to\!$\textbf{0.253} & 0.339$\!\to\!$\textbf{0.145} \\
    \texttt{CRNPO} & NPO & 0.890$\!\to\!$\textbf{0.026} & 0.998$\!\to\!$\textbf{0.228} & 0.592$\!\to\!$0.040 & \textbf{0.7934} & 0.579$\!\to\!$\textbf{0.206} & 0.572$\!\to\!$\textbf{0.181} \\
    \midrule
    \texttt{UNDIAL} & Distil & 0.505$\!\to\!$\textbf{0.283} & 0.938$\!\to\!$\textbf{0.651} & 0.574$\!\to\!$0.207 & 0.0085 & 0.381$\!\to\!$\textbf{0.282} & 0.390$\!\to\!$\textbf{0.284} \\
    \texttt{JensUn} & Distil & 0.708$\!\to\!$\textbf{0.039} & 0.995$\!\to\!$\textbf{0.210} & 0.582$\!\to\!$0.090 & 0.0043 & 0.447$\!\to\!$\textbf{0.147} & 0.423$\!\to\!$\textbf{0.079} \\
    \midrule
    \textbf{\textsc{MC} wins (of 14)} & & \textbf{14} & \textbf{13}\,(raw) & 1 & \textbf{2}\,(KS$\geq 0.05$) & \textbf{14} & \textbf{14} \\
    \bottomrule
  \end{tabular}%
  }
\end{table*}

\subsection{Margin Calibration}
\label{sec:m-v3}

Theorem~\ref{thm:kkt-cliff} implies a cliff-crossing loss must keep forget-side pressure after probability suppression. Gradient ascent does, but diverges, with margins running to $-\infty$ and utility collapsing within tens of steps (\S\ref{sec:exp-ablations}). \textsc{MC}, a LoRA polish with a single fixed configuration reused across all sweeps (\S\ref{sec:exp-setup}), anchors this pressure at a per-token margin target via a one-sided softplus hinge.

\paragraph{Anchors.}
$\theta_{\text{mar}}$ is the \emph{margin anchor} defining the per-token target $m_{\text{mar}}(t) := m_{\theta_{\text{mar}}}(t; x, y)$ on $\mathcal{D}_f$, and $\theta_{\text{anc}}$ is the \emph{behavior anchor} used by the KL probe. The reference-anchored variant (main results) sets $\theta_{\text{mar}} = \theta_{\text{anc}} = \theta_{\text{ref}}$. The deployment-phase variant (\S\ref{sec:exp-deploy}) applies when no retain-trained reference exists, taking the pre-unlearning target model as margin anchor, $\theta_{\text{mar}} = \theta_0$. Under this anchor the forget hinge is largely inactive, so the variant replaces the KL probe by a retain-side hinge at $\theta_0$'s retain margins ($\lambda_r = 1.0$, App.~\ref{app:deploy}), with all other settings unchanged. The reported cliff gap $\Delta$ is always evaluated against $m_{\text{ref}}$ regardless of training anchor.

\paragraph{Forget hinge.}
With $m_{\text{mar}}(t)$ computed under \texttt{no\_grad} through the frozen $\theta_{\text{mar}}$, the forget hinge is
\begin{align}
&\mathcal{L}_{\text{forget}}(\theta) = \label{eq:v3-forget}
\\ 
&\mathbb{E}_{\mathcal{D}_f} \left[
\tfrac{1}{T} \sum_{t=1}^T
\tfrac{1}{\kappa}\softplus\left(\kappa(m_\theta(t) - m_{\text{mar}}(t))\right)
\right]. \nonumber
\end{align}
The derivative $\sigma(\kappa(m_\theta - m_{\text{mar}}))$ scales with the margin gap rather than $p_\theta$, so it does not vanish under cliff suppression, and one-sidedness lets the stationary set descend past $m_{\text{mar}}$ rather than lock at it. The gap weighting also couples the hinge to the diagnostic of Eq.~\eqref{eq:margin-agg}, concentrating its activation near the same maximum-entropy positions (measured at Spearman $+0.4$ to $+0.5$ for saturating bases, App.~\ref{app:proofs}, with explicit reweighting and a reference-confidence gate ablated in App.~\ref{app:hinge-ew} and App.~\ref{app:refgate}).

\paragraph{KL probe.}
Utility is anchored on an instruction corpus $\mathcal{D}_a$ (Alpaca) disjoint from $\mathcal{D}_f \cup \mathcal{D}_r$, over the answer positions of each probe example,
\begin{equation}
\small
\label{eq:v3-kl}
\mathcal{L}_{\text{KL}}(\theta)
 =
\mathbb{E}_{x \sim \mathcal{D}_a} \left[
\KL\left(p_{\theta_{\text{anc}}}(\cdot \mid x) \big\| p_\theta(\cdot \mid x)\right)
\right],
\end{equation}
avoiding the gradient cancellation a retain hinge over $\mathcal{D}_r$ would induce on entity tokens shared with $\mathcal{D}_f$. The forward direction is mass-covering, penalizing removal of probability mass the anchor assigns (the failure mode of over-aggressive polishing), and keeps the trajectory close to the anchor of Lemma~\ref{lem:kl-bounded} (App.~\ref{app:formal}).

\paragraph{Full objective.}
\textsc{MC} is a \emph{strict plug-in}. The polish keeps the base method's own loss $\mathcal{L}_{\text{native}}$ running (its forget and retain terms, though training-time perturbation schedules of complex variants are not replicated, App.~\ref{app:loss-saturation}) and adds the \textsc{MC} terms on top,
\begin{equation}
\label{eq:v3-objective}
\boxed{
\mathcal{L}_{\text{polish}}(\theta)
 =
\mathcal{L}_{\text{native}}(\theta) + \mathcal{L}_{\text{forget}}(\theta)  +  \lambda_{\text{KL}} \mathcal{L}_{\text{KL}}(\theta).
 }
\end{equation}
Retain data thus enters only through the native retain term the base method already used, and the \textsc{MC} terms themselves touch only $\mathcal{D}_f$ and $\mathcal{D}_a$. Because the polish initializes at $\hat\theta_{\text{base}}$, the native gradient is small at the start and enters Theorem~\ref{thm:v3-cliff-cross} only through the constant $G_{\text{nat}}$.

\subsection{\textsc{MC} Crosses the Cliff with Bounded Attacker Lift}
\label{sec:m-thm2}

Theorem~\ref{thm:v3-cliff-cross} rules out cliff-side stationarity for reference-anchored \textsc{MC} ($\theta_{\text{mar}} = \theta_{\text{anc}} = \theta_{\text{ref}}$). The deployment variant lies outside this mechanism, its cliff crossing established empirically (Table~\ref{tab:t4-deploy}, App.~\ref{app:deploy}) rather than by stationarity. Theorem~\ref{thm:relearn-bound} then bounds the margin lift any fine-tuning attacker can produce, so a sufficiently negative cliff gap of $\hat\theta_{\text{MC}}$ survives the attack for every attacker in the class.

\begin{theorem}[\textsc{MC} removes positive cliff-stationarity]
\label{thm:v3-cliff-cross}
Assume the KL-probe gradient is uniformly bounded by $\|\nabla_\theta \mathcal{L}_{\text{KL}}(\theta)\| \le G_{\text{KL}}^{\max}$ for $\theta$ in a bounded neighborhood of $\theta_{\text{anc}}$ on Alpaca inputs disjoint from $\mathcal{D}_f \cup \mathcal{D}_r$ (Lemma~\ref{lem:kl-bounded}), and that the native-loss gradient is bounded by $\|\nabla_\theta \mathcal{L}_{\text{native}}(\theta)\| \le G_{\text{nat}}$ on the same region (App.~\ref{app:loss-saturation}). Assume further that the forget hinge satisfies a \emph{directional margin coercivity} condition, requiring that whenever $\theta \in \Theta_0$ (the compact polish region of Lemma~\ref{lem:kl-bounded}) and $\Delta(\theta) \ge \delta$, there is a unit vector $u_\theta$ with $u_\theta^\top \xi \ge G_f(\delta) > 0$ for every $\xi \in \partial \mathcal{L}_{\text{forget}}(\theta)$, the Clarke subdifferential (App.~\ref{app:proofs}). For any $\delta$ with $\lambda_{\text{KL}} G_{\text{KL}}^{\max} + G_{\text{nat}} < G_f(\delta)$, every strict local minimum of $\mathcal{L}_{\text{polish}}$ (in the unconstrained sense) that lies in $\Theta_0$ satisfies $\Delta(\theta^\star) < \delta$.
\end{theorem}

\begin{figure}[t]
  \centering
  \includegraphics[width=\columnwidth]{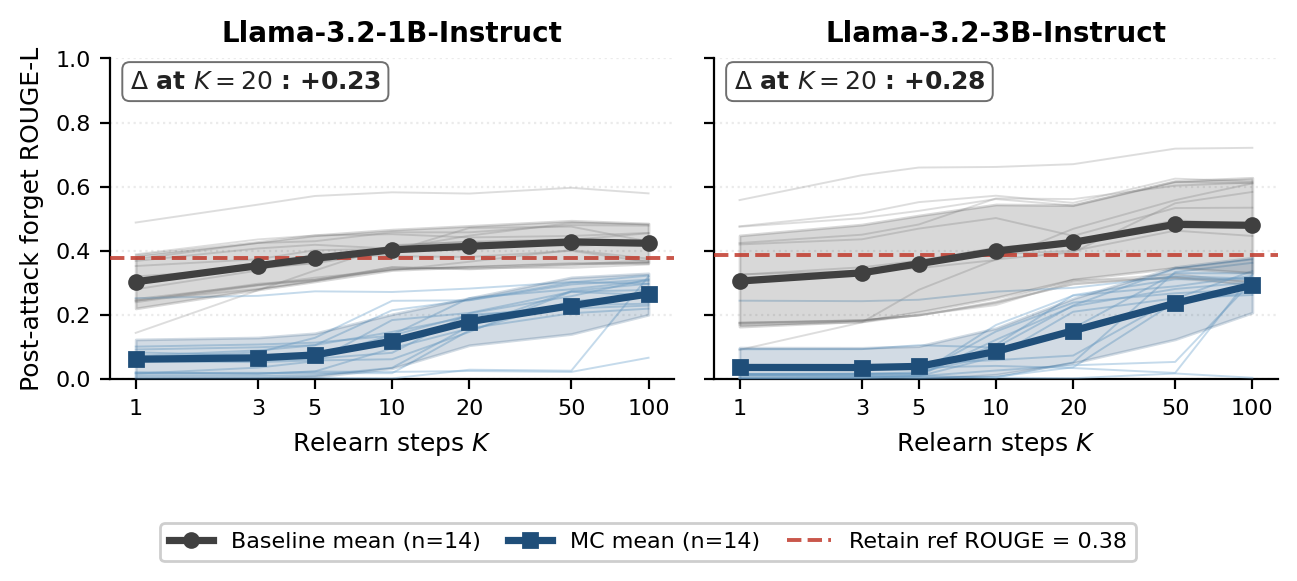}
  \caption{Held-out post-attack ROUGE-L under LoRA-r8 over $K$ for 14 baselines and \textsc{MC} at 1B/3B. Red dashed line is the $\theta_0$ recovery ceiling. Per-method numbers in App.~\ref{app:attacker-scaling}.}
  \label{fig:kcurve}
\end{figure}

The intuition is gradient dominance. Above the threshold $\delta$, the forget hinge gradient strictly exceeds the bounded KL and native contributions combined, so no strict local minimum can sit there (full proof in App.~\ref{app:proofs}). Together with Theorem~\ref{thm:kkt-cliff}, token-saturating objectives sit on the positive-cliff side while \textsc{MC} admits no strict local minimum above any admissible $\delta$. Empirically, \textsc{MC}-polished checkpoints satisfy $\Delta(\hat\theta_{\text{MC}}) < 0$ in $66$ of $70$ (method, panel) cells across the five completed TOFU panels (mean $\Delta = -30.8$ at 1B \texttt{forget10}, App.~\ref{app:cliff-numbers}). All four exceptions are UNDIAL ($\Delta \in [+1.6, +3.0]$ at table precision, Table~\ref{tab:a-t-mc-delta}), whose distillation target flattens the answer distribution, precisely the case where margin gradients, and hence the coercivity constant $G_f(\delta)$, are smallest, so the fixed-budget polish can terminate before crossing. As a statement over the whole region $\Theta_0$, the coercivity hypothesis is a proof device rather than an observable, but it is not vacuous. Its on-trajectory signature, a hinge gradient norm bounded away from zero above the cliff, is measured during instrumented polish runs (App.~\ref{app:proofs}, Fig.~\ref{fig:traj}), and its failure mode is visible as non-crossing, with UNDIAL the observed instance.

\begin{theorem}[Attack-budget margin lift bound]
\label{thm:relearn-bound}
Fix a convex compact attack region $\Theta_A \ni \hat\theta$ and a step budget $H > 0$, and let
$\mathcal{A}(\eta, N, H, \Theta_A)$ be the class of iterative attackers
$\theta^{(0)} = \hat\theta$, $\theta^{(n+1)} = \theta^{(n)} + \delta^{(n)}$, $\tilde\theta_N :=
\theta^{(N)}$, whose increments obey $\|\delta^{(n)}\| \le \eta H$ ($\eta$ a per-step scale and $H$ a step budget,
and only their product enters the bound) and whose iterates remain in
$\Theta_A$, covering LoRA, full-parameter, and adaptive fine-tuning with bounded steps. Let
$m_\beta$ be the $\beta$-smoothed diagnostic (competitor max replaced by an inverse-temperature-$\beta$
log-sum-exp, positions fixed at the defender's diagnostic positions $\hat t_{\hat\theta}$,
App.~\ref{app:proofs}), which brackets the hard margin within
$c_\beta := \log(V{-}1)/\beta$ ($V$ the vocabulary size), and set $G_m := \sup_{\Theta_A} \|\nabla_\theta m_\beta\|$ and $L_m$
the smoothness constant of $m_\beta$ on $\Theta_A$ (both finite by compactness). Then, uniformly over
$\mathcal{A}(\eta, N, H, \Theta_A)$,
\begin{equation}
\label{eq:relearn-bound}
\Delta(\tilde\theta_N)  \le  \Delta(\hat\theta) + \eta N G_m H + \tfrac{1}{2} L_m \eta^2 N H^2 + c_\beta ,
\end{equation}
where $\Delta(\tilde\theta_N)$ is evaluated at the defender's diagnostic positions. At Llama-3's
vocabulary, $\beta = 120$ gives $c_\beta < 0.1$. The derivation is in App.~\ref{app:proofs}, and the FPFT
instantiation is in App.~\ref{app:fpft-bound}.
\end{theorem}

Hence whenever $-\Delta(\hat\theta_{\text{MC}})$ exceeds the lift budget on the right-hand side,
\emph{every} attacker in the class leaves $\Delta(\tilde\theta_N) < 0$, a certificate uniform over
the class rather than per realized trajectory. A baseline with $\Delta(\hat\theta) \ge 0$ starts on the
attacker's side, where the attack only needs sign preservation rather than reversal. Two scope notes
follow. The evaluation diagnostic re-selects positions per model and coincides with the theorem's
frozen-position diagnostic unless the attack relocates the entropy peak, and comparisons \emph{between}
attacker classes (rank, FPFT) are empirical, since the bound orders budgets within a class, not realized
lifts across classes (App.~\ref{app:strong-attackers}).

\begin{table*}[t]
  \centering
  \small
  \setlength{\tabcolsep}{4pt}
  \caption{\textbf{Cross-axis robustness matrix.} Cells are \emph{(\textsc{MC} wins) / (populated cells)} with 95\% Clopper--Pearson CIs, paired per base. MIA.agg counts raw-AUC wins (advantage summary in App.~\ref{app:mia-breakdown}, per-method backing in App.~\ref{app:t2-per-cell}). $^{\dagger}$RSNPO 8B is reported in Fig.~\ref{fig:cliff} only and is not part of the cross-axis attack sweep, $^{\ddagger}$MUSE-News on Llama-2-7B-hf omits CRNPO, was not run under FPFT, and shows a reversed advantage shift $0.028 \to 0.228$ (App.~\ref{app:muse}), and $^{\S}$the Phi-3.5-mini panel was not run under FPFT and its gradient-family and JensUn \textsc{MC} cells collapse utility (App.~\ref{app:phi}).}
  \label{tab:t2-robust}
  \resizebox{\textwidth}{!}{%
  \begin{tabular}{l c ccccc}
    \toprule
    Stress axis & $n$ & F.agg ($\downarrow$) & MU ($\uparrow$) & MIA.agg ($\downarrow$) & K20-LoRA ($\downarrow$) & K20-FPFT ($\downarrow$) \\
    \midrule
    Multi-seed (5 base $\times$ 3 seed, 1B f10) & 15 & \textbf{15/15}\,{\scriptsize [0.78,\,1.00]} & \textbf{0/15}\,{\scriptsize [0.00,\,0.22]} & --- & \textbf{15/15}\,{\scriptsize [0.78,\,1.00]} & \textbf{5/5}\,{\scriptsize [0.48,\,1.00]} \\
    1B forget01 (14 base) & 14 & \textbf{14/14}\,{\scriptsize [0.77,\,1.00]} & \textbf{0/14}\,{\scriptsize [0.00,\,0.23]} & \textbf{14/14}\,{\scriptsize [0.77,\,1.00]} & \textbf{14/14}\,{\scriptsize [0.77,\,1.00]} & \textbf{14/14}\,{\scriptsize [0.77,\,1.00]} \\
    1B forget05 (14 base) & 14 & \textbf{14/14}\,{\scriptsize [0.77,\,1.00]} & \textbf{0/14}\,{\scriptsize [0.00,\,0.23]} & \textbf{14/14}\,{\scriptsize [0.77,\,1.00]} & \textbf{14/14}\,{\scriptsize [0.77,\,1.00]} & \textbf{14/14}\,{\scriptsize [0.77,\,1.00]} \\
    3B forget10 (14 base) & 14 & \textbf{14/14}\,{\scriptsize [0.77,\,1.00]} & \textbf{1/14}\,{\scriptsize [0.00,\,0.34]} & \textbf{14/14}\,{\scriptsize [0.77,\,1.00]} & \textbf{14/14}\,{\scriptsize [0.77,\,1.00]} & ---/14 \\
    8B forget10 (13 base$^{\dagger}$) & 13 & \textbf{13/13}\,{\scriptsize [0.75,\,1.00]} & 4/13\,{\scriptsize [0.10,\,0.61]} & \textbf{13/13}\,{\scriptsize [0.75,\,1.00]} & \textbf{13/13}\,{\scriptsize [0.75,\,1.00]} & \textbf{13/13}\,{\scriptsize [0.75,\,1.00]} \\
    MUSE-News (13 base$^{\ddagger}$) & 13 & \textbf{13/13}\,{\scriptsize [0.75,\,1.00]} & ---/13 & \textbf{13/13}\,{\scriptsize [0.75,\,1.00]} & \textbf{13/13}\,{\scriptsize [0.75,\,1.00]} & ---/13 \\
    Phi-3.5 forget10 (14 base$^{\S}$) & 14 & \textbf{14/14}\,{\scriptsize [0.77,\,1.00]} & 1/14\,{\scriptsize [0.00,\,0.34]} & \textbf{14/14}\,{\scriptsize [0.77,\,1.00]} & \textbf{14/14}\,{\scriptsize [0.77,\,1.00]} & ---/14 \\
    \midrule
    \textbf{Total} & \textbf{97} & \textbf{97/97} (100.0\%) & 6/84 (7.1\%) & \textbf{82/82} (100.0\%) & \textbf{97/97} (100.0\%) & \textbf{46/46} (100.0\%) \\
    \bottomrule
  \end{tabular}%
  }
\end{table*}

\paragraph{Empirical bridge.}
Eq.~\eqref{eq:relearn-bound} bounds the post-attack margin while \S\ref{sec:exp} reports post-attack ROUGE-L. The two co-move, since recovering forget content under greedy decoding requires many relevant token margins to move toward or into the positive region. Across the 14-method panel, the cliff gap correlates with ROUGE-L recovery at Pearson $r = 0.69 / 0.65 / 0.45$ (1B/3B/8B, with 95\% bootstrap CI at 1B $[0.54, 0.82]$) and Spearman $\rho = 0.91 / 0.89 / 0.78$ (Fig.~\ref{fig:bridge}, App.~\ref{app:bridge}), with the four UNDIAL non-crossing cells among the least robust \textsc{MC} cells, as the link predicts.

\section{Experiments}
\label{sec:exp}

We stress-test \textsc{MC} across one of the broadest axis matrices in robust unlearning, spanning fourteen post-hoc methods from three loss families, three Llama-3 sizes, three forget tiers, three seeds, a second benchmark (MUSE-News on Llama-2-7B-hf), a fourth model family (Phi-3.5-mini, App.~\ref{app:phi}), and an attack spectrum from LoRA and soft prompts to full-parameter fine-tuning over budgets $K{=}1$ to $K{=}100$, for $97$ populated cross-axis cells in all, with every theoretical ingredient measured alongside (\textbf{Verification summary}, \S\ref{sec:exp-ablations}). Twenty-two appendix tables and five figures across twenty-four sections back these axes at per-cell depth. The central result is Fig.~\ref{fig:kcurve}, where baseline held-out post-attack ROUGE rises with the relearn budget $K$ toward the recovery ceiling while \textsc{MC} stays well below it at every budget.

\subsection{Setup}
\label{sec:exp-setup}

\textbf{Benchmarks, models, and methods.}
\textsc{TOFU}~\citep{maini2024tofu} provides \texttt{forget01/05/10} splits ($|\mathcal{D}_f|=40/200/400$) with matched retain references, and headline results use \texttt{forget10}.
We evaluate \texttt{Llama-3.2-Instruct} (1B/3B) and \texttt{Llama-3.1-8B-Instruct}, using TOFU bases and references from \texttt{open-unlearning}. \textsc{MUSE-News}~\citep{shi2024muse} \texttt{knowmem} ($n=100$) tests cross-benchmark and cross-family transfer on \texttt{Llama-2-7B-hf}, with each baseline trained 5 epochs from \texttt{MUSE-news\_target} (App.~\ref{app:muse}). The 14-method panel spans three gradient, nine preference, and two distillation methods (full list with families in Table~\ref{tab:t1-main}).

\textbf{Metrics and attacks.}
Forget metrics ($\downarrow$) are \textit{Q\_A\_Prob}, ROUGE-L, EM, and ES, plus \textit{forget\_quality}, the KS $p$-value of forget Truth-Ratio against $\theta_{\text{ref}}$ ($p>0.05$ meaning distributional indistinguishability). Utility is TOFU MU, the harmonic mean of 9 retain-side measurements. MIA uses five AUCs (loss, min-$k$~\citep{shi2024detecting}, min-$k$++~\citep{zhang2024minkpp}, zlib, and gradnorm capped at $10^6$), and because AUC below 0.5 is reversed separability rather than no signal, we also report the per-detector membership advantage $\frac{1}{5}\sum_i |\mathrm{AUC}_i-0.5|$ (App.~\ref{app:mia-breakdown}). Relearn attacks train or prompt $\hat\theta$ on $\mathcal{D}_{\text{rel}}\subset\mathcal{D}_f$ and evaluate on $\mathcal{D}_f\setminus\mathcal{D}_{\text{rel}}$. We cap $K\le100$ on \texttt{forget05/10}, $K\le20$ on \texttt{forget01}, and $K\le50$ on MUSE. Attacks include LoRA-r8 over $K\in\{1,3,5,10,20,50,100\}$, LoRA-r32 at $K=20$, soft-prompt tuning with $n_{\text{soft}}\in\{20,100\}$~\citep{lester2021power}, full-parameter fine-tuning (FPFT) at $K\in\{20,50\}$, and a diagnostic linear probe on last-layer hidden states.

\textbf{Metric provenance.}
Baseline columns in Tables~\ref{tab:t1-main} and~\ref{tab:a-t1-ext} come from the official
\texttt{open-unlearning} evaluator, directly comparable to published leaderboard numbers, while
\textsc{MC} columns come from our pipeline evaluator (matched prompts, $n_{\text{eval}}{=}200$)
since the official harness was not run on merged LoRA checkpoints. The two agree closely on forget
metrics but differ systematically on MU (panel mean $0.44$ official vs.\ $0.27$ pipeline \emph{on
the same baseline checkpoints}), so the headline MU drop overstates the cost, and under a single
evaluator the comparison is $0.27 \to 0.11$ with identical win counts (App.~\ref{app:cliff-numbers}).
Relearn columns are pipeline-vs-pipeline throughout.

\textbf{\textsc{MC} and statistics.}
We fix $(\kappa,\lambda_{\text{KL}},r,N_{\text{pol}})=(5.0,0.05,32,200)$ (hinge sharpness, probe weight, LoRA rank, polish forget-sample pool) with an 80-step single-sample optimizer schedule, from an HP sensitivity study on 5 representative 1B \texttt{forget10} bases, and reuse it unchanged across all reported settings. The KL probe uses Alpaca~\citep{taori2023alpaca} and $m_{\text{ref}}(t)$ is cached per model and tier. All experiments run on one NVIDIA DGX Spark (App.~\ref{app:compute}, \ref{app:hp-grid}). Multi-seed evaluation uses GradDiff, NPO, SimNPO, UNDIAL, and CRNPO with seeds $\{0,1,2\}$, with 95\% Clopper--Pearson CIs for win-rates, 95\% bootstrap CIs for correlations, and paired comparisons throughout.

\subsection{Main Results at 1B \textsc{forget10}}
\label{sec:exp-main}

Table~\ref{tab:t1-main} shows the 14-method head-to-head. \textsc{MC} wins $14/14$ on F.agg and $14/14$ on K20-LoRA, reducing held-out ROUGE-L from $0.41$ to $0.18$, and also wins $14/14$ on K20-FPFT. On membership inference, \textsc{MC} lowers the raw AUC on $13/14$ methods and on all five detectors' panel means, but it frequently \emph{overshoots past chance}, with per-detector membership advantage improving on only $6/14$ methods ($0.33 \to 0.32$ panel mean), because strongly suppressed confidence is itself separable in the reversed direction (App.~\ref{app:mia-breakdown}). The main cost is TOFU MU ($0.44\to0.11$ as tabulated, $0.27\to0.11$ under a single evaluator per \textbf{Metric provenance} above, and $13/14$ losing either way). Forget-quality is the stricter bar, with only RSNPO and CRNPO exceeding the KS threshold $p>0.05$, yet all fourteen baselines fail it too (best baseline $p\approx10^{-3}$ under the same evaluator), so \textsc{MC} strictly improves the panel's hardest metric. The seven NPO-based defense variants do not cross the cliff (K20-LoRA mean $0.41$ vs vanilla NPO's $0.35$), while \textsc{MC} reduces all seven to mean $0.21$ with $\Delta<0$.

\subsection{Cross-Axis Robustness}
\label{sec:exp-robust}

Table~\ref{tab:t2-robust} is the cross-axis robustness matrix, aggregating paired win-rates across seeds, three forget tiers, three model sizes, MUSE-News, the Phi-3.5 panel, and both LoRA and FPFT attackers, with per-method backing in App.~\ref{app:t2-per-cell}. \textsc{MC} wins every populated relearn cell, $97/97$ on K20-LoRA and $46/46$ on K20-FPFT, lowers raw membership AUC on $82/82$, and wins the forget aggregate in all $97$ (qualitative post-attack outputs in App.~\ref{app:depth}). The one axis where the tradeoff surfaces is utility, with MU wins at $6/84$ (MUSE excluded since TOFU MU is undefined on \texttt{knowmem}), a loss that is benchmark-local rather than general capability (MMLU evidence in \S\ref{sec:exp-ablations}).

On MUSE-News, \textsc{MC} reduces K20-LoRA ROUGE-L from $0.053$ to $0.006$ ($13/13$, Table~\ref{tab:t5-muse}), so operational recovery and relearn robustness transfer even where the margin diagnostic is less sharp (MIA caveat in Limitations and App.~\ref{app:muse}).

\subsection{Deployment without a Retain-Trained Reference}
\label{sec:exp-deploy}

Table~\ref{tab:t4-deploy} anchors the polish at $\theta_0$ (margin anchor plus retain hinge, \S\ref{sec:m-v3}) and uses $\theta_{\text{ref}}$ only for evaluation. Deployment beats baseline on $14/14$ methods, lies only $0.023$ above reference-anchored \textsc{MC} on mean K20-LoRA, and on three methods even beats the reference-anchored variant. The cost is extra MU loss ($0.11\to0.07$, $5/14$ methods recovering some MU, cross-size and cross-tier transfer in App.~\ref{app:deploy}).

\subsection{Ablations}
\label{sec:exp-ablations}

Table~\ref{tab:t3-ablation} shows that gradient ascent diverges, and that both the forget-only hinge and the two-sided retain hinge reach robustness comparable to full \textsc{MC} but at markedly lower utility (MU $0.091$ and $0.080$ vs.\ $0.149$). It is the KL probe, not a retain-side hinge, that limits the utility cost of crossing. HP perturbations change K20-LoRA by at most $0.015$ in the shown corners and $0.040$ across the full sensitivity table (App.~\ref{app:hp-grid}), far below the $0.255$ gap to baseline on the same panel. Stronger attackers preserve the lead, with LoRA-r32, soft-prompt $n_{\text{soft}}{=}100$, and FPFT $K{=}50$ reducing ROUGE-L by $0.149/0.188/0.260$ (App.~\ref{app:strong-attackers}, \ref{app:fpft-per-family}), and an \emph{adaptive} attacker that directly ascends the margin diagnostic does no better than generic relearning ($0.154$ vs $0.179$ panel mean, App.~\ref{app:depth}), so the gain is tied to no single attacker, tuning point, or attacker blindness. A further ablation line decomposes the utility cost into a reallocatable residual-pressure leak, recovered by a reference-confidence gate wherever the native forget term is inert, and an irreducible remainder that a controlled $2{\times}2$ attributes to crossing pressure itself (App.~\ref{app:hinge-ew}, \ref{app:refgate}, \ref{app:pareto}).

\paragraph{Verification summary.}
Every theoretical ingredient is measured rather than assumed. The average-log-odds premise certifies the cliff on $34/42$ cells from a single forward pass per checkpoint (App.~\ref{app:logodds-verify}), the coercivity premise shows its on-trajectory gradient signature in instrumented polish runs (App.~\ref{app:proofs}), and the relearn bound holds with one to two orders of slack on instrumented attacks, which is why the operational K-sweep, flat under budgets up to $K{=}100$ with every \textsc{MC} cell at or below $0.33$, carries the empirical load (App.~\ref{app:thm3-constants}, \ref{app:attacker-scaling}). Zero-shot MMLU stays within $0.03$ of each baseline for eleven of fourteen methods at 3B and 8B, with named exceptions (App.~\ref{app:mmlu}).

\section*{Limitations}

\textsc{MC} reduces TOFU MU ($0.44 \to 0.11$ at 1B \texttt{forget10}), and forget-quality KS $p$ clears the $0.05$ bar in only $2/14$ cells (though every baseline fails the same bar). Part of this cost is a recoverable residual-pressure leak and the attributed remainder is the price of crossing pressure itself (\S\ref{sec:exp-ablations}), so the cost is characterized but not eliminated. Its margin pressure often overshoots membership detectors past chance into reversed separability, improving per-detector advantage on only $6/14$ TOFU methods despite raw AUCs falling on $13/14$, and reaching a $0.228$ reversed signal on MUSE-News whose baselines already sit near chance ($0.028$). The same overshoot collapses 1B MMLU toward chance for three of fourteen methods, a damage that recedes at 3B and 8B for all methods except CRNPO and JensUn (App.~\ref{app:mmlu}). An MIA- and capability-aware stopping rule for margin pressure is left to future work, with its margin-targeted half implemented and swept in App.~\ref{app:pareto}. The cliff bound holds at any stationary point meeting the diagnostic floor (Asm.~\ref{asm:overlap}), the relearn bound covers bounded-step weight-space attackers (App.~\ref{app:fpft-bound}), and the floor and overlap constants are observable but not adversarially controlled. Evaluation spans three benchmarks and three model families, leaving others to future work.

\section{Conclusion}
\label{sec:conclusion}

We identify the margin cliff, the convergence of per-token forget-set margins above the retain reference in $41$ of $42$ method--size cells, as a stationarity property of retain-regularized unlearning, and show that crossing it requires a non-saturating forget objective. Margin Calibration does so as a strict plug-in under a single frozen configuration, crossing in $66$ of $70$ evaluated cells and improving forget recovery and relearn robustness on every populated cell, with a deployment variant that needs no retain-trained reference. The main cost is reduced utility, leaving utility-preserving cliff crossing as the key direction for future work.

\section*{Acknowledgments}
Funded by the European Union. Views and opinions expressed are however those of the author(s) only and do not necessarily reflect those of the European Union or the European Health and Digital Executive Agency (HADEA). Neither the European Union nor the granting authority can be held responsible for them. RobustifAI project, ID 101212818.

\bibliography{aaai2027}

\clearpage
\appendix
%

\section{Formal assumptions and lemmas}
\label{app:formal}

\subsection{Assumption~\ref{asm:overlap}, formal version}

We quantify forget--retain coupling by a single \emph{data-side} statistic and package its
optimization consequence as a probability floor.

\begin{assumption}[Forget--retain coupling and retain floor]
\label{asm:overlap}
Let $\epsilon \in [0, 1]$ be the forget--retain prompt overlap, defined as the fraction of
$\mathcal{D}_f$ bigrams also covered by $\mathcal{D}_r$ ($\epsilon = 0$ for disjoint $n$-grams and
$\epsilon = 1$ for identical coverage), and let
$\ell(\theta) := \mathbb{E}_{\mathcal{D}_f}\big[\log\frac{p_{\hat t}}{1 - p_{\hat t}}\big]$ be the
average gold log-odds at the diagnostic positions of Eq.~\eqref{eq:margin-agg}. There exists a finite
floor $\ell_\star = \ell_\star(\epsilon, \lambda)$, \emph{non-decreasing} in $\epsilon$ and in the
retain weight $\lambda$, such that at every strict local minimum $\theta^\star$ of
\eqref{eq:standard-unlearn},
\[
\ell(\theta^\star) \ge \ell_\star .
\]
The assumption is stated for the joint objective as such. Definition~\ref{def:forget-saturating}
supplies the mechanism that makes it plausible for token-saturating forget losses, while for the
bounded-gradient losses (GradDiff family, UNDIAL) it is assumed directly with no mechanistic
derivation (App.~\ref{app:loss-saturation}).
\end{assumption}

\paragraph{Observability and direction.}
$\epsilon$ is a pure data statistic computable in one pass over $\mathcal{D}_f \cup \mathcal{D}_r$. We
measure $\epsilon \approx 0.44$ at 1B forget10 (unique word-level bigram types of forget Q\&A covered
by retain Q\&A, probe script in the code release), and $\ell(\theta)$ is computable from the same forward
pass that produces the margins, so the assumption is directly checkable per checkpoint
(App.~\ref{app:logodds-verify}). The floor exists because forget and retain share sub-token structure.
The retain NLL resists collapsing shared-representation logits, and a token-saturated forget loss
(Def.~\ref{def:forget-saturating}) supplies vanishing suppressing force at exactly the low-probability
content tokens the diagnostic evaluates. The resistance, and hence $\ell_\star$, \emph{grows} with the
coupling $\epsilon$. At $\epsilon = 0$ the retain term exerts no forget-side pull, $\ell_\star \to
-\infty$, and the cliff dissolves. This is the direction the mechanism requires, with the cliff a
consequence of coupling, not of its absence. The assumption is stated at the aggregate level
deliberately, since a worst-case \emph{per-token} floor would imply it but is empirically vacuous (measured
minimum token gold probabilities reach $10^{-13}$--$10^{-5}$ at 1B), whereas the aggregate is dominated by
typical diagnostic tokens.

\subsection{Lemma~\ref{lem:kl-bounded}, formal version}

\begin{lemma}[KL probe gradient bound, precise]
\label{lem:kl-bounded}
Let $\mathcal{D}_a$ be disjoint from $\mathcal{D}_f$ and $\mathcal{D}_r$, and let $\theta_{\text{anc}}$ be
the behavior anchor. For any $\theta$ in a compact set $\Theta_0 \subset \Theta$ that contains
$\theta_0$, $\theta_{\text{ref}}$, and the polish trajectory, there exists
$G_{\text{KL}}^{\max} > 0$ depending only on $\Theta_0$, $\theta_{\text{anc}}$, and $\mathcal{D}_a$ such
that $\left\| \nabla_\theta \mathcal{L}_{\text{KL}}(\theta) \right\| \le G_{\text{KL}}^{\max}$. In
particular, $G_{\text{KL}}^{\max}$ does not depend on $p_\theta(y_f \mid x_f)$ for any
$(x_f, y_f) \in \mathcal{D}_f$.
\end{lemma}

\begin{proof}
The probe $\mathcal{L}_{\text{KL}}(\theta) = \mathbb{E}_{x \sim \mathcal{D}_a}
[\mathrm{KL}(p_{\theta_{\text{anc}}}(\cdot \mid x) \,\|\, p_\theta(\cdot \mid x))]$ depends on $\theta$
only through the softmax outputs of a fixed, finitely-parameterized network with smooth activations
(SiLU-family in our models) on the fixed inputs
$\mathcal{D}_a$. Softmax outputs are strictly positive, so the log-ratios are finite and the probe is
continuously differentiable in $\theta$ on all of $\Theta$. (Containing the polish trajectory in
$\Theta_0$ presumes bounded iterates, which we assume.) A
continuous function ($\theta \mapsto \|\nabla_\theta \mathcal{L}_{\text{KL}}(\theta)\|$) attains a
finite maximum on the compact set $\Theta_0$, giving
$G_{\text{KL}}^{\max} := \max_{\theta \in \Theta_0} \|\nabla_\theta \mathcal{L}_{\text{KL}}(\theta)\| < \infty$.
Because $\mathcal{D}_a \cap \mathcal{D}_f = \emptyset$, the functional $\mathcal{L}_{\text{KL}}$ makes no
reference to any $p_\theta(y_f \mid x_f)$, so neither does its gradient or the bound.
\end{proof}

\section{Full proofs}
\label{app:proofs}

\subsection{Proof of Theorem~\ref{thm:kkt-cliff}}

\begin{proof}
Let $\theta^\star$ be a strict local minimum of
$\mathcal{L}(\theta) := \mathcal{L}_f(\theta; \mathcal{D}_f) + \lambda \mathcal{L}_r(\theta; \mathcal{D}_r)$.

\paragraph{Step 1 (token-wise saturation caps the suppressing force, motivation).}
By Definition~\ref{def:forget-saturating}, the loss's marginal incentive to suppress any individual
gold token further is capped token by token,
$\left|\partial \ell_f / \partial z_{y_t}\right| \le \rho(p_t)$ with
$\rho(p) \to 0$ as $p \to 0^+$. A suppressed token therefore receives a vanishing forget
push as its probability falls, \emph{independently} of the confident easy tokens elsewhere in the
answer.

\paragraph{Step 2 (the retain term holds the diagnostic level up, assumption).}
Because the retain NLL is minimized on $\mathcal{D}_r$ and forget and retain share sub-token structure
with overlap $\epsilon$, depressing a forget-token logit also raises $\mathcal{L}_r$. The retain
curvature exerts a restoring pull, while by Step~1 the opposing forget push decays like $\rho(p_t)$.
Assumption~\ref{asm:overlap} packages this balance at stationarity as the floor
$\ell(\theta^\star) \ge \ell_\star$. Steps~1--2 motivate the assumption and are not used
quantitatively below, since the floor carries the quantitative load and is checked per checkpoint
(App.~\ref{app:logodds-verify}).

\paragraph{Step 3 (competitor-mass lemma, unconditional).}
Fix any answer position $t$ with gold probability $p_t := p_{\theta^\star}(y_t \mid x, y_{<t})$. Its
strongest competitor carries probability $\max_{v \neq y_t} p_{\theta^\star}(v) \le 1 - p_t$, so
\begin{align*}
m_{\theta^\star}(t)
&= \log p_t - \log \max_{v \neq y_t} p_{\theta^\star}(v) \\
&\ge \log p_t - \log(1 - p_t)
= \log \tfrac{p_t}{1 - p_t},
\end{align*}
with no assumption at all, in particular at the max-entropy position $\hat t_{\theta^\star}$ that
the diagnostic~\eqref{eq:margin-agg} evaluates.

\paragraph{Step 4 (cliff lower bound).}
Averaging Step~3 over samples at their diagnostic positions and applying
Assumption~\ref{asm:overlap},
\begin{align*}
m_{\theta^\star}(\mathcal{D}_f) &\ge
\mathbb{E}_{\mathcal{D}_f}\left[\log\tfrac{p_{\hat t}}{1 - p_{\hat t}}\right]
= \ell(\theta^\star) \ge \ell_\star ,
\\ \text{hence}\quad
\Delta(\theta^\star) &\ge \ell_\star - m_{\text{ref}}(\mathcal{D}_f) =: C .
\end{align*}
The retain-only reference treats forget content as out-of-distribution at exactly these
high-uncertainty positions, so $m_{\text{ref}}(\mathcal{D}_f) < 0$, and it is directly measured ($-4.01$
at 1B forget10). $C > 0$ iff $\ell_\star > m_{\text{ref}}$. Because Steps 3--4 are unconditional, the measured
$\ell(\theta^\star)$ certifies the conclusion directly. The certificate $\ell(\theta^\star) > m_{\text{ref}}$, and
hence $\Delta(\theta^\star) \ge \ell(\theta^\star) - m_{\text{ref}} > 0$ with no appeal to
Asm.~\ref{asm:overlap}, holds on $34/42$ (method, size) cells. In the remainder the competitor-mass
relaxation is too loose to certify, though the measured $\Delta$ stays positive in all of them
except the degenerate GradDiff-8B cell (App.~\ref{app:logodds-verify}).
\end{proof}

\begin{table*}[tp]
\centering
\small
\setlength{\tabcolsep}{4.0pt}
\begin{tabular}{lrrcrrcrrc}
\toprule
Method & \multicolumn{3}{c}{1B} & \multicolumn{3}{c}{3B} & \multicolumn{3}{c}{8B}\\
\cmidrule(lr){2-4} \cmidrule(lr){5-7} \cmidrule(lr){8-10}
  & $\ell(\hat\theta)$ & $m_{\hat\theta}$ & c. & $\ell(\hat\theta)$ & $m_{\hat\theta}$ & c. & $\ell(\hat\theta)$ & $m_{\hat\theta}$ & c.\\
\midrule
GradDiff & $-3.77$ & $-2.33$ & \ding{51} & $-2.84$ & $-1.55$ & \ding{51} & $-6.31$ & $-5.42$ & --\\
RMU & $-2.92$ & $-1.22$ & \ding{51} & $-3.05$ & $-1.66$ & \ding{51} & $-5.25$ & $-3.62$ & --\\
PDU & $-4.82$ & $-1.12$ & -- & $-6.68$ & $-0.56$ & -- & $-4.24$ & $-0.10$ & \ding{51}\\
NPO & $-4.30$ & $-2.03$ & -- & $-3.83$ & $-1.64$ & \ding{51} & $-3.61$ & $-2.36$ & \ding{51}\\
SimNPO & $-2.09$ & $-0.59$ & \ding{51} & $-1.34$ & $-0.33$ & \ding{51} & $+0.19$ & $+0.50$ & \ding{51}\\
SAMNPO & $-2.37$ & $-0.65$ & \ding{51} & $-1.60$ & $-0.12$ & \ding{51} & $+0.10$ & $+0.71$ & \ding{51}\\
CRNPO & $-1.74$ & $-0.26$ & \ding{51} & $-0.67$ & $+0.36$ & \ding{51} & $+2.18$ & $+2.49$ & \ding{51}\\
RSNPO & $-2.23$ & $-0.66$ & \ding{51} & $-1.09$ & $+0.02$ & \ding{51} & $-1.20$ & $-0.15$ & \ding{51}\\
SWANPO & $-4.32$ & $-2.02$ & -- & $-3.77$ & $-1.60$ & \ding{51} & $-3.68$ & $-2.42$ & \ding{51}\\
LATNPO & $-4.38$ & $-2.13$ & -- & $-3.80$ & $-1.61$ & \ding{51} & $-3.52$ & $-2.19$ & \ding{51}\\
ILUNPO & $-4.32$ & $-2.06$ & -- & $-3.74$ & $-1.58$ & \ding{51} & $-3.56$ & $-2.29$ & \ding{51}\\
RNANPO & $-3.81$ & $-1.80$ & \ding{51} & $-3.74$ & $-1.50$ & \ding{51} & $-3.54$ & $-2.25$ & \ding{51}\\
UNDIAL & $-2.24$ & $-0.14$ & \ding{51} & $-1.63$ & $+0.21$ & \ding{51} & $-1.01$ & $+0.47$ & \ding{51}\\
JensUn & $-2.11$ & $-0.56$ & \ding{51} & $-1.24$ & $-0.10$ & \ding{51} & $-0.08$ & $+0.28$ & \ding{51}\\
\midrule
$\theta_{\text{ref}}$ (retain90) & $-5.73$ & $-4.01$ & --- & $-5.74$ & $-4.17$ & --- & $-6.17$ & $-4.93$ & ---\\
\bottomrule
\end{tabular}
\caption{Empirical verification of the average-log-odds premise of
Theorem~\ref{thm:kkt-cliff} on TOFU \texttt{forget10} (200 samples, diagnostic
positions of Eq.~\eqref{eq:margin-agg}, one forward pass per checkpoint).
Per size: $\ell(\hat\theta)$ is the mean diagnostic-position gold log-odds,
$m_{\hat\theta}$ the measured margin diagnostic, and \emph{c.} whether the
average-form bound certifies the cliff ($\ell > m_{\text{ref}}$, with
$m_{\text{ref}} = -4.01 / -4.17 / -4.93$).
The unconditional inequality $m_{\hat\theta} \ge \ell(\hat\theta)$ holds in
every populated cell; the bound certifies $9/14$ at 1B, $13/14$ at 3B, $12/14$ at 8B.
Uncertified cells are loose rather than violated: diagnostic positions are
maximum-entropy positions carrying diffuse competitor mass (measured
top-competitor probability far below the worst-case $1 - p \approx 1$), and
measured $\Delta$ stays positive in every cell except GradDiff at 8B, the utility-collapsed boundary case discussed in \S\ref{sec:m-cliff}.}
\label{tab:a-t-logodds}
\end{table*}

\paragraph{Remark (per-token floor as an informal sufficient condition).}
If every suppressed forget token additionally obeyed a worst-case floor
$p_t \ge p_\star$, Step~3 would give $\ell_\star \ge \log\frac{p_\star}{1-p_\star}$, but such a floor
is empirically vacuous at the converged checkpoints (measured minimum token gold probabilities reach
$10^{-13}$--$10^{-5}$ at 1B, App.~\ref{app:logodds-verify}), which is precisely why
Assumption~\ref{asm:overlap} is stated at the aggregate level, where $\ell$ is dominated by typical
diagnostic tokens and is robust to a small mass of extreme outliers.

\paragraph{Remark (independence from the position-selection rule).}
The max-entropy rule enters the proof only through the choice of evaluation positions. The
competitor-mass bound of Step~3 holds at \emph{every} answer position, so Steps~3--4 yield the same
cliff bound for any measurable selection rule $\hat t(x, y)$, with the floor $\ell$ and the reference
level $m_{\text{ref}}$ re-evaluated at the positions that rule selects. The max-entropy instantiation
is the one Eq.~\eqref{eq:margin-agg} reports, and nothing in the argument privileges it.

\subsection{Proof of Theorem~\ref{thm:v3-cliff-cross}}

\begin{proof}
Suppose, for contradiction, that $\theta^\star \in \Theta_0$ is a strict local minimum of
$\mathcal{L}_{\text{polish}} = \mathcal{L}_{\text{native}} + \mathcal{L}_{\text{forget}} + \lambda_{\text{KL}} \mathcal{L}_{\text{KL}}$
with $\Delta(\theta^\star) \ge \delta$.

\paragraph{Step 1 (forget gradient dominates above the threshold).}
By the directional coercivity hypothesis there is a unit vector $u$ with
$u^\top \xi \ge G_f(\delta) > 0$ for every $\xi \in \partial \mathcal{L}_{\text{forget}}(\theta^\star)$.
At points of differentiability this reads
$u^\top \nabla_\theta \mathcal{L}_{\text{forget}}(\theta^\star) \ge G_f(\delta) > 0$, which is the case
we write out (competitor ties are handled by the same hypothesis through the remark below).
This is the sole substantive hypothesis, and it is a genuine assumption about the hinge landscape.
$\Delta(\theta) \ge \delta$ does not by itself force per-token training gaps (the diagnostic
compares each model at its \emph{own} max-entropy positions, while the hinge acts on common-position
gaps $m_\theta(t) - m_{\text{mar}}(t)$), so nonemptiness of the super-threshold set
$S_\delta := \{(x,t) : m_\theta(t) - m_{\text{mar}}(t) \ge \delta\}$ does not follow from $\Delta$
alone. A sufficient primitive condition, on any set of parameters where $\mu(S_\delta(\theta))$ is
bounded below, with $\mu$ the empirical probability measure over (sample, token) pairs in the
hinge average so that complements have mass at most $1$, is a \emph{common ascent direction}. If there is a unit
vector $u$ with $u^\top \nabla_\theta m_\theta(t) \ge \gamma > 0$ for every token whose gap exceeds
$-\delta'$ (and $\|\nabla_\theta m_\theta(t)\| \le G_{\text{tok}}$ throughout $\Theta_0$), then, since the hinge weight
$\sigma(\kappa(m_\theta(t) - m_{\text{mar}}(t)))$ is at least $\sigma(\kappa\delta)$ on $S_\delta$,
nonnegative everywhere, and at most $e^{-\kappa\delta'}$ on tokens with gap below $-\delta'$,
projecting on $u$ gives
\[
u^\top \nabla_\theta \mathcal{L}_{\text{forget}}(\theta)
 \ge  \sigma(\kappa\delta)\, \gamma\, \mu(S_\delta)  -  e^{-\kappa\delta'} G_{\text{tok}} ,
\]
yielding a uniform $G_f(\delta) > 0$ precisely when $\gamma$ and $\mu(S_\delta)$ admit uniform lower
bounds over $\{\theta \in \Theta_0 : \Delta(\theta) \ge \delta\}$. We also probe the coercivity
empirically. Instrumenting the polish (hinge-only gradient on a
fixed probe batch, logger in the code release) gives $\|\nabla_\theta \mathcal{L}_{\text{forget}}\| \ge 5.4$ (starting at $12.2$)
at every logged point until the diagnostic has crossed the reference, decaying to zero only after the
trajectory has descended far below the anchor (Fig.~\ref{fig:traj}, panel~d), the one-sided
stationary set of Theorem~\ref{thm:v3-cliff-cross}, reached on the crossing side only. A single
trajectory samples the region rather than covering it, so this is the premise's observable footprint
and failure check, not a certification of the region-wide bound, which functions as a proof device.

\begin{figure}[t]
  \centering
  \includegraphics[width=\columnwidth]{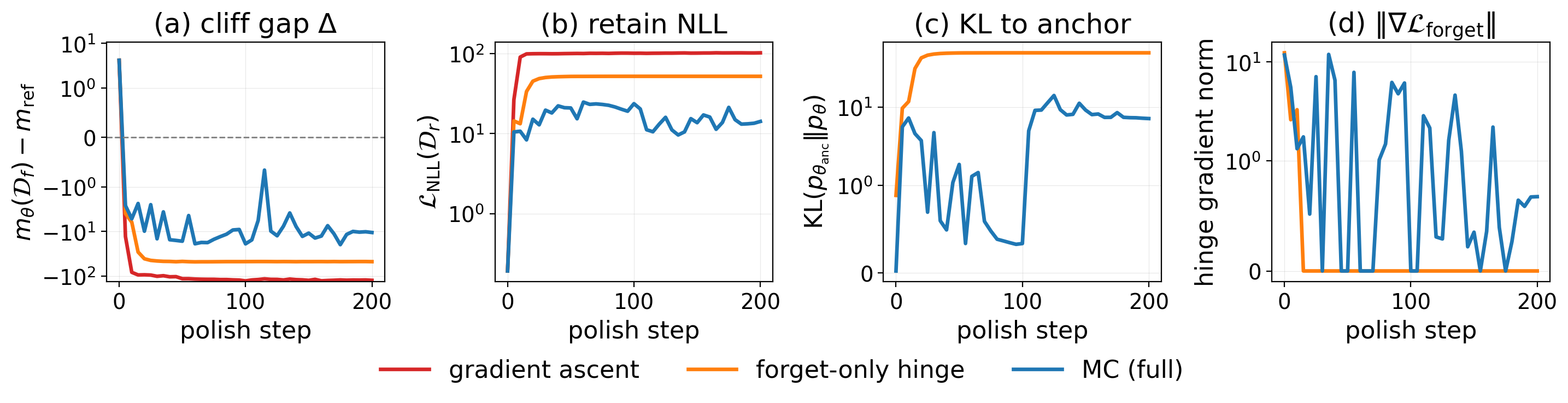}
  \caption{Cliff trajectory over 200 polish steps on CRNPO at 1B \texttt{forget10}, for three variants (gradient ascent in red, forget-only hinge in orange, full \textsc{MC} in blue). Panel (a) plots the cliff gap $\Delta = m_\theta(\mathcal{D}_f) - m_{\mathrm{ref}}$ under the evaluation diagnostic of Eq.~\eqref{eq:margin-agg} (max-entropy position) on a symmetric log scale, panel (b) plots the retain-set NLL as a utility proxy logged during polish, and panel (c) plots the KL probe to the reference anchor on Alpaca (gradient ascent has no anchor in its objective and therefore no curve in this panel by construction). Panel (d) plots the hinge gradient norm $\|\nabla_\theta\mathcal{L}_{\mathrm{forget}}\|$ on a fixed probe batch during instrumented reruns of the two hinge variants (gradient ascent has no hinge and no curve, as in panel c). The norm stays bounded away from zero on the positive-cliff side and decays toward zero only after the trajectory has descended far below the anchor, the on-trajectory signature of the coercivity premise of Theorem~\ref{thm:v3-cliff-cross} (necessary along the observed path, though a single trajectory cannot certify the region-wide statement). Gradient ascent drives retain NLL into the hundreds within roughly ten steps, the forget-only hinge crosses the cliff but lets retain NLL and anchor KL drift, while full \textsc{MC} crosses the cliff with bounded retain NLL and KL. The logger replicates the polish objective with single-sample steps at a fixed learning rate over a 200-step horizon for illustration, longer than the 80-step production budget (\S\ref{sec:exp-setup}), and for the CRNPO base shown here the production polish terminates far shallower ($\Delta = -0.9$, App.~\ref{app:cliff-numbers}).}
  \label{fig:traj}
\end{figure}

\paragraph{Step 2 (the competing gradients are bounded above).}
By Lemma~\ref{lem:kl-bounded}, $\|\nabla_\theta \mathcal{L}_{\text{KL}}(\theta^\star)\| \le G_{\text{KL}}^{\max}$,
uniformly and independently of any forget probability. By hypothesis
$\|\nabla_\theta \mathcal{L}_{\text{native}}(\theta^\star)\| \le G_{\text{nat}}$. The polish initializes
at $\hat\theta_{\text{base}}$, at or near a stationary point of $\mathcal{L}_{\text{native}}$ (near for
variants whose training-time tricks are stripped, App.~\ref{app:loss-saturation}), so the native
gradient is small near the start and, like the probe gradient, bounded by continuity on the compact
polish region.

\paragraph{Step 3 (stationarity contradiction).}
Projecting the polish gradient on $u$ and using $u^\top v \ge -\|v\|$ together with the hypothesis
$\lambda_{\text{KL}} G_{\text{KL}}^{\max} + G_{\text{nat}} < G_f(\delta)$,
\begin{align*}
&\left\| \nabla_\theta \mathcal{L}_{\text{polish}}(\theta^\star) \right\|
 \ge  u^\top \nabla_\theta \mathcal{L}_{\text{polish}}(\theta^\star)
 \ge  u^\top \nabla_\theta \mathcal{L}_{\text{forget}}(\theta^\star) \\
&\quad - \lambda_{\text{KL}} \left\| \nabla_\theta \mathcal{L}_{\text{KL}}(\theta^\star) \right\|
   - \left\| \nabla_\theta \mathcal{L}_{\text{native}}(\theta^\star) \right\| \\
&\ge  G_f(\delta) - \lambda_{\text{KL}} G_{\text{KL}}^{\max} - G_{\text{nat}}  >  0,
\end{align*}
contradicting the stationarity condition $\nabla_\theta \mathcal{L}_{\text{polish}}(\theta^\star) = 0$
(at nondifferentiable $\theta^\star$, the Clarke extension in the remark below applies).
Hence every strict local minimum (in the unconstrained sense) lying in $\Theta_0$ satisfies $\Delta(\theta^\star) < \delta$.

\paragraph{From $\delta$ to cliff crossing.}
The bound holds for every $\delta > 0$ at which coercivity is available with
$\lambda_{\text{KL}} G_{\text{KL}}^{\max} + G_{\text{nat}} < G_f(\delta)$. If coercivity persists as
$\delta \downarrow 0$, i.e.\
$\liminf_{\delta \to 0^+} G_f(\delta) > \lambda_{\text{KL}} G_{\text{KL}}^{\max} + G_{\text{nat}}$, then
$\Delta(\theta^\star) \le 0$, and reference-anchored \textsc{MC} admits no positive-cliff strict local
minimum. For the deployment variant ($m_{\text{mar}} = m_{\theta_0}$, retain hinge in place of the KL
probe, \S\ref{sec:m-v3}), the identical dominance argument applies with a re-anchored coercivity
hypothesis (stated with the deployment gap $m_\theta - m_{\theta_0}$ in place of $\Delta$) and with
the retain-hinge gradient in place of the KL bound (the hinge is locally Lipschitz, so a
Clarke-subgradient bound on the compact region replaces the $C^1$ argument of
Lemma~\ref{lem:kl-bounded}), but yields only the deployment-gap conclusion, no strict local minimum above $\theta_0$'s margin level,
which is weak, since the polished model starts \emph{below} that level already. The deployment
variant's crossing of $m_{\text{ref}}$ is an empirical finding (Table~\ref{tab:t4-deploy}), not a
consequence of this theorem.
\end{proof}

\subsection{Proof of Theorem~\ref{thm:relearn-bound}}

\begin{proof}
\textbf{Setup.} Fix the defender's diagnostic positions $\hat t := \hat t_{\hat\theta}(x, y)$ once, at
the pre-attack checkpoint. The $\beta$-smoothed diagnostic is
\[
m_\beta(\theta) := \mathbb{E}_{\mathcal{D}_f}\Big[ \log p_\theta(y_{\hat t} \mid \cdot)
 - \mathrm{LSE}_\beta\big(\{\log p_\theta(v \mid \cdot)\}_{v \neq y_{\hat t}}\big) \Big],
\]
with $\mathrm{LSE}_\beta(x) := \tfrac{1}{\beta}\log\sum_v e^{\beta x_v}$. Since
$\max \le \mathrm{LSE}_\beta \le \max + \log(V{-}1)/\beta$, the frozen-position hard margin $m$
satisfies $m_\beta \le m \le m_\beta + c_\beta$ with $c_\beta = \log(V{-}1)/\beta$. As a composition of
smooth maps on the compact convex $\Theta_A$, $m_\beta$ is $C^1$ with $L_m$-Lipschitz gradient and
$G_m := \sup_{\Theta_A}\|\nabla m_\beta\| < \infty$. Positions are fixed, so no selection switching
occurs along any trajectory.

\paragraph{Step 1 (bounded increments).}
By the class definition, every attacker step satisfies
$\|\theta^{(n+1)} - \theta^{(n)}\| = \|\delta^{(n)}\| \le \eta H$, and every segment
$[\theta^{(n)}, \theta^{(n+1)}]$ lies in $\Theta_A$ by convexity.

\paragraph{Step 2 (per-step lift via smoothness).}
The quadratic upper bound for the $L_m$-smooth $m_\beta$ on each segment gives
{\small
\begin{align*}
m_\beta(\theta^{(n+1)})
&\le{} m_\beta(\theta^{(n)})
  + \nabla m_\beta(\theta^{(n)})^\top \delta^{(n)} \\
  &+ \tfrac{L_m}{2} \|\delta^{(n)}\|^2
\;\le\; m_\beta(\theta^{(n)}) + \eta G_m H + \tfrac{L_m}{2} \eta^2 H^2 ,
\end{align*}
}
using Cauchy--Schwarz for the inner product (the worst case is a step perfectly aligned to raise the
margin).

\paragraph{Step 3 (telescoping and bracket transfer).}
Summing over $N$ steps, $m_\beta(\tilde\theta_N) \le m_\beta(\hat\theta) + \eta N G_m H + \tfrac12 L_m
\eta^2 N H^2$. Transferring through the bracket ($m \le m_\beta + c_\beta$ at $\tilde\theta_N$ and
$m_\beta \le m$ at $\hat\theta$) and subtracting the constant $m_{\text{ref}}(\mathcal{D}_f)$,
\[
\Delta(\tilde\theta_N)
\le \Delta(\hat\theta) + \eta N G_m H + \tfrac{1}{2} L_m \eta^2 N H^2 + c_\beta ,
\]
which is exactly~\eqref{eq:relearn-bound}. All constants are defined on $\Theta_A$ before any attack
runs, so the bound is uniform over the class $\mathcal{A}(\eta, N, H, \Theta_A)$.

\begin{figure*}[t]
  \centering
  \includegraphics[width=\textwidth]{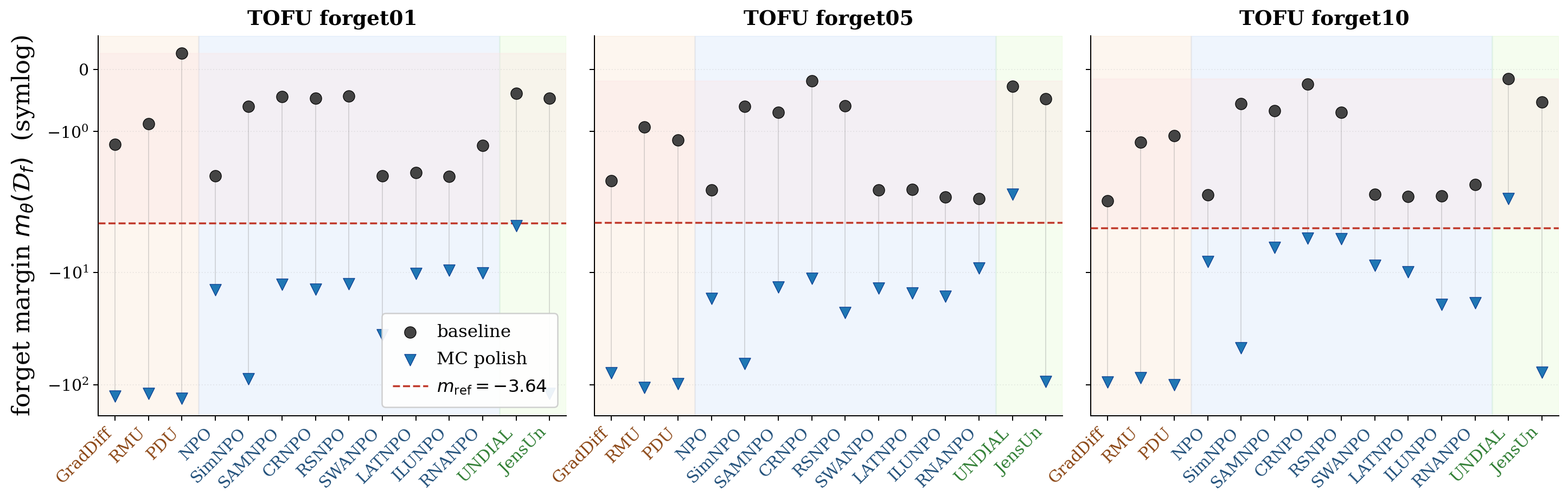}
  \caption{Tier stability of the margin cliff at Llama-3.2-1B. Forget-set margin diagnostic (Eq.~\eqref{eq:margin-agg}) for fourteen baseline methods (gray circles) and their \textsc{MC}-polished counterparts (blue triangles) at TOFU forget01, forget05, and forget10. The retain reference $m_{\mathrm{ref}}$ (red dashed) is the same retain90 checkpoint as in Figure~\ref{fig:cliff} evaluated on each forget split, giving $m_{\mathrm{ref}}=-3.64 / -3.59 / -4.01$ at forget01 / forget05 / forget10 respectively. Retain90 serves as the common evaluation yardstick across tiers, while the tier \emph{calibrations} use the matched retain-99/95 references as training anchors (\S\ref{sec:m-v3}, Table~\ref{tab:a-t2-forget01}). Baseline margins cluster well above $m_{\mathrm{ref}}$ at every tier, and \textsc{MC} crosses the cliff at every tier for $13$ of $14$ methods (the exception is UNDIAL at forget05/10, \S\ref{sec:m-thm2}).}
  \label{fig:cliff-tiers}
\end{figure*}

\paragraph{Cliff-crossing consequence.}
Write the \emph{lift budget} $\Lambda(N) := \eta N G_m H + \tfrac12 L_m \eta^2 N H^2 + c_\beta$, an
a-priori quantity of the attack class with no dependence on the starting gap. Then
$\Delta(\tilde\theta_N) < 0$ for \emph{every} attacker in the class whenever
$-\Delta(\hat\theta) > \Lambda(N)$. A checkpoint that begins sufficiently far below the cliff cannot be
pulled back across it within the class budget. A baseline with $\Delta(\hat\theta) \ge 0$ already sits
on the positive-cliff side, where the attacker needs no lift at all. Comparisons \emph{across} attack
classes (larger rank, FPFT) are not ordered by the theorem (each class has its own realized $H$)
and are settled empirically (App.~\ref{app:strong-attackers}). We also instrument the bound on the
paper's own attack trajectories (App.~\ref{app:thm3-constants}). The one-sided inequality holds with
room on every instrumented trajectory, and its sup-based constants make $\Lambda(N)$ conservative by
one to two orders of magnitude, so the certificate is a sufficient condition that defines the
attack-budget scaling while the operational K-sweep carries the empirical robustness evidence
(App.~\ref{app:attacker-scaling}).
\end{proof}

\paragraph{Remark (implicit selection by the anchored hinge).}
Although Eq.~\eqref{eq:v3-forget} averages over all answer positions while the
diagnostic of Eq.~\eqref{eq:margin-agg} selects a single maximum-entropy
position, the two are coupled through the reference. The hinge's per-token
activation $\sigma(\kappa(m_\theta(t) - m_{\text{ref}}(t)))$ weights each
position by its gap, and the gap is largest where the reference is most
uncertain, which is the property the entropy rule selects for. Measured at the
polish starting point on $100$ forget10 samples at 1B, the activation
rank-correlates with the model's per-token entropy at Spearman $+0.40$ for NPO
and $+0.51$ for UNDIAL, and the diagnostic token sits at the $0.76$ and $0.79$
hinge-weight quantile respectively, so for saturating bases the uniform hinge
already concentrates its pressure near the positions the diagnostic selects.
For GradDiff, whose content margins start near or below the reference, the
correlation is weak ($-0.13$) and the residual pressure falls on easier
shared-structure tokens instead, the base-dependent leak that the weighted
variants of the next remark target explicitly.

\paragraph{Remark (token-weighted hinge variants).}
Eq.~\eqref{eq:v3-forget} is stated with uniform token weights. Replacing the
uniform mean by a nonnegative weighting $w_t$ summing to one leaves
Theorem~\ref{thm:v3-cliff-cross} and its proof unchanged provided the
weighting does not break local Lipschitz continuity of the objective, which
holds in two natural cases, weights \emph{frozen} in $\theta$ (computed once
from the margin anchor, e.g.\ its high-entropy positions) and weights
\emph{continuous} in $\theta$ under stop-gradient (e.g.\ a softmax over the
current model's entropies). A hard top-$K$ selection re-evaluated on the
current model is excluded, as its weight vector jumps at entropy-rank ties
and the objective becomes discontinuous, the same selection-switching issue
that Theorem~\ref{thm:relearn-bound} avoids by freezing positions. In the
admitted cases the theorem is unchanged, since
$\partial \mathcal{L}_{\text{forget}}$ remains a nonnegative combination of
per-token hinge subgradients and the directional coercivity hypothesis is
stated over exactly that structure. Two refinements follow when the weighting
concentrates on high-entropy positions. First, the sufficient condition of
Step~1 scales with the \emph{weighted} mass of the super-threshold set, and
whenever the weighting concentrates on $S_\delta$ this mass exceeds the
uniform fraction $\mu(S_\delta)$, enlarging $G_f(\delta)$ and loosening the
dominance requirement. Second, the caveat of Step~1 (that $\Delta \ge
\delta$ alone does not populate $S_\delta$, because the diagnostic and the
hinge act at different positions) narrows, since the weighted hinge
concentrates its mass at the same maximum-entropy positions the diagnostic
selects. The uniform instantiation is the one all reported results use, and
App.~\ref{app:hinge-ew} evaluates both admitted variants empirically.

\paragraph{Remark (nonsmoothness of the margin).}
The per-token margin~\eqref{eq:margin-def} contains a max over competitors and is differentiable only
where each token's strongest competitor is unique, and ties form a measure-zero set in parameter space. All
statements extend across ties, with one care point. In Theorem~\ref{thm:v3-cliff-cross}, the extension
runs through the \emph{directional} coercivity form. A bound $u^\top g \ge G_f(\delta)$ with a
\emph{common} unit vector $u$ is linear in
$g$ and therefore survives the convex combinations and limits that define the Clarke subdifferential
(a norm lower bound alone would not, because convex combinations of large-norm gradients can vanish, as for
$|x|$ at $0$, and per-element $u$'s would not either, since combinations could escape each). The
theorem's hypothesis is stated in exactly this form (one $u_\theta$ bounding every
$g \in \partial \mathcal{L}_{\text{forget}}(\theta^\star)$), so every
$\xi \in \partial \mathcal{L}_{\text{polish}}(\theta^\star) \subseteq \nabla \mathcal{L}_{\text{native}}
+ \lambda_{\text{KL}} \nabla \mathcal{L}_{\text{KL}} + \partial \mathcal{L}_{\text{forget}}$
satisfies $u^\top \xi > 0$, so $0 \notin \partial \mathcal{L}_{\text{polish}}(\theta^\star)$ and the
contradiction is unchanged. Theorem~\ref{thm:relearn-bound} avoids both nonsmoothness sources by construction rather than by
limit-taking. It is stated for the $\beta$-smoothed, position-frozen functional $m_\beta$ at a
\emph{fixed} temperature, which is $C^\infty$, and the hard diagnostic enters only through the explicit
bracket $c_\beta = \log(V{-}1)/\beta$ that appears in the bound. Sharpening $\beta \to \infty$ (or a
softmax position selection toward the hard $\arg\max_t$) is deliberately \emph{not} taken, since the
smoothness constants would diverge in that limit, which is why the slack is carried additively
instead.

\begin{table*}[tp]
  \centering
  \scriptsize
  \setlength{\tabcolsep}{3pt}
  \caption{\textbf{Extended per-method panel} at 1B forget10. Each cell shows \emph{baseline} $\to$ \textsc{MC}. F.agg / ES / EM / MIA.agg / K20-LoRA / K20-FPFT lower is better; MU / KS-$p$ higher is better (KS-$p$ is the \textsc{MC} checkpoint's value, single column). Bold marks \textsc{MC} improvement. Extended companion to Table~\ref{tab:t1-main}.}
  \label{tab:a-t1-ext}
  \resizebox{\textwidth}{!}{%
  \begin{tabular}{ll ccccc cc cc}
    \toprule
    Base & Family & F.agg & ES & EM & MIA.agg & MU & KS-$p$ & priv.\,leak & K20-LoRA & K20-FPFT \\
    \midrule
    \texttt{GradDiff} & GradDiff & 0.357$\!\to\!$\textbf{0.018} & 0.122$\!\to\!$\textbf{0.002} & 0.737$\!\to\!$\textbf{0.050} & 0.791$\!\to\!$\textbf{0.228} & 0.470$\!\to\!$0.109 & 0.0000 & 33.220 & 0.429$\!\to\!$\textbf{0.158} & 0.408$\!\to\!$\textbf{0.088} \\
    \texttt{RMU} & GradDiff & 0.279$\!\to\!$\textbf{0.002} & 0.062$\!\to\!$\textbf{0.000} & 0.556$\!\to\!$\textbf{0.005} & 0.394$\!\to\!$0.466 & 0.571$\!\to\!$0.000 & 0.0021 & -2.278 & 0.471$\!\to\!$\textbf{0.029} & 0.398$\!\to\!$\textbf{0.031} \\
    \texttt{PDU} & GradDiff & 0.024$\!\to\!$\textbf{0.009} & 0.033$\!\to\!$\textbf{0.000} & 0.057$\!\to\!$\textbf{0.023} & 0.472$\!\to\!$\textbf{0.226} & 0.000$\!\to\!$\textbf{0.127} & 0.0000 & 36.470 & 0.428$\!\to\!$\textbf{0.152} & 0.437$\!\to\!$\textbf{0.027} \\
    \midrule
    \texttt{SimNPO} & SimNPO & 0.533$\!\to\!$\textbf{0.010} & 0.232$\!\to\!$\textbf{0.000} & 0.864$\!\to\!$\textbf{0.021} & 0.889$\!\to\!$\textbf{0.129} & 0.594$\!\to\!$0.179 & 0.0000 & 39.574 & 0.458$\!\to\!$\textbf{0.025} & 0.450$\!\to\!$\textbf{0.026} \\
    \midrule
    \texttt{NPO} & NPO & 0.296$\!\to\!$\textbf{0.041} & 0.095$\!\to\!$\textbf{0.002} & 0.637$\!\to\!$\textbf{0.117} & 0.716$\!\to\!$\textbf{0.165} & 0.300$\!\to\!$0.211 & 0.0000 & 29.486 & 0.348$\!\to\!$\textbf{0.246} & 0.342$\!\to\!$\textbf{0.246} \\
    \texttt{SAMNPO} & NPO & 0.762$\!\to\!$\textbf{0.104} & 0.558$\!\to\!$\textbf{0.017} & 0.951$\!\to\!$\textbf{0.299} & 0.980$\!\to\!$\textbf{0.241} & 0.599$\!\to\!$0.031 & 0.0085 & 22.465 & 0.414$\!\to\!$\textbf{0.201} & 0.401$\!\to\!$\textbf{0.146} \\
    \texttt{ILUNPO} & NPO & 0.300$\!\to\!$\textbf{0.027} & 0.095$\!\to\!$\textbf{0.004} & 0.646$\!\to\!$\textbf{0.068} & 0.723$\!\to\!$\textbf{0.106} & 0.337$\!\to\!$0.054 & 0.0001 & 29.536 & 0.343$\!\to\!$\textbf{0.205} & 0.335$\!\to\!$\textbf{0.122} \\
    \texttt{LATNPO} & NPO & 0.298$\!\to\!$\textbf{0.043} & 0.095$\!\to\!$\textbf{0.002} & 0.643$\!\to\!$\textbf{0.118} & 0.723$\!\to\!$\textbf{0.160} & 0.336$\!\to\!$0.162 & 0.0010 & 30.904 & 0.342$\!\to\!$\textbf{0.247} & 0.337$\!\to\!$\textbf{0.219} \\
    \texttt{RNANPO} & NPO & 0.277$\!\to\!$\textbf{0.005} & 0.091$\!\to\!$\textbf{0.000} & 0.635$\!\to\!$\textbf{0.015} & 0.735$\!\to\!$\textbf{0.114} & 0.323$\!\to\!$0.150 & 0.0000 & 38.174 & 0.366$\!\to\!$\textbf{0.165} & 0.357$\!\to\!$\textbf{0.041} \\
    \texttt{RSNPO} & NPO & 0.528$\!\to\!$\textbf{0.114} & 0.194$\!\to\!$\textbf{0.011} & 0.835$\!\to\!$\textbf{0.353} & 0.948$\!\to\!$\textbf{0.254} & 0.500$\!\to\!$0.070 & 0.2705 & 9.143 & 0.442$\!\to\!$\textbf{0.195} & 0.427$\!\to\!$\textbf{0.154} \\
    \texttt{SWANPO} & NPO & 0.300$\!\to\!$\textbf{0.052} & 0.094$\!\to\!$\textbf{0.000} & 0.646$\!\to\!$\textbf{0.144} & 0.724$\!\to\!$\textbf{0.268} & 0.338$\!\to\!$0.110 & 0.0006 & 22.267 & 0.349$\!\to\!$\textbf{0.253} & 0.339$\!\to\!$\textbf{0.145} \\
    \texttt{CRNPO} & NPO & 0.890$\!\to\!$\textbf{0.026} & 0.810$\!\to\!$\textbf{0.002} & 0.982$\!\to\!$\textbf{0.096} & 0.998$\!\to\!$\textbf{0.228} & 0.592$\!\to\!$0.040 & 0.7934 & 19.513 & 0.579$\!\to\!$\textbf{0.206} & 0.572$\!\to\!$\textbf{0.181} \\
    \midrule
    \texttt{UNDIAL} & Distil & 0.505$\!\to\!$\textbf{0.283} & 0.220$\!\to\!$\textbf{0.060} & 0.841$\!\to\!$\textbf{0.637} & 0.938$\!\to\!$\textbf{0.651} & 0.574$\!\to\!$0.207 & 0.0085 & -57.751 & 0.381$\!\to\!$\textbf{0.282} & 0.390$\!\to\!$\textbf{0.284} \\
    \texttt{JensUn} & Distil & 0.708$\!\to\!$\textbf{0.039} & 0.680$\!\to\!$\textbf{0.008} & 0.945$\!\to\!$\textbf{0.117} & 0.995$\!\to\!$\textbf{0.210} & 0.582$\!\to\!$0.090 & 0.0043 & 35.438 & 0.447$\!\to\!$\textbf{0.147} & 0.423$\!\to\!$\textbf{0.079} \\
    \bottomrule
  \end{tabular}%
  }
\end{table*}

\section{Loss-saturation verification}
\label{app:loss-saturation}

We verify the mechanism behind cliff termination for representatives of all three loss families in the 14-method panel. The NPO family and JensUn satisfy Definition~\ref{def:forget-saturating}, while the GradDiff family and UNDIAL, whose forget gradients are bounded but non-vanishing, are covered by assuming the diagnostic floor of Asm.~\ref{asm:overlap} directly.

\paragraph{GradDiff family (floor assumed directly).} The forget objective ascends the forget NLL, so the loss being minimized is $\mathcal{L}_f = \log p_\theta(y_f \mid x_f)$. Its gradient in logit space is $(e_{y_f} - p_\theta)$, where $e_{y_f}$ is the unit vector at the gold token. As $p_\theta(y_f) \to 0$, the gold-direction component of this gradient approaches magnitude $1$, so the forget gradient is \emph{bounded} (by $\sqrt{2} \|\partial z/\partial \theta\|$) but \emph{does not vanish}, and GradDiff is not token-saturating. Because the forget push stays $O(1)$ rather than vanishing, stationarity alone does not force a floor, since the balance $\nabla \mathcal{L}_f(\theta^\star) = -\lambda \nabla \mathcal{L}_r(\theta^\star)$ is satisfiable at arbitrarily small forget probability unless the retain pull dominates the $O(1)$ forget push below some level, a retain-coercivity property we do not derive. We therefore cover the GradDiff family by \emph{assuming} the diagnostic floor of Asm.~\ref{asm:overlap} for it directly. Like the rest of the panel, the assumption is checked per checkpoint (App.~\ref{app:logodds-verify}), and its failure mode is visible in the data, where the GradDiff cell at 8B terminates \emph{below} $m_{\text{ref}}$ ($\Delta \approx -0.5$) with collapsed general utility (real-author and world-fact ROUGE $\le 0.09$), exactly the regime where the retain pull fails to hold the level, degenerate in the same sense as GradAscent (App.~\ref{app:gradascent}). RMU and PDU inherit this structure with additional representation-space regularizers.

\paragraph{NPO family.} The per-sample preference loss (as published and as implemented in our
trainer) applies the sigmoid to the \emph{sequence-summed} log-ratio,
$\ell_f = -\tfrac{2}{\beta}\log\sigma\big({-}\beta \sum_s r_s\big)$ (NPO's own hyperparameter
$\beta$, unrelated to the smoothing temperature of Thm.~\ref{thm:relearn-bound}) with
$r_s := \log p_s - \log p_{0,s}$ and $p_{0,s}$ the base model's probability. The gold-logit partial is
$\partial \ell_f / \partial z_{y_t} = 2\,\sigma\big(\beta \sum_s r_s\big)(1 - p_t)$, a single sigmoid
envelope shared across positions. Unlearning only suppresses gold probabilities relative to $\theta_0$
on $\mathcal{D}_f$ ($r_s \le 0$ for all answer tokens, which we observe at every checkpoint), so
$\sum_s r_s \le r_t$ and hence
$|\partial \ell_f / \partial z_{y_t}| \le 2\,\sigma(\beta r_t) \le 2 (p_t / p_{0,t})^{\beta}$. With
$\theta_0$ memorizing the content tokens ($p_{0,t}$ near $1$), this is
Definition~\ref{def:forget-saturating} with $\rho(p) \propto p^\beta$. The shared envelope makes the
saturation \emph{stronger} than a per-token version. Once any content token is deeply suppressed, the
sample's entire forget gradient, easy tokens included, collapses with it. SAMNPO, ILUNPO,
LATNPO, RNANPO, RSNPO, SWANPO, CRNPO add adversarial perturbations, sharpness-aware updates, or
weight-space smoothing that do not change the leading-order saturation. SimNPO replaces the reference
dependence but keeps the bounded sigmoid envelope.

\paragraph{Distillation family.} The two distillation losses behave differently under
Definition~\ref{def:forget-saturating} and must be treated separately.

\emph{UNDIAL (anchor-saturating, floor assumed directly).} As implemented (and as published), the
per-position loss is the cross-entropy of the student against a \emph{fixed} flattened teacher
target $q = \mathrm{softmax}(z_{\text{teacher}} - \beta_{\text{UD}}\, e_{y_t})$, whose gold-logit
partial is the standard softmax-CE difference
\[
\frac{\partial \phi}{\partial z_{y_t}} = p_t - q_t .
\]
As $p_t \to 0$ this tends to $-q_t \neq 0$ ($q_t$ is a positive constant of the frozen teacher), so
\emph{UNDIAL is not token-saturating}, since no envelope $\rho$ with $\rho(p) \to 0$ can bound the
partial. Its gradient instead vanishes as $p_\theta \to q$, a saturation at the \emph{anchor}, not
at zero probability. Like the GradDiff family above, UNDIAL therefore carries a bounded,
non-vanishing forget-side push and is covered by assuming the diagnostic floor of
Asm.~\ref{asm:overlap} directly. The anchor-saturation mechanism makes the floor especially natural
here (the loss actively pins margins at the flattened target's level) and is consistent with
UNDIAL being the flattest-target base and the only method contributing non-crossing cells under
\textsc{MC} polish (\S\ref{sec:exp-main}).

\emph{JensUn (token-saturating).} JensUn minimizes the (bounded, symmetric) Jensen--Shannon
divergence to a target that places vanishing mass on the gold token. Writing
$m = \tfrac12(p_\theta + p_{\text{tgt}})$, direct computation via
$\partial p_v / \partial z_y = p_v(\delta_{vy} - p_y)$ gives a gold-logit partial of the form
$\tfrac12 p_t\big[\log(p_t/m_t) - \mathrm{KL}(p_\theta \| m)\big] = p_t \cdot O(1 + |\log p_t|)$,
which vanishes as $p_t \to 0$, matching Definition~\ref{def:forget-saturating} with
$\rho(p) = p\,(|\log p| + \log V)$ up to constants.

\begin{table}[htbp]
  \centering
  \scriptsize
  \setlength{\tabcolsep}{4pt}
  \caption{\textbf{Per-base \textsc{MC} cliff gap} $\Delta(\hat\theta_{\text{MC}})$ (pipeline evaluator, diagnostic of Eq.~\eqref{eq:margin-agg}; reference margins $-3.64$, $-3.59$, $-4.01$, $-4.17$, $-4.93$ per column). Negative crosses the cliff; the only positive cells are UNDIAL (4 of its 5 panels), giving the $66/70$ count of \S\ref{sec:m-thm2}.}
  \label{tab:a-t-mc-delta}
  \begin{tabular}{lccccc}
    \toprule
    Base & 1B f01 & 1B f05 & 1B f10 & 3B f10 & 8B f10 \\
    \midrule
    \texttt{GradDiff} & $-123$ & $-74$ & $-91$ & $-82$ & $-247$ \\
    \texttt{RMU} & $-116$ & $-103$ & $-82$ & $-76$ & $-138$ \\
    \texttt{PDU} & $-129$ & $-94$ & $-96$ & $-109$ & $-243$ \\
    \texttt{NPO} & $-11$ & $-13$ & $-4.0$ & $-13$ & $-19$ \\
    \texttt{SAMNPO} & $-9.1$ & $-10.0$ & $-2.0$ & $-13$ & $-23$ \\
    \texttt{ILUNPO} & $-5.9$ & $-13$ & $-15$ & $-18$ & $-23$ \\
    \texttt{LATNPO} & $-6.6$ & $-12$ & $-5.9$ & $-24$ & $-13$ \\
    \texttt{RNANPO} & $-6.5$ & $-5.6$ & $-15$ & $-16$ & $-27$ \\
    \texttt{RSNPO} & $-9.0$ & $-19$ & $-1.0$ & $-18$ & $-14$ \\
    \texttt{SWANPO} & $-33$ & $-10$ & $-4.6$ & $-24$ & $-22$ \\
    \texttt{CRNPO} & $-11$ & $-7.7$ & $-0.9$ & $-20$ & $-13$ \\
    \texttt{SimNPO} & $-85$ & $-62$ & $-43$ & $-60$ & $-65$ \\
    \texttt{UNDIAL} & $-0.2$ & $+1.6$ & $+1.8$ & $+2.3$ & $+3.0$ \\
    \texttt{JensUn} & $-117$ & $-91$ & $-73$ & $-92$ & $-124$ \\
    \midrule
    \emph{Mean} & $-47.2$ & $-36.6$ & $-30.8$ & $-40.1$ & $-69.1$ \\
    \bottomrule
  \end{tabular}
\end{table}

\paragraph{Empirical check.} At 1B \texttt{forget10}, all 14 baselines terminate with diagnostic-position margins in $[-2.33, -0.14]$ (at least $1.68$ above $m_{\text{ref}} = -4.01$) while their \emph{all-token} mean gold probability remains high ($0.29$--$0.83$). For the token-saturating members of the panel this is precisely the signature of \emph{token-wise} saturation. Content tokens are suppressed to low probability (diagnostic-position gold probability $0.8\%$--$15\%$), where the per-token forget envelope vanishes, while easy continuation tokens remain confident and are irrelevant to the cliff, and the bounded-gradient members (GradDiff family, UNDIAL) terminate in the same band, consistent with the directly assumed floor. Per-size values, including the boundary GradDiff cell at 8B, are in App.~\ref{app:logodds-verify}.

\section{Empirical verification of the average-log-odds premise}
\label{app:logodds-verify}

Table~\ref{tab:a-t-logodds} reports, for each of the 14 baselines at 1B/3B/8B \texttt{forget10}, the
mean diagnostic-position gold log-odds $\ell(\hat\theta)$, the measured margin diagnostic
$m_{\hat\theta}(\mathcal{D}_f)$, and whether the average-form bound of App.~\ref{app:proofs} certifies
the cliff ($\ell > m_{\text{ref}}$), all from a single forward pass per checkpoint over the same 200
forget samples and encoding used by the evaluation pipeline (the probe reproduces the cached
$m_{\text{ref}} = -4.01$ at 1B exactly).

Four observations. \emph{(i)}~The unconditional inequality
$m_{\hat\theta}(\mathcal{D}_f) \ge \ell(\hat\theta)$ holds in every one of the $42$ populated cells,
as the competitor-mass lemma requires.
\emph{(ii)}~The average-form premise
$\ell(\hat\theta) > m_{\text{ref}}$ certifies the cliff for $9/14$ methods at 1B, $13/14$ at 3B, and
$12/14$ at 8B ($34/42$ overall). Uncertified cells are loose rather than violated. Diagnostic
positions are maximum-entropy positions, where competitor mass is diffuse. The measured
top-competitor share of non-gold mass $e^{\ell - m}$ (approximately the top-competitor probability
at the small diagnostic-position gold probabilities observed here) is $0.02$--$0.11$ at 1B, far
below the worst-case $1 - p \approx 1$ that the log-odds relaxation assumes, and the measured $\Delta$ remains positive
in every cell except the utility-collapsed GradDiff cell at 8B (\S\ref{sec:m-cliff}).
\emph{(iii)}~The worst-case per-token floor is empirically vacuous (minimum token gold probabilities
reach $10^{-32}$--$10^{-5}$ across the 42 cells), confirming that the average form, not the worst-case floor, is the
operative premise, which is why Theorem~\ref{thm:kkt-cliff}'s constant should be read through
$\ell(\hat\theta)$. \emph{(iv)}~The certification pattern is not a small-model artifact, as the certified
fraction \emph{rises} with model size on the NPO family, whose envelopes saturate closer to the
reference at 1B.

\section{Why GradAscent is excluded from the 14-method panel}
\label{app:gradascent}

Vanilla gradient ascent minimizes $\mathcal{L}_f = -\mathcal{L}_{\mathrm{NLL}}^{(f)} = \log p_\theta(y_f \mid x_f)$
with no anchor and no retain balance. This objective is \emph{unbounded below} and has no stationary
point. It is not token-saturating in the sense of Def.~\ref{def:forget-saturating} (its gold-logit
partial $\partial \ell_f / \partial z_{y_t} = 1 - p_t$ keeps magnitude $\Theta(1)$ even as
$p_t \to 0$, so no envelope $\rho$ with $\rho(p) \to 0$ bounds it), and with no counter-term the iterates drive every gold
probability toward $0$ without limit. Empirically the per-token margin runs to $-\infty$ and retain
NLL explodes within tens of steps (Fig.~\ref{fig:traj}, panels a--b), MU collapses to $0.000$ at the evaluated checkpoint (Table~\ref{tab:t3-ablation}, GradAscent row), and the model degenerates to noise output. GradAscent is thus the same bounded, non-vanishing forget force as
GradDiff but \emph{without} the retain regularizer that pins GradDiff at a cliff-side stationary point.
Lacking both that balance and an anchor, it does not converge at all. It is a degenerate boundary case,
not a method, and we use it only as a trajectory counterexample.

\section{Extending Theorem~\ref{thm:relearn-bound} to full-parameter attackers}
\label{app:fpft-bound}

Theorem~\ref{thm:relearn-bound} is stated for an attack \emph{class} defined by a step budget $H$ and
a region $\Theta_A$, so full-parameter fine-tuning requires no separate derivation. An FPFT attacker
with (possibly clipped) step norms $\|\delta^{(n)}\| \le \eta H_{\text{full}}$ and iterates in
$\Theta_A$ belongs to $\mathcal{A}(\eta, N, H_{\text{full}}, \Theta_A)$, and the bound applies verbatim
with $H \mapsto H_{\text{full}}$,
\[
\Delta(\tilde\theta_N)
 \le
\Delta(\hat\theta) + \eta N\, G_m H_{\text{full}} + \tfrac{1}{2} L_m \eta^2 N\, H_{\text{full}}^2
+ c_\beta .
\]
Whether FPFT realizes larger steps than a rank-$r$ adapter is an empirical property of the runs, not a
consequence of the theorem. In our sweeps the FPFT attacker is the strongest
(App.~\ref{app:strong-attackers}), consistent with larger realized increments, and \textsc{MC}'s
survival there is consistent with its starting depth $-\Delta(\hat\theta_{\text{MC}})$ exceeding the
corresponding class budget.

\section{Cliff observation figure (per size and per forget tier)}
\label{app:cliff-figure}

Figure~\ref{fig:cliff} (main text) shows per-method forget-set margins at 1B/3B/8B for the TOFU forget10 tier. Figure~\ref{fig:cliff-tiers} below confirms the same cliff pattern at forget01 and forget05, so the observation is not specific to the largest tier.

\section{Per-base cliff numbers and retain-hinge versus KL-probe ablation}
\label{app:cliff-numbers}
\label{app:v2-vs-v3}

Table below extends the head-to-head panel at 1B forget10 with forget aggregate, ES, EM, MIA aggregate, MU, KS-$p$, privacy leakage, and K20 recovery under both LoRA and FPFT, baseline $\to$ \textsc{MC}. Table~\ref{tab:a-t-mc-delta} reports the per-base \textsc{MC} cliff gap $\Delta(\hat\theta_{\text{MC}})$ across all five completed TOFU panels, backing the $66/70$ crossing count, the panel means quoted in \S\ref{sec:m-thm2} (e.g.\ $-30.8$ at 1B \texttt{forget10}), the four UNDIAL exception values, and the CRNPO production-polish gap $-0.9$ cited in Fig.~\ref{fig:traj}. The retain-hinge versus KL-probe ablation (\S\ref{sec:m-v3}) is the two-sided retain-hinge row of Table~\ref{tab:t3-ablation}, and per-base detail is in App.~\ref{app:t3-full}.
\section{Per-cell details for T2}
\label{app:t2-per-cell}
Each sub-table below lists the per-method base$\to$\textsc{MC} numbers that the corresponding row of Table~\ref{tab:t2-robust} compresses to a single win-rate cell. The MUSE-News row uses the 13-method panel (CRNPO omitted) on Llama-2-7B-hf described in \S\ref{sec:exp-robust}.

\begin{table}[htbp]
  \centering
  \scriptsize
  \setlength{\tabcolsep}{3pt}
  \caption{Per-method backing for T2 row \emph{Multi-seed (5 base $\times$ 3 seed, 1B \texttt{forget10})}. Numbers shown are seed-0 (canonical). The T2 row aggregates the 5 representative bases $\times$ 3 seeds under K20-LoRA (15 cells); K20-FPFT was run at seed 0 only, hence its 5-cell count.}
  \label{tab:a-t2-multi-seed}
  \begin{tabular}{l cccc}
    \toprule
    Base & F.agg ($\downarrow$) & MU ($\uparrow$) & MIA.agg ($\downarrow$) & K20-LoRA ($\downarrow$) \\
    \midrule
    \texttt{GradDiff} & 0.357$\!\to\!$\textbf{0.018} & 0.470$\!\to\!$0.109 & 0.791$\!\to\!$\textbf{0.228} & 0.429$\!\to\!$\textbf{0.158} \\
    \texttt{RMU} & 0.279$\!\to\!$\textbf{0.002} & 0.571$\!\to\!$0.000 & 0.394$\!\to\!$0.466 & 0.471$\!\to\!$\textbf{0.029} \\
    \texttt{PDU} & 0.024$\!\to\!$\textbf{0.009} & 0.000$\!\to\!$\textbf{0.127} & 0.472$\!\to\!$\textbf{0.226} & 0.428$\!\to\!$\textbf{0.152} \\
    \texttt{NPO} & 0.296$\!\to\!$\textbf{0.041} & 0.300$\!\to\!$0.211 & 0.716$\!\to\!$\textbf{0.165} & 0.348$\!\to\!$\textbf{0.246} \\
    \texttt{SimNPO} & 0.533$\!\to\!$\textbf{0.010} & 0.594$\!\to\!$0.179 & 0.889$\!\to\!$\textbf{0.129} & 0.458$\!\to\!$\textbf{0.025} \\
    \texttt{SAMNPO} & 0.762$\!\to\!$\textbf{0.104} & 0.599$\!\to\!$0.031 & 0.980$\!\to\!$\textbf{0.241} & 0.414$\!\to\!$\textbf{0.201} \\
    \texttt{CRNPO} & 0.890$\!\to\!$\textbf{0.026} & 0.592$\!\to\!$0.040 & 0.998$\!\to\!$\textbf{0.228} & 0.579$\!\to\!$\textbf{0.206} \\
    \texttt{RSNPO} & 0.528$\!\to\!$\textbf{0.114} & 0.500$\!\to\!$0.070 & 0.948$\!\to\!$\textbf{0.254} & 0.442$\!\to\!$\textbf{0.195} \\
    \texttt{SWANPO} & 0.300$\!\to\!$\textbf{0.052} & 0.338$\!\to\!$0.110 & 0.724$\!\to\!$\textbf{0.268} & 0.349$\!\to\!$\textbf{0.253} \\
    \texttt{LATNPO} & 0.298$\!\to\!$\textbf{0.043} & 0.336$\!\to\!$0.162 & 0.723$\!\to\!$\textbf{0.160} & 0.342$\!\to\!$\textbf{0.247} \\
    \texttt{ILUNPO} & 0.300$\!\to\!$\textbf{0.027} & 0.337$\!\to\!$0.054 & 0.723$\!\to\!$\textbf{0.106} & 0.343$\!\to\!$\textbf{0.205} \\
    \texttt{RNANPO} & 0.277$\!\to\!$\textbf{0.005} & 0.323$\!\to\!$0.150 & 0.735$\!\to\!$\textbf{0.114} & 0.366$\!\to\!$\textbf{0.165} \\
    \texttt{UNDIAL} & 0.505$\!\to\!$\textbf{0.283} & 0.574$\!\to\!$0.207 & 0.938$\!\to\!$\textbf{0.651} & 0.381$\!\to\!$\textbf{0.282} \\
    \texttt{JensUn} & 0.708$\!\to\!$\textbf{0.039} & 0.582$\!\to\!$0.090 & 0.995$\!\to\!$\textbf{0.210} & 0.447$\!\to\!$\textbf{0.147} \\
    \bottomrule
  \end{tabular}
\end{table}

\begin{table}[htbp]
  \centering
  \scriptsize
  \setlength{\tabcolsep}{3pt}
  \caption{Per-method backing for T2 row \emph{1B \texttt{forget01} (14 base)}. Calibration uses retain-99 as $\theta_{\text{ref}}$.}
  \label{tab:a-t2-forget01}
  \begin{tabular}{l cccc}
    \toprule
    Base & F.agg ($\downarrow$) & MU ($\uparrow$) & MIA.agg ($\downarrow$) & K20-LoRA ($\downarrow$) \\
    \midrule
    \texttt{GradDiff} & 0.480$\!\to\!$\textbf{0.001} & 0.282$\!\to\!$0.169 & 0.907$\!\to\!$\textbf{0.150} & 0.483$\!\to\!$\textbf{0.140} \\
    \texttt{RMU} & 0.471$\!\to\!$\textbf{0.007} & 0.274$\!\to\!$0.105 & 0.888$\!\to\!$\textbf{0.184} & 0.495$\!\to\!$\textbf{0.009} \\
    \texttt{PDU} & 0.544$\!\to\!$\textbf{0.000} & 0.270$\!\to\!$0.000 & 0.923$\!\to\!$\textbf{0.188} & 0.524$\!\to\!$\textbf{0.000} \\
    \texttt{NPO} & 0.433$\!\to\!$\textbf{0.002} & 0.293$\!\to\!$0.122 & 0.894$\!\to\!$\textbf{0.000} & 0.410$\!\to\!$\textbf{0.196} \\
    \texttt{SimNPO} & 0.588$\!\to\!$\textbf{0.003} & 0.274$\!\to\!$0.074 & 0.970$\!\to\!$\textbf{0.046} & 0.612$\!\to\!$\textbf{0.078} \\
    \texttt{SAMNPO} & 0.515$\!\to\!$\textbf{0.008} & 0.263$\!\to\!$0.115 & 0.932$\!\to\!$\textbf{0.000} & 0.485$\!\to\!$\textbf{0.019} \\
    \texttt{CRNPO} & 0.681$\!\to\!$\textbf{0.013} & 0.269$\!\to\!$0.064 & 0.980$\!\to\!$\textbf{0.002} & 0.632$\!\to\!$\textbf{0.024} \\
    \texttt{RSNPO} & 0.604$\!\to\!$\textbf{0.013} & 0.287$\!\to\!$0.088 & 0.978$\!\to\!$\textbf{0.003} & 0.616$\!\to\!$\textbf{0.137} \\
    \texttt{SWANPO} & 0.433$\!\to\!$\textbf{0.000} & 0.293$\!\to\!$0.080 & 0.894$\!\to\!$\textbf{0.000} & 0.414$\!\to\!$\textbf{0.168} \\
    \texttt{LATNPO} & 0.432$\!\to\!$\textbf{0.004} & 0.293$\!\to\!$0.087 & 0.892$\!\to\!$\textbf{0.021} & 0.419$\!\to\!$\textbf{0.168} \\
    \texttt{ILUNPO} & 0.433$\!\to\!$\textbf{0.046} & 0.294$\!\to\!$0.125 & 0.893$\!\to\!$\textbf{0.003} & 0.430$\!\to\!$\textbf{0.258} \\
    \texttt{RNANPO} & 0.437$\!\to\!$\textbf{0.010} & 0.293$\!\to\!$0.088 & 0.883$\!\to\!$\textbf{0.000} & 0.427$\!\to\!$\textbf{0.191} \\
    \texttt{UNDIAL} & 0.459$\!\to\!$\textbf{0.249} & 0.246$\!\to\!$0.144 & 0.938$\!\to\!$\textbf{0.534} & 0.431$\!\to\!$\textbf{0.327} \\
    \texttt{JensUn} & 0.604$\!\to\!$\textbf{0.015} & 0.280$\!\to\!$0.067 & 0.974$\!\to\!$\textbf{0.157} & 0.551$\!\to\!$\textbf{0.034} \\
    \bottomrule
  \end{tabular}
\end{table}

\begin{table}[htbp]
  \centering
  \scriptsize
  \setlength{\tabcolsep}{3pt}
  \caption{Per-method backing for T2 row \emph{1B \texttt{forget05} (14 base)}. Calibration uses retain-95 as $\theta_{\text{ref}}$.}
  \label{tab:a-t2-forget05}
  \begin{tabular}{l cccc}
    \toprule
    Base & F.agg ($\downarrow$) & MU ($\uparrow$) & MIA.agg ($\downarrow$) & K20-LoRA ($\downarrow$) \\
    \midrule
    \texttt{GradDiff} & 0.388$\!\to\!$\textbf{0.001} & 0.291$\!\to\!$0.033 & 0.797$\!\to\!$\textbf{0.400} & 0.371$\!\to\!$\textbf{0.083} \\
    \texttt{RMU} & 0.450$\!\to\!$\textbf{0.000} & 0.245$\!\to\!$0.084 & 0.867$\!\to\!$\textbf{0.210} & 0.644$\!\to\!$\textbf{0.009} \\
    \texttt{PDU} & 0.256$\!\to\!$\textbf{0.000} & 0.238$\!\to\!$0.087 & 0.731$\!\to\!$\textbf{0.234} & 0.596$\!\to\!$\textbf{0.005} \\
    \texttt{NPO} & 0.350$\!\to\!$\textbf{0.013} & 0.340$\!\to\!$0.076 & 0.750$\!\to\!$\textbf{0.191} & 0.267$\!\to\!$\textbf{0.248} \\
    \texttt{SimNPO} & 0.565$\!\to\!$\textbf{0.000} & 0.280$\!\to\!$0.000 & 0.975$\!\to\!$\textbf{0.356} & 0.489$\!\to\!$\textbf{0.014} \\
    \texttt{SAMNPO} & 0.515$\!\to\!$\textbf{0.004} & 0.242$\!\to\!$0.021 & 0.951$\!\to\!$\textbf{0.264} & 0.610$\!\to\!$\textbf{0.215} \\
    \texttt{CRNPO} & 0.659$\!\to\!$\textbf{0.016} & 0.279$\!\to\!$0.028 & 0.985$\!\to\!$\textbf{0.262} & 0.586$\!\to\!$\textbf{0.143} \\
    \texttt{RSNPO} & 0.552$\!\to\!$\textbf{0.005} & 0.281$\!\to\!$0.035 & 0.975$\!\to\!$\textbf{0.260} & 0.733$\!\to\!$\textbf{0.256} \\
    \texttt{SWANPO} & 0.350$\!\to\!$\textbf{0.023} & 0.340$\!\to\!$0.062 & 0.750$\!\to\!$\textbf{0.076} & 0.461$\!\to\!$\textbf{0.357} \\
    \texttt{LATNPO} & 0.352$\!\to\!$\textbf{0.021} & 0.341$\!\to\!$0.139 & 0.752$\!\to\!$\textbf{0.157} & 0.472$\!\to\!$\textbf{0.314} \\
    \texttt{ILUNPO} & 0.350$\!\to\!$\textbf{0.013} & 0.339$\!\to\!$0.049 & 0.751$\!\to\!$\textbf{0.092} & 0.473$\!\to\!$\textbf{0.248} \\
    \texttt{RNANPO} & 0.350$\!\to\!$\textbf{0.033} & 0.339$\!\to\!$0.043 & 0.755$\!\to\!$\textbf{0.085} & 0.456$\!\to\!$\textbf{0.377} \\
    \texttt{UNDIAL} & 0.410$\!\to\!$\textbf{0.280} & 0.201$\!\to\!$0.129 & 0.938$\!\to\!$\textbf{0.625} & 0.385$\!\to\!$\textbf{0.305} \\
    \texttt{JensUn} & 0.572$\!\to\!$\textbf{0.009} & 0.269$\!\to\!$0.000 & 0.980$\!\to\!$\textbf{0.329} & 0.737$\!\to\!$\textbf{0.025} \\
    \bottomrule
  \end{tabular}
\end{table}

\begin{table}[htbp]
  \centering
  \scriptsize
  \setlength{\tabcolsep}{3pt}
  \caption{Per-method backing for T2 row \emph{3B \texttt{forget10} (14 base)}. Calibration uses Llama-3.2-3B-Instruct retain-90.}
  \label{tab:a-t2-3b}
  \begin{tabular}{l cccc}
    \toprule
    Base & F.agg ($\downarrow$) & MU ($\uparrow$) & MIA.agg ($\downarrow$) & K20-LoRA ($\downarrow$) \\
    \midrule
    \texttt{GradDiff} & 0.365$\!\to\!$\textbf{0.041} & 0.315$\!\to\!$0.128 & 0.757$\!\to\!$\textbf{0.388} & 0.431$\!\to\!$\textbf{0.230} \\
    \texttt{RMU} & 0.389$\!\to\!$\textbf{0.012} & 0.354$\!\to\!$0.165 & 0.757$\!\to\!$\textbf{0.214} & 0.469$\!\to\!$\textbf{0.035} \\
    \texttt{PDU} & 0.160$\!\to\!$\textbf{0.001} & 0.116$\!\to\!$0.047 & 0.631$\!\to\!$\textbf{0.428} & 0.396$\!\to\!$\textbf{0.052} \\
    \texttt{NPO} & 0.313$\!\to\!$\textbf{0.008} & 0.127$\!\to\!$0.067 & 0.770$\!\to\!$\textbf{0.035} & 0.302$\!\to\!$\textbf{0.260} \\
    \texttt{SimNPO} & 0.664$\!\to\!$\textbf{0.010} & 0.348$\!\to\!$0.084 & 0.966$\!\to\!$\textbf{0.438} & 0.541$\!\to\!$\textbf{0.003} \\
    \texttt{SAMNPO} & 0.620$\!\to\!$\textbf{0.019} & 0.380$\!\to\!$0.060 & 0.972$\!\to\!$\textbf{0.454} & 0.447$\!\to\!$\textbf{0.074} \\
    \texttt{CRNPO} & 0.767$\!\to\!$\textbf{0.005} & 0.352$\!\to\!$0.056 & 0.982$\!\to\!$\textbf{0.469} & 0.670$\!\to\!$\textbf{0.210} \\
    \texttt{RSNPO} & 0.672$\!\to\!$\textbf{0.009} & 0.377$\!\to\!$0.080 & 0.974$\!\to\!$\textbf{0.606} & 0.549$\!\to\!$\textbf{0.037} \\
    \texttt{SWANPO} & 0.315$\!\to\!$\textbf{0.008} & 0.124$\!\to\!$0.074 & 0.774$\!\to\!$\textbf{0.031} & 0.296$\!\to\!$\textbf{0.239} \\
    \texttt{LATNPO} & 0.313$\!\to\!$\textbf{0.004} & 0.127$\!\to\!$0.071 & 0.772$\!\to\!$\textbf{0.141} & 0.310$\!\to\!$\textbf{0.227} \\
    \texttt{ILUNPO} & 0.315$\!\to\!$\textbf{0.008} & 0.123$\!\to\!$0.095 & 0.774$\!\to\!$\textbf{0.088} & 0.296$\!\to\!$\textbf{0.043} \\
    \texttt{RNANPO} & 0.308$\!\to\!$\textbf{0.009} & 0.130$\!\to\!$0.106 & 0.773$\!\to\!$\textbf{0.025} & 0.304$\!\to\!$\textbf{0.262} \\
    \texttt{UNDIAL} & 0.452$\!\to\!$\textbf{0.259} & 0.280$\!\to\!$\textbf{0.344} & 0.947$\!\to\!$\textbf{0.602} & 0.400$\!\to\!$\textbf{0.285} \\
    \texttt{JensUn} & 0.677$\!\to\!$\textbf{0.006} & 0.341$\!\to\!$0.028 & 0.980$\!\to\!$\textbf{0.396} & 0.561$\!\to\!$\textbf{0.159} \\
    \bottomrule
  \end{tabular}
\end{table}

\begin{table}[htbp]
  \centering
  \scriptsize
  \setlength{\tabcolsep}{3pt}
  \caption{Per-method backing for T2 row \emph{8B \texttt{forget10} (13 base; RSNPO 8B reported in Fig.~\ref{fig:cliff} only)}. Calibration uses Llama-3.1-8B-Instruct retain-90.}
  \label{tab:a-t2-8b}
  \begin{tabular}{l cccc}
    \toprule
    Base & F.agg ($\downarrow$) & MU ($\uparrow$) & MIA.agg ($\downarrow$) & K20-LoRA ($\downarrow$) \\
    \midrule
    \texttt{GradDiff} & 0.317$\!\to\!$\textbf{0.002} & 0.001$\!\to\!$\textbf{0.123} & 0.701$\!\to\!$\textbf{0.362} & 0.278$\!\to\!$\textbf{0.030} \\
    \texttt{RMU} & 0.224$\!\to\!$\textbf{0.001} & 0.108$\!\to\!$0.081 & 0.489$\!\to\!$\textbf{0.311} & 0.447$\!\to\!$\textbf{0.048} \\
    \texttt{PDU} & 0.370$\!\to\!$\textbf{0.004} & 0.148$\!\to\!$0.001 & 0.799$\!\to\!$\textbf{0.473} & 0.516$\!\to\!$\textbf{0.258} \\
    \texttt{NPO} & 0.304$\!\to\!$\textbf{0.011} & 0.074$\!\to\!$\textbf{0.093} & 0.651$\!\to\!$\textbf{0.204} & 0.251$\!\to\!$\textbf{0.053} \\
    \texttt{SimNPO} & 0.719$\!\to\!$\textbf{0.009} & 0.089$\!\to\!$0.035 & 0.968$\!\to\!$\textbf{0.488} & 0.474$\!\to\!$\textbf{0.000} \\
    \texttt{SAMNPO} & 0.794$\!\to\!$\textbf{0.079} & 0.133$\!\to\!$\textbf{0.147} & 0.990$\!\to\!$\textbf{0.267} & 0.488$\!\to\!$\textbf{0.247} \\
    \texttt{CRNPO} & 0.886$\!\to\!$\textbf{0.082} & 0.097$\!\to\!$\textbf{0.118} & 0.998$\!\to\!$\textbf{0.387} & 0.614$\!\to\!$\textbf{0.177} \\
    \texttt{SWANPO} & 0.304$\!\to\!$\textbf{0.013} & 0.075$\!\to\!$0.007 & 0.649$\!\to\!$\textbf{0.286} & 0.248$\!\to\!$\textbf{0.023} \\
    \texttt{LATNPO} & 0.305$\!\to\!$\textbf{0.015} & 0.079$\!\to\!$0.046 & 0.662$\!\to\!$\textbf{0.054} & 0.254$\!\to\!$\textbf{0.214} \\
    \texttt{ILUNPO} & 0.299$\!\to\!$\textbf{0.014} & 0.077$\!\to\!$0.074 & 0.637$\!\to\!$\textbf{0.096} & 0.245$\!\to\!$\textbf{0.029} \\
    \texttt{RNANPO} & 0.314$\!\to\!$\textbf{0.007} & 0.075$\!\to\!$0.054 & 0.679$\!\to\!$\textbf{0.105} & 0.256$\!\to\!$\textbf{0.025} \\
    \texttt{UNDIAL} & 0.451$\!\to\!$\textbf{0.217} & 0.146$\!\to\!$0.080 & 0.977$\!\to\!$\textbf{0.568} & 0.346$\!\to\!$\textbf{0.239} \\
    \texttt{JensUn} & 0.784$\!\to\!$\textbf{0.011} & 0.121$\!\to\!$0.000 & 0.997$\!\to\!$\textbf{0.505} & 0.554$\!\to\!$\textbf{0.000} \\
    \bottomrule
  \end{tabular}
\end{table}

\begin{table}[htbp]
  \centering
  \scriptsize
  \setlength{\tabcolsep}{3pt}
  \caption{Per-method backing for T2 row \emph{MUSE-News (13 base, Llama-2-7B-hf; CRNPO omitted)}. Calibration uses \texttt{MUSE-news\_retrain} as $\theta_{\text{ref}}$, F-ROUGE is the no-attack forget Q\&A ROUGE-L on \texttt{knowmem}, TOFU MU is undefined on this benchmark and printed \texttt{---}, and MIA.agg here averages the two MUSE AUCs (loss and min-$k$).}
  \label{tab:a-t2-muse}
  \begin{tabular}{l cccc}
    \toprule
    Base & F-ROUGE ($\downarrow$) & MU ($\uparrow$) & MIA.agg ($\downarrow$) & K20-LoRA ($\downarrow$) \\
    \midrule
    \texttt{GradDiff} & 0.088$\!\to\!$\textbf{0.000} & --- & 0.511$\!\to\!$\textbf{0.384} & 0.052$\!\to\!$\textbf{0.000} \\
    \texttt{NPO}      & 0.046$\!\to\!$\textbf{0.007} & --- & 0.499$\!\to\!$\textbf{0.278} & 0.045$\!\to\!$\textbf{0.009} \\
    \texttt{UNDIAL}   & 0.085$\!\to\!$\textbf{0.032} & --- & 0.493$\!\to\!$\textbf{0.182} & 0.126$\!\to\!$\textbf{0.015} \\
    \texttt{LATNPO}   & 0.049$\!\to\!$\textbf{0.000} & --- & 0.500$\!\to\!$\textbf{0.304} & 0.048$\!\to\!$\textbf{0.004} \\
    \texttt{RMU}      & 0.038$\!\to\!$\textbf{0.000} & --- & 0.514$\!\to\!$\textbf{0.298} & 0.057$\!\to\!$\textbf{0.000} \\
    \texttt{PDU}      & 0.001$\!\to\!$\textbf{0.000} & --- & 0.528$\!\to\!$\textbf{0.305} & 0.007$\!\to\!$\textbf{0.000} \\
    \texttt{SAMNPO}   & 0.049$\!\to\!$\textbf{0.005} & --- & 0.498$\!\to\!$\textbf{0.308} & 0.048$\!\to\!$\textbf{0.010} \\
    \texttt{ILUNPO}   & 0.047$\!\to\!$\textbf{0.004} & --- & 0.498$\!\to\!$\textbf{0.276} & 0.049$\!\to\!$\textbf{0.008} \\
    \texttt{RNANPO}   & 0.045$\!\to\!$\textbf{0.007} & --- & 0.499$\!\to\!$\textbf{0.262} & 0.046$\!\to\!$\textbf{0.005} \\
    \texttt{RSNPO}    & 0.044$\!\to\!$\textbf{0.011} & --- & 0.500$\!\to\!$\textbf{0.288} & 0.062$\!\to\!$\textbf{0.005} \\
    \texttt{SWANPO}   & 0.046$\!\to\!$\textbf{0.000} & --- & 0.498$\!\to\!$\textbf{0.337} & 0.046$\!\to\!$\textbf{0.000} \\
    \texttt{SimNPO}   & 0.040$\!\to\!$\textbf{0.001} & --- & 0.490$\!\to\!$\textbf{0.176} & 0.063$\!\to\!$\textbf{0.018} \\
    \texttt{JensUn}   & 0.038$\!\to\!$\textbf{0.000} & --- & 0.497$\!\to\!$\textbf{0.143} & 0.044$\!\to\!$\textbf{0.003} \\
    \midrule
    \emph{Mean}       & 0.047$\!\to\!$\textbf{0.005} & --- & 0.502$\!\to\!$\textbf{0.280} & 0.053$\!\to\!$\textbf{0.006} \\
    \bottomrule
  \end{tabular}
\end{table}

\paragraph{Consistent-evaluator check.}
Because the tabulated baseline columns come from the official evaluator while \textsc{MC} columns come
from the pipeline evaluator (\S\ref{sec:exp-setup}, \emph{Metric provenance}), we re-ran the pipeline
evaluator on all 14 baseline checkpoints and repeated every headline comparison under a single
evaluator. F.agg wins remain $14/14$ (panel mean $0.402 \to 0.055$), MU wins remain $1/14$ (panel mean
$0.271 \to 0.110$, so the tabulated $0.44 \to 0.11$ \emph{overstates} the utility cost),
raw-MIA wins move from $13/14$ to $14/14$, and per-detector advantage wins from $6/14$ to $7/14$,
each comparison shifting by at most one method, in \textsc{MC}'s favor. No conclusion in the paper
depends on the evaluator mix.

\section{Five MIA AUC breakdown}
\label{app:mia-breakdown}

Table~\ref{tab:a-t-5mia} expands the MIA.agg column of Table~\ref{tab:t1-main} into the five constituent AUCs (loss, min-$k$, min-$k$++, zlib, and gradnorm) per method at 1B \texttt{forget10}. \textsc{MC} reduces every detector's panel-mean AUC (loss $0.85 \to 0.17$, min-$k$ $0.86 \to 0.17$, min-$k$++ $0.86 \to 0.21$, zlib $0.77 \to 0.17$, gradnorm $0.60 \to 0.51$), so the raw-AUC reduction is broad-spectrum rather than driven by one detector.

The same numbers, however, show the \emph{overshoot} phenomenon. Post-polish AUCs land far \emph{below} $0.5$, and an AUC of $0.17$ is as separable as one of $0.83$ once the attacker inverts the detector. Measured cancellation-free through the per-detector membership advantage $\frac{1}{5}\sum_i |\mathrm{AUC}_i - 0.5|$, \textsc{MC} improves only $6/14$ methods ($0.33 \to 0.32$ panel mean), and the loss, min-$k$, and zlib detectors \emph{worsen} in advantage on $7$, $8$, and $8$ of $14$ methods respectively because \textsc{MC} pushes confident members well past chance. Advantage must be computed per detector before averaging, since taking $|\cdot - 0.5|$ of the five-AUC \emph{mean} lets oppositely-reversed detectors cancel, as RMU's five-AUC mean sits at $0.46$ (apparent advantage near zero) while its loss and gradnorm AUCs sit at $0.39$ and $0.91$, both individually informative. The honest summary is that \textsc{MC} removes the baseline's confident membership signal (raw AUCs fall on $13/14$ methods) but converts much of it into reversed separability rather than chance behavior, the TOFU-side analogue of the MUSE reversed signal (App.~\ref{app:muse}, Limitations), and the reason an MIA-aware stopping rule for margin pressure is future work.

\begin{table*}[tp]
  \centering
  \scriptsize
  \setlength{\tabcolsep}{3pt}
  \caption{\textbf{Per-MIA AUC breakdown} at 1B forget10. Each cell shows \emph{baseline} $\to$ \textsc{MC}; \textbf{bold} marks \textsc{MC} improvement (lower AUC = better, $\downarrow$). Defends \emph{MIA.agg} in Tables~\ref{tab:t1-main} and~\ref{tab:t2-robust}: no single MIA suite is hiding a collapse mode.}
  \label{tab:a-t-5mia}
  \begin{tabular}{ll ccccc}
    \toprule
    Base & Family & MIA-loss & MIA-min$k$ & MIA-min$k$++ & MIA-zlib & MIA-gradnorm \\
    \midrule
    \texttt{GradDiff} & GradDiff & 0.88$\!\to\!$\textbf{0.07} & 0.87$\!\to\!$\textbf{0.08} & 0.68$\!\to\!$\textbf{0.21} & 0.84$\!\to\!$\textbf{0.09} & 0.68$\!\to\!$0.68 \\
    \texttt{RMU} & GradDiff & 0.44$\!\to\!$\textbf{0.39} & 0.47$\!\to\!$\textbf{0.32} & 0.65$\!\to\!$\textbf{0.32} & 0.36$\!\to\!$0.38 & 0.04$\!\to\!$0.91 \\
    \texttt{PDU} & GradDiff & 0.52$\!\to\!$\textbf{0.05} & 0.55$\!\to\!$\textbf{0.05} & 0.57$\!\to\!$\textbf{0.05} & 0.37$\!\to\!$\textbf{0.03} & 0.36$\!\to\!$0.94 \\
    \midrule
    \texttt{SimNPO} & SimNPO & 0.97$\!\to\!$\textbf{0.04} & 0.97$\!\to\!$\textbf{0.03} & 0.82$\!\to\!$\textbf{0.18} & 0.93$\!\to\!$\textbf{0.05} & 0.76$\!\to\!$\textbf{0.35} \\
    \midrule
    \texttt{NPO} & NPO & 0.83$\!\to\!$\textbf{0.11} & 0.83$\!\to\!$\textbf{0.10} & 0.90$\!\to\!$\textbf{0.11} & 0.66$\!\to\!$\textbf{0.11} & 0.35$\!\to\!$0.39 \\
    \texttt{SAMNPO} & NPO & 0.99$\!\to\!$\textbf{0.17} & 0.99$\!\to\!$\textbf{0.15} & 0.97$\!\to\!$\textbf{0.26} & 0.99$\!\to\!$\textbf{0.23} & 0.96$\!\to\!$\textbf{0.38} \\
    \texttt{ILUNPO} & NPO & 0.84$\!\to\!$\textbf{0.11} & 0.84$\!\to\!$\textbf{0.10} & 0.90$\!\to\!$\textbf{0.10} & 0.67$\!\to\!$\textbf{0.08} & 0.37$\!\to\!$\textbf{0.13} \\
    \texttt{LATNPO} & NPO & 0.84$\!\to\!$\textbf{0.11} & 0.84$\!\to\!$\textbf{0.09} & 0.90$\!\to\!$\textbf{0.11} & 0.66$\!\to\!$\textbf{0.11} & 0.38$\!\to\!$0.38 \\
    \texttt{RNANPO} & NPO & 0.86$\!\to\!$\textbf{0.03} & 0.85$\!\to\!$\textbf{0.04} & 0.77$\!\to\!$\textbf{0.11} & 0.71$\!\to\!$\textbf{0.03} & 0.49$\!\to\!$\textbf{0.36} \\
    \texttt{RSNPO} & NPO & 0.98$\!\to\!$\textbf{0.24} & 0.98$\!\to\!$\textbf{0.24} & 0.97$\!\to\!$\textbf{0.28} & 0.98$\!\to\!$\textbf{0.28} & 0.82$\!\to\!$\textbf{0.24} \\
    \texttt{SWANPO} & NPO & 0.84$\!\to\!$\textbf{0.16} & 0.84$\!\to\!$\textbf{0.15} & 0.90$\!\to\!$\textbf{0.24} & 0.67$\!\to\!$\textbf{0.16} & 0.37$\!\to\!$0.63 \\
    \texttt{CRNPO} & NPO & 1.00$\!\to\!$\textbf{0.16} & 1.00$\!\to\!$\textbf{0.17} & 1.00$\!\to\!$\textbf{0.18} & 1.00$\!\to\!$\textbf{0.17} & 1.00$\!\to\!$\textbf{0.46} \\
    \midrule
    \texttt{UNDIAL} & Distil & 0.98$\!\to\!$\textbf{0.66} & 0.99$\!\to\!$\textbf{0.71} & 1.00$\!\to\!$\textbf{0.72} & 0.95$\!\to\!$\textbf{0.61} & 0.78$\!\to\!$\textbf{0.56} \\
    \texttt{JensUn} & Distil & 1.00$\!\to\!$\textbf{0.08} & 1.00$\!\to\!$\textbf{0.06} & 0.99$\!\to\!$\textbf{0.05} & 1.00$\!\to\!$\textbf{0.08} & 1.00$\!\to\!$\textbf{0.78} \\
    \bottomrule
  \end{tabular}
\end{table*}

\section{Hyperparameter sensitivity grid and full ablation per base}
\label{app:hp-grid}
\label{app:t3-full}

Table~\ref{tab:a-t-hp-full} reports the five single-axis HP perturbations around the default $(\kappa, \lambda_{\text{KL}}, r) = (5, 0.05, 32)$, evaluated on the 5 representative 1B \texttt{forget10} bases under K20-LoRA. The maximum perturbation moves panel-mean K20-LoRA by at most $0.04$, well below the $0.255$ gap to baseline, supporting the use of a single global \textsc{MC} configuration across all reported settings without per-cell tuning.

\begin{table*}[tp]
  \centering
  \scriptsize
  \setlength{\tabcolsep}{3pt}
  \caption{\textbf{Hyperparameter sensitivity, single-axis perturbations.} Each row is one $(\kappa, \lambda_{\text{KL}}, r)$ configuration; cells are post-attack K20-LoRA ROUGE on the 5 representative bases at 1B \texttt{forget10} (lower = more robust, $\downarrow$). The default config $(\kappa, \lambda_{\text{KL}}, r) = (5, 0.05, 32)$ (boxed) lies at the interior of the safety region; perturbing each axis to its low and high extreme moves the mean K20-LoRA by at most $0.040$, well below the $0.255$ gap to the baseline panel mean ($0.439$). Five perturbations cover the design space's local neighborhood; we therefore retain a single global configuration across all (size, tier, benchmark) cells without per-setting tuning.}
  \label{tab:a-t-hp-full}
  \begin{tabular}{ccc c c c c c c}
    \toprule
    $\kappa$ & $\lambda_{\text{KL}}$ & $r$ & \texttt{GradDiff} & \texttt{NPO} & \texttt{SimNPO} & \texttt{UNDIAL} & \texttt{CRNPO} & mean \\
    \midrule
    \boxed{5} & \boxed{0.05} & \boxed{32} & 0.158 & 0.246 & 0.025 & 0.282 & 0.206 & \textbf{0.184} \\
    \midrule
    1  & 0.05 & 32 & 0.031 & 0.250 & 0.065 & 0.287 & 0.213 & \textbf{0.169} \\
    20 & 0.05 & 32 & 0.134 & 0.253 & 0.236 & 0.279 & 0.215 & \textbf{0.224} \\
    \midrule
    5  & 0.2  & 32 & 0.028 & 0.217 & 0.310 & 0.279 & 0.249 & \textbf{0.217} \\
    \midrule
    5  & 0.05 & 16 & 0.119 & 0.275 & 0.024 & 0.292 & 0.222 & \textbf{0.186} \\
    5  & 0.05 & 64 & 0.132 & 0.052 & 0.207 & 0.260 & 0.134 & \textbf{0.157} \\
    \bottomrule
  \end{tabular}
\end{table*}

The low corner $\{\kappa = 1, r = 16\}$ weakens both the hinge sharpness and the adapter capacity, yet its panel-mean K20-LoRA of $0.178$ (Table~\ref{tab:t3-ablation} row 5) is within run-to-run variation of the default $0.184$, indicating the configuration sits on a plateau rather than at the edge of under-crossing. The high corner $\{\kappa = 20, \lambda_{\text{KL}} = 0.2, r = 64\}$ pushes margin pressure and utility anchor simultaneously and yields $0.199$ on row 6, the worse of the two corner configurations (the single-axis $\kappa{=}20$ perturbation reaches $0.224$, Table~\ref{tab:a-t-hp-full}). Neither corner exceeds the default panel mean ($0.184$) by more than $0.015$, and the single-axis sensitivity table (Table~\ref{tab:a-t-hp-full}) confirms the same flat shape along every axis, indicating that the chosen configuration sits inside a plateau rather than at a brittle optimum. The configuration is therefore safe to reuse across model sizes, forget tiers, and the MUSE-News transfer without re-tuning.

\begin{table}[htbp]
  \centering
  \small
  \setlength{\tabcolsep}{4pt}
  \caption{\textbf{Ablation matrix at 1B \texttt{forget10}}, averaged over 5 bases ($\Delta$ is the cliff gap of Eq.~\eqref{eq:cliff-gap}, pipeline evaluator, negative crosses; --- where the margin probe was not run). HP corners perturb the default $(\kappa,\lambda_{\text{KL}},r){=}(5,0.05,32)$, with low $\{\kappa{=}1,r{=}16\}$, high $\{\kappa{=}20,\lambda_{\text{KL}}{=}0.2,r{=}64\}$; full single-axis sensitivity table in App.~\ref{app:hp-grid}.}
  \label{tab:t3-ablation}
  \resizebox{\linewidth}{!}{%
  \begin{tabular}{l c cccc}
    \toprule
    Variant & $n$ & $\Delta$ & MU & F-ROUGE & K20-LoRA \\
    \midrule
    GradAscent          & 1/5 & $-70.58$ & 0.000 & 0.000 & 0.040 \\
    Forget-only hinge   & 5/5 & $-26.28$ & 0.091 & 0.121 & 0.157 \\
    Two-sided retain-hinge & 5/5 & $-34.67$ & 0.080 & 0.106 & 0.173 \\
    \textbf{MC}         & 5/5 & $-27.29$ & 0.149 & 0.066 & 0.184 \\
    HP corner low       & 5/5 & --- & 0.099 & 0.067 & 0.178 \\
    HP corner high      & 5/5 & --- & 0.106 & 0.070 & 0.199 \\
    \bottomrule
  \end{tabular}%
  }
\end{table}

\section{Attacker scaling, all K}
\label{app:attacker-scaling}

Table~\ref{tab:a-t-attacker-scaling} reports per-method post-attack ROUGE for LoRA-r8 across the full $K$ sweep $\{1, 3, 5, 10, 20, 50, 100\}$ at 1B \texttt{forget10}, expanding the headline $K{=}20$ column of Table~\ref{tab:t1-main}. Baseline ROUGE rises toward the recovery ceiling as $K$ grows (with small non-monotone fluctuations at large $K$), while \textsc{MC} rises only slowly with $K$ and stays at or below $0.33$ in every cell even at $K{=}100$, far under the baseline curve at every budget, confirming that the robustness gain is not specific to the canonical $K{=}20$ budget.

The qualitative shape of the two curves is consistent with Theorem~\ref{thm:relearn-bound} (a one-sided
upper bound, which caps the \textsc{MC} curve but does not by itself predict the baselines' rise). The
lift budget $\Lambda(N) = \eta N G_m H + \tfrac12 L_m \eta^2 N H^2 + c_\beta$ grows with the attacker
budget $N$ but is independent of the starting cliff gap. For baselines with positive cliff gap, no lift is even
needed to stay recoverable, and empirically ROUGE rises toward the recovery ceiling as $N$ grows, with
diminishing returns once a substantial fraction of the holdout has been recovered. For
\textsc{MC}-polished checkpoints with negative cliff gap, ROUGE rises far more slowly (every cell
$\le 0.33$ at $K{=}100$, versus baselines approaching the recovery ceiling). The measured-constants
instrumentation (App.~\ref{app:thm3-constants}) shows the sup-based $\Lambda(N)$ is conservative, so
the K-sweep itself, rather than the certificate, carries the empirical load. The K-sweep therefore stress-tests
the budget bound across two orders of magnitude in $K$ rather than producing a separately tuned
result.

\begin{table*}[tp]
  \centering
  \scriptsize
  \setlength{\tabcolsep}{2pt}
  \caption{\textbf{Attacker scaling, all $K$.} Per-method post-attack ROUGE for LoRA-r8 across $K\in\{1,3,5,10,20,50,100\}$ at 1B forget10. Cells: \emph{baseline} $\to$ \textsc{MC} (lower = more robust, $\downarrow$). Strong-attacker variants (LoRA-r32, soft-prompt $n{=}100$, FPFT $K{=}50$) appear in Tables~\ref{tab:a-t-strong-r32}--\ref{tab:a-t-strong-fpft50}.}
  \label{tab:a-t-attacker-scaling}
  \resizebox{\textwidth}{!}{%
  \begin{tabular}{l c c c c c c c}
    \toprule
    Base & $K{=}1$ & $K{=}3$ & $K{=}5$ & $K{=}10$ & $K{=}20$ & $K{=}50$ & $K{=}100$ \\
    \midrule
    \texttt{GradDiff} & 0.309$\!\to\!$0.017 & 0.347$\!\to\!$0.017 & 0.363$\!\to\!$0.024 & 0.416$\!\to\!$0.096 & 0.429$\!\to\!$0.158 & 0.436$\!\to\!$0.205 & 0.427$\!\to\!$0.220 \\
    \texttt{RMU} & 0.320$\!\to\!$0.005 & 0.346$\!\to\!$0.005 & 0.361$\!\to\!$0.005 & 0.398$\!\to\!$0.002 & 0.471$\!\to\!$0.029 & 0.475$\!\to\!$0.026 & 0.429$\!\to\!$0.320 \\
    \texttt{PDU} & 0.144$\!\to\!$0.020 & 0.277$\!\to\!$0.020 & 0.339$\!\to\!$0.018 & 0.418$\!\to\!$0.019 & 0.428$\!\to\!$0.152 & 0.414$\!\to\!$0.243 & 0.427$\!\to\!$0.259 \\
    \midrule
    \texttt{SimNPO} & 0.374$\!\to\!$0.021 & 0.425$\!\to\!$0.017 & 0.430$\!\to\!$0.024 & 0.458$\!\to\!$0.023 & 0.458$\!\to\!$0.025 & 0.480$\!\to\!$0.023 & 0.478$\!\to\!$0.067 \\
    \midrule
    \texttt{NPO} & 0.245$\!\to\!$0.065 & 0.294$\!\to\!$0.084 & 0.309$\!\to\!$0.129 & 0.343$\!\to\!$0.244 & 0.348$\!\to\!$0.246 & 0.347$\!\to\!$0.295 & 0.357$\!\to\!$0.310 \\
    \texttt{SAMNPO} & 0.368$\!\to\!$0.101 & 0.408$\!\to\!$0.108 & 0.419$\!\to\!$0.114 & 0.407$\!\to\!$0.136 & 0.414$\!\to\!$0.201 & 0.435$\!\to\!$0.220 & 0.455$\!\to\!$0.229 \\
    \texttt{ILUNPO} & 0.240$\!\to\!$0.059 & 0.293$\!\to\!$0.058 & 0.309$\!\to\!$0.064 & 0.349$\!\to\!$0.082 & 0.343$\!\to\!$0.205 & 0.352$\!\to\!$0.259 & 0.365$\!\to\!$0.311 \\
    \texttt{LATNPO} & 0.245$\!\to\!$0.072 & 0.281$\!\to\!$0.090 & 0.306$\!\to\!$0.105 & 0.349$\!\to\!$0.215 & 0.342$\!\to\!$0.247 & 0.357$\!\to\!$0.301 & 0.371$\!\to\!$0.323 \\
    \texttt{RNANPO} & 0.250$\!\to\!$0.010 & 0.294$\!\to\!$0.011 & 0.316$\!\to\!$0.012 & 0.340$\!\to\!$0.035 & 0.366$\!\to\!$0.165 & 0.403$\!\to\!$0.275 & 0.375$\!\to\!$0.298 \\
    \texttt{RSNPO} & 0.389$\!\to\!$0.092 & 0.436$\!\to\!$0.099 & 0.448$\!\to\!$0.106 & 0.452$\!\to\!$0.147 & 0.442$\!\to\!$0.195 & 0.446$\!\to\!$0.231 & 0.456$\!\to\!$0.235 \\
    \texttt{SWANPO} & 0.244$\!\to\!$0.084 & 0.297$\!\to\!$0.073 & 0.310$\!\to\!$0.067 & 0.344$\!\to\!$0.138 & 0.349$\!\to\!$0.253 & 0.358$\!\to\!$0.280 & 0.363$\!\to\!$0.306 \\
    \texttt{CRNPO} & 0.488$\!\to\!$0.017 & 0.544$\!\to\!$0.036 & 0.571$\!\to\!$0.056 & 0.582$\!\to\!$0.184 & 0.579$\!\to\!$0.206 & 0.597$\!\to\!$0.272 & 0.579$\!\to\!$0.263 \\
    \midrule
    \texttt{UNDIAL} & 0.280$\!\to\!$0.253 & 0.339$\!\to\!$0.260 & 0.392$\!\to\!$0.274 & 0.383$\!\to\!$0.272 & 0.381$\!\to\!$0.282 & 0.400$\!\to\!$0.303 & 0.366$\!\to\!$0.297 \\
    \texttt{JensUn} & 0.353$\!\to\!$0.055 & 0.374$\!\to\!$0.053 & 0.402$\!\to\!$0.058 & 0.406$\!\to\!$0.062 & 0.447$\!\to\!$0.147 & 0.487$\!\to\!$0.273 & 0.483$\!\to\!$0.278 \\
    \bottomrule
  \end{tabular}%
  }
\end{table*}

\section{Measured constants for Theorem~\ref{thm:relearn-bound}}
\label{app:thm3-constants}

We instrument the paper's own K20 LoRA-r8 attack (same seed, samples, and learning rate as the
pipeline attacker) on three \textsc{MC}-polished 1B \texttt{forget10} checkpoints, recording at every
step the realized step norm $\|\delta^{(n)}\|$, the smoothed-diagnostic gradient norm
$\|\nabla m_\beta\|$, and a segment curvature estimate, all on the frozen held-out diagnostic
positions of Theorem~\ref{thm:relearn-bound} ($\beta = 120$, $c_\beta = 0.098$, $n = 40$ samples,
instrumentation script in the code release). Table~\ref{tab:a-thm3-const} reports the sup
constants, the realized-trajectory budget
$\hat\Lambda(N) = G_m^{\sup} \sum_n \|\delta^{(n)}\| + \tfrac{1}{2} L_m^{\sup} \sum_n \|\delta^{(n)}\|^2 + c_\beta$,
and the measured lift $m_\beta(\tilde\theta_N) - m_\beta(\hat\theta)$.

Two observations. First, the one-sided bound holds with room on every instrumented trajectory, with
measured lifts of $+102.0$, $+7.2$, and $+1.0$ against budgets $\hat\Lambda(N)$ of $12641$, $335$, and
$53$ for GradDiff, NPO, and UNDIAL respectively. Second, the sup-based constants make the budget
conservative by one to two orders of magnitude, driven by rare high-gradient steps along the
trajectory, so the certificate $-\Delta(\hat\theta) > \Lambda(N)$ is a sufficient condition that does
not fire at worst-case constants on these runs. The theorem's role is therefore to define the
attack-budget scaling that the K-sweep stress-tests empirically (App.~\ref{app:attacker-scaling}),
not to provide a numerically tight per-checkpoint certificate, and we report the constants so this
looseness is explicit rather than implied.

\begin{table}[htbp]
\centering
\small
\begin{tabular}{lrrrrr}
\toprule
Base & $G_m^{\sup}$ & $L_m^{\sup}$ & $\sum_n\|\delta^{(n)}\|$ & $\hat\Lambda(N)$ & lift \\
\midrule
GradDiff & 2081 & 11035 & 3.79 & 12641 & $+102.0$ \\
NPO      & 61.0 & 113.7 & 4.47 & 335   & $+7.2$ \\
UNDIAL   & 10.5 & 21.4  & 4.11 & 53    & $+1.0$ \\
\bottomrule
\end{tabular}
\caption{Instrumented Theorem~\ref{thm:relearn-bound} constants on the K20 LoRA-r8 attack at 1B
\texttt{forget10} (\textsc{MC}-polished checkpoints, $N{=}20$, $c_\beta{=}0.098$). The one-sided bound
lift $\le \hat\Lambda(N)$ holds with one to two orders of slack in every case.}
\label{tab:a-thm3-const}
\end{table}

\section{Strong-attacker variants}
\label{app:strong-attackers}

Tables~\ref{tab:a-t-strong-r32}, \ref{tab:a-t-strong-sp100}, and \ref{tab:a-t-strong-fpft50} expand the strong-attacker results of \S\ref{sec:exp-ablations} to per-method numbers on the 5 representative bases at 1B \texttt{forget10}. LoRA-r32 raises adapter rank by $4\times$ over the default LoRA-r8, soft-prompt with $n_{\text{soft}}{=}100$ uses a $5\times$ longer prompt than the default $n_{\text{soft}}{=}20$, and FPFT $K{=}50$ both increases the relearn budget and removes the LoRA bottleneck. \textsc{MC} wins $5/5$ under every strong-attacker variant. In every cell the baseline's cliff gap is positive, and the \textsc{MC}-polished counterpart's is negative in every cell except UNDIAL (\S\ref{sec:m-thm2}).

\begin{table}[htbp]
  \centering
  \scriptsize
  \setlength{\tabcolsep}{3pt}
  \caption{\textbf{Strong attacker: LoRA-r32} at $K{=}20$. Per-method post-attack ROUGE on 5 representative bases at 1B \texttt{forget10}. $\Delta_{\text{R}}$ is the baseline$-$\textsc{MC} post-attack ROUGE gap (positive = \textsc{MC} safer; not the cliff gap $\Delta$ of Eq.~\eqref{eq:cliff-gap}). \textsc{MC} wins $5$/$5$.}
  \label{tab:a-t-strong-r32}
  \begin{tabular}{lccc}
    \toprule
    Base & Baseline ($\downarrow$) & \textsc{MC} ($\downarrow$) & $\Delta_{\text{R}}$ \\
    \midrule
    \texttt{GradDiff} & 0.399 & \textbf{0.235} & $+0.164$ \\
    \texttt{NPO}      & 0.325 & \textbf{0.248} & $+0.077$ \\
    \texttt{SimNPO}   & 0.438 & \textbf{0.272} & $+0.165$ \\
    \texttt{UNDIAL}   & 0.306 & \textbf{0.276} & $+0.030$ \\
    \texttt{CRNPO}    & 0.560 & \textbf{0.252} & $+0.307$ \\
    \midrule
    \emph{Mean}       & 0.406 & \textbf{0.257} & $+0.149$ \\
    \bottomrule
  \end{tabular}
\end{table}

The three variants probe distinct relaxations of the canonical LoRA-r8 setup. LoRA-r32 quadruples the
adapter rank, enlarging the increments the attacker can realize (a larger class parameter $H$ in
Theorem~\ref{thm:relearn-bound}), yet \textsc{MC}'s lead persists. Soft-prompt
with $n_{\text{soft}}{=}100$ does not perturb the model weights at all and instead optimizes a $100$-token
prefix, sidestepping the adapter bottleneck via an input-space attack outside the weight-space attack
class of Theorem~\ref{thm:relearn-bound}, and that it too fails against \textsc{MC} is evidence beyond the
theorem's scope. FPFT $K{=}50$ removes the adapter bottleneck entirely (App.~\ref{app:fpft-bound}) and
grows $N$ to $50$, the worst case in our budget envelope. The panel-mean baseline-to-\textsc{MC} ROUGE-L drops are $+0.149$, $+0.188$, and $+0.260$ respectively, with the FPFT case producing the largest gap because the baseline rises fastest under the strongest attacker while \textsc{MC} stays floored at its post-polish margin level.

\begin{table}[htbp]
  \centering
  \scriptsize
  \setlength{\tabcolsep}{3pt}
  \caption{\textbf{Strong attacker: soft-prompt $n_{\text{soft}}{=}100$} at $200$ optimization steps. Per-method post-attack ROUGE on 5 representative bases at 1B \texttt{forget10}. $\Delta_{\text{R}}$ is the baseline$-$\textsc{MC} post-attack ROUGE gap (positive = \textsc{MC} safer; not the cliff gap $\Delta$ of Eq.~\eqref{eq:cliff-gap}). \textsc{MC} wins $5$/$5$.}
  \label{tab:a-t-strong-sp100}
  \begin{tabular}{lccc}
    \toprule
    Base & Baseline ($\downarrow$) & \textsc{MC} ($\downarrow$) & $\Delta_{\text{R}}$ \\
    \midrule
    \texttt{GradDiff} & 0.031 & \textbf{0.026} & $+0.005$ \\
    \texttt{NPO}      & 0.039 & \textbf{0.026} & $+0.012$ \\
    \texttt{SimNPO}   & 0.441 & \textbf{0.039} & $+0.402$ \\
    \texttt{UNDIAL}   & 0.314 & \textbf{0.250} & $+0.064$ \\
    \texttt{CRNPO}    & 0.473 & \textbf{0.018} & $+0.455$ \\
    \midrule
    \emph{Mean}       & 0.260 & \textbf{0.072} & $+0.188$ \\
    \bottomrule
  \end{tabular}
\end{table}

\begin{table}[htbp]
  \centering
  \scriptsize
  \setlength{\tabcolsep}{3pt}
  \caption{\textbf{Strong attacker: full-parameter fine-tune at $K{=}50$}. Per-method post-attack ROUGE on 5 representative bases at 1B \texttt{forget10}. $\Delta_{\text{R}}$ is the baseline$-$\textsc{MC} post-attack ROUGE gap (positive = \textsc{MC} safer; not the cliff gap $\Delta$ of Eq.~\eqref{eq:cliff-gap}). \textsc{MC} wins $5$/$5$.}
  \label{tab:a-t-strong-fpft50}
  \begin{tabular}{lccc}
    \toprule
    Base & Baseline ($\downarrow$) & \textsc{MC} ($\downarrow$) & $\Delta_{\text{R}}$ \\
    \midrule
    \texttt{GradDiff} & 0.408 & \textbf{0.104} & $+0.305$ \\
    \texttt{NPO}      & 0.357 & \textbf{0.268} & $+0.089$ \\
    \texttt{SimNPO}   & 0.485 & \textbf{0.044} & $+0.441$ \\
    \texttt{UNDIAL}   & 0.393 & \textbf{0.294} & $+0.100$ \\
    \texttt{CRNPO}    & 0.564 & \textbf{0.201} & $+0.363$ \\
    \midrule
    \emph{Mean}       & 0.442 & \textbf{0.182} & $+0.260$ \\
    \bottomrule
  \end{tabular}
\end{table}

\begin{figure}[htbp]
  \centering
  \includegraphics[width=\linewidth]{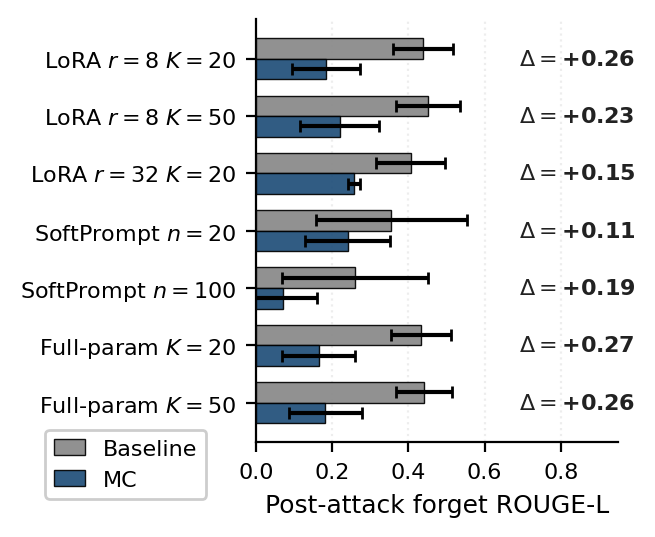}
  \caption{Held-out post-attack ROUGE-L at 1B \texttt{forget10} across LoRA, SoftPrompt, and FPFT attackers, averaged over 5 bases.}
  \label{fig:attack_bars}
\end{figure}

\section{Cliff gap predicts relearn recovery (Theorem~\ref{thm:relearn-bound} validation)}
\label{app:bridge}

Figure~\ref{fig:bridge} visualizes the empirical bridge between the geometric quantity $\Delta(\hat\theta)$ and the security-relevant quantity (post-attack forget ROUGE-L). Across three Llama-3 model sizes, the positive correlation is strong and statistically significant in every case, confirming that the cliff gap is not merely a method-internal diagnostic but a predictor of how well an unlearned checkpoint will resist a relearn attack. This correlation motivates Theorem~\ref{thm:relearn-bound}'s focus on $\Delta$, while the theorem itself is a one-sided upper bound and neither implies nor requires it.

\begin{figure*}[t]
  \centering
  \includegraphics[width=\textwidth]{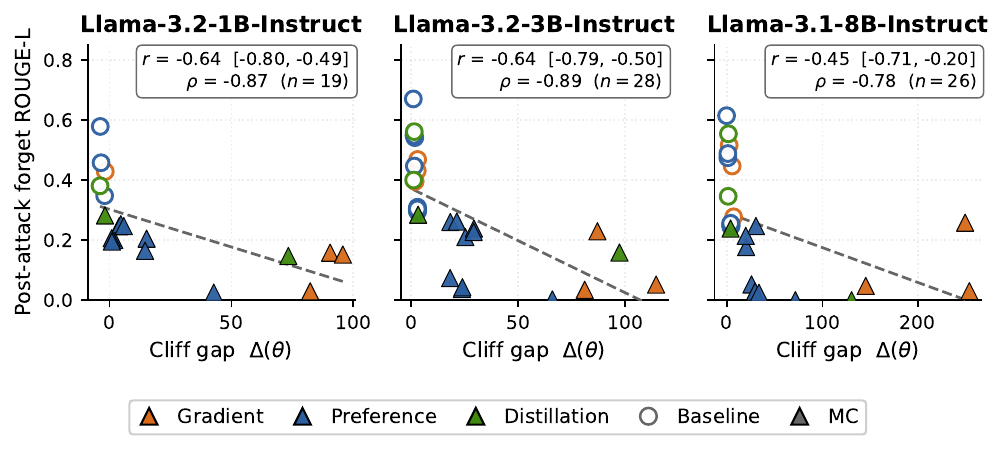}
  \caption{Cliff gap predicts relearn-recovery across model sizes (the empirical bridge motivating Theorem~\ref{thm:relearn-bound}). Each point is one (method, target) cell at the indicated model size, where circles are baselines and triangles are \textsc{MC}, colored by loss family (gradient / preference / distillation). The $x$-axis is the cliff gap $\Delta(\hat\theta) = m_{\hat\theta}(\mathcal{D}_f) - m_{\mathrm{ref}}(\mathcal{D}_f)$ (positive for baselines except the degenerate GradDiff cell at 8B, \S\ref{sec:m-cliff}, and negative for \textsc{MC}-polished checkpoints except UNDIAL, positive at each of the three sizes shown, \S\ref{sec:m-thm2}), and the $y$-axis is post-attack forget ROUGE-L at $K{=}20$ under LoRA-r8. The strong positive correlation is consistent across all three Llama-3 sizes, with Pearson $r = 0.686$ at 1B ($n=28$), $0.645$ at 3B ($n=28$), and $0.451$ at 8B ($n=26$), and 95\% bootstrap CIs excluding zero in every case. Spearman $\rho$ is $0.912$ / $0.885$ / $0.776$. The monotone relation between $\Delta$ and post-attack recovery is therefore not a 1B artifact and supports using $\Delta$ as a robustness predictor at every size we evaluate (Theorem~\ref{thm:relearn-bound} itself is a one-sided bound and does not imply this correlation).}
  \label{fig:bridge}
\end{figure*}

\section{Family-stratified FPFT}
\label{app:fpft-per-family}

Table~\ref{tab:t3-fpft} reports the family-stratified robustness gap (baseline ROUGE / \textsc{MC} ROUGE) under the full-parameter $K{=}20$ attacker at 1B and 8B (a 3B FPFT sweep was not run and is left to future work). The ratio remains at least $3.2\times$ in every populated family-size cell and exceeds $10\times$ for the GradDiff family at 1B, indicating that \textsc{MC}'s cliff-crossing gain transfers under the strictest attacker without family-specific tuning.

The family ranking reflects the per-method baseline-to-\textsc{MC} contrast more than any family-level interaction. SimNPO has the largest gaps ($17.3\times$ at 1B, $203.5\times$ at 8B) because its single base method drives \textsc{MC} ROUGE close to zero under FPFT, so the ratio amplifies dramatically. NPO and Distil sit lower at 1B ($3.2\times$ and $3.4\times$) because both their baselines and their \textsc{MC} ROUGE values stay positive, compressing the ratio. The gaps \emph{grow} from 1B to 8B for three of the four families (NPO $3.2 \to 4.5$, Distil $3.4 \to 26.5$, SimNPO $17.3 \to 203.5$), and the GradDiff family's finite mean decreases ($11.3 \to 6.6$) only because its most extreme 8B cell (RMU, whose \textsc{MC} ROUGE is numerically zero and whose gap therefore diverges) is excluded from the mean. The NPO family is computed over $7$ rather than $8$ bases at 8B since RSNPO 8B is reported in Fig.~\ref{fig:cliff} only and not under FPFT attacks.

\begin{table}[htbp]
  \centering
  \small
  \caption{\textbf{Family-stratified FPFT $K{=}20$ robustness ratio} (baseline post-attack ROUGE / \textsc{MC} post-attack ROUGE; higher = larger \textsc{MC} advantage). 3B FPFT was not run (---).}
  \label{tab:t3-fpft}
  \begin{tabular}{l c ccc}
    \toprule
    Family & $n$ & 1B mean & 3B mean & 8B mean \\
    \midrule
    GradDiff & 3 & 11.3$\times$ & --- & 6.6$\times$ \\
    SimNPO & 1 & 17.3$\times$ & --- & 203.5$\times$ \\
    NPO & 8 & 3.2$\times$ & --- & 4.5$\times$ \\
    Distil & 2 & 3.4$\times$ & --- & 26.5$\times$ \\
    \bottomrule
  \end{tabular}
\end{table}

\section{MUSE-News per-method and rank correlation}
\label{app:muse}

The per-method MUSE-News \texttt{knowmem} forget ROUGE-L and held-out K20-LoRA post-attack ROUGE-L for the 13-method panel (the full 14-method set with CRNPO omitted for a 7B resource reason) at \textsc{MC} calibration are reported in Table~\ref{tab:t5-muse}, with per-cell backing in Table~\ref{tab:a-t2-muse}. \textsc{MC} wins forget ROUGE-L and K20-LoRA on all $13/13$ methods (panel means $0.047 \to 0.005$ and $0.053 \to 0.006$). The cliff diagnostic on MUSE behaves differently from TOFU. The retain-trained reference $\theta_{\text{ref}} = $\texttt{MUSE-news\_retrain} has median diagnostic-position forget margin $m_{\text{ref}} = -6.30$ on the knowmem forget split (median across samples, reported in place of the mean of Eq.~\eqref{eq:margin-agg} for robustness on long news passages), and eleven of the thirteen baselines cluster within $\pm 1.1$ of this value rather than sitting $+1.7$ to $+3.9$ above it as at TOFU 1B, with only UNDIAL ($+2.4$) and PDU ($+5.9$) well above, so the per-token margin diagnostic is less discriminative on long news passages and the operational improvement in forget recovery and relearn robustness is the more reliable transfer indicator. The MUSE MIA breakdown (per-detector AUCs from the pipeline evaluator, with the advantage shift footnoted in Table~\ref{tab:t2-robust}) has panel-mean raw \texttt{mia\_loss\_auc} dropping $0.529 \to 0.259$ and \texttt{mia\_min\_k\_auc} dropping $0.475 \to 0.285$ (both win $13/13$ on raw-AUC direction), while the corresponding membership advantage $|\mathrm{AUC} - 0.5|$ panel mean rises $0.028 \to 0.228$ (advantage worsens on $13/13$ methods). This is the reversed-direction MIA signal flagged in Limitations and reflects that the 5-epoch MUSE baseline unlearning stage had already brought baseline AUC near $0.5$.

\begin{table}[htbp]
  \centering
  \small
  \setlength{\tabcolsep}{5pt}
  \caption{\textbf{MUSE-News transfer} (\texttt{knowmem}, $\downarrow$), 13-method panel, baseline $\to$ \textsc{MC} calibration. \textbf{Bold} marks \textsc{MC} improvement. CRNPO omitted on MUSE (App.~\ref{app:muse}).}
  \label{tab:t5-muse}
  \begin{tabular}{l cc cc}
    \toprule
    & \multicolumn{2}{c}{Baseline} & \multicolumn{2}{c}{\textsc{MC} calibration} \\
    \cmidrule(lr){2-3}\cmidrule(lr){4-5}
    Method  & f-ROUGE & K20-LoRA & f-ROUGE & K20-LoRA \\
    \midrule
    \texttt{GradDiff} & 0.088 & 0.052 & \textbf{0.000} & \textbf{0.000} \\
    \texttt{NPO}      & 0.046 & 0.045 & \textbf{0.007} & \textbf{0.009} \\
    \texttt{UNDIAL}   & 0.085 & 0.126 & \textbf{0.032} & \textbf{0.015} \\
    \texttt{LATNPO}   & 0.049 & 0.048 & \textbf{0.000} & \textbf{0.004} \\
    \texttt{RMU}      & 0.038 & 0.057 & \textbf{0.000} & \textbf{0.000} \\
    \texttt{PDU}      & 0.001 & 0.007 & \textbf{0.000} & \textbf{0.000} \\
    \texttt{SAMNPO}   & 0.049 & 0.048 & \textbf{0.005} & \textbf{0.010} \\
    \texttt{ILUNPO}   & 0.047 & 0.049 & \textbf{0.004} & \textbf{0.008} \\
    \texttt{RNANPO}   & 0.045 & 0.046 & \textbf{0.007} & \textbf{0.005} \\
    \texttt{RSNPO}    & 0.044 & 0.062 & \textbf{0.011} & \textbf{0.005} \\
    \texttt{SWANPO}   & 0.046 & 0.046 & \textbf{0.000} & \textbf{0.000} \\
    \texttt{SimNPO}   & 0.040 & 0.063 & \textbf{0.001} & \textbf{0.018} \\
    \texttt{JensUn}   & 0.038 & 0.044 & \textbf{0.000} & \textbf{0.003} \\
    \midrule
    \emph{Mean}       & 0.047 & 0.053 & \textbf{0.005} & \textbf{0.006} \\
    \bottomrule
  \end{tabular}
\end{table}

\section{Deployment-phase full panel and cross-tier transfer}
\label{app:deploy}

The deployment-phase sweep replaces the gold retain reference $\theta_{\text{ref}}$ with the pre-unlearning target model $\theta_0$ as the margin anchor and swaps the KL probe for a $\theta_0$-anchored retain-side hinge (\S\ref{sec:m-v3}, the two-sided objective of Table~\ref{tab:t3-ablation}), reusing $(\kappa, \lambda_r, r, N_{\text{pol}}) = (5.0, 1.0, 32, 200)$ ($\lambda_r$ the retain-hinge weight). Table~\ref{tab:t4-deploy} reports the full 14-method panel at 1B \texttt{forget10} with calibration-phase numbers reproduced for paired comparison. The same recipe transfers across size and benchmark. At 3B \texttt{forget10} the deployment sweep wins K20-LoRA on $14/14$ methods (panel-mean $0.211$). Cross-tier transfer to MUSE-News was verified on the original 4-method subset (GradDiff, NPO, UNDIAL, LATNPO), where the base-anchored deployment recipe on Llama-2-7B-hf reaches mean K20-LoRA $0.007$, on par with the reference-anchored calibration mean ($0.006$ over the full 13-method panel). An 8B deployment sweep is left to future work.

Within the deployment column, GradDiff, PDU, and JensUn show \emph{lower} K20-LoRA than their reference-anchored \textsc{MC} counterparts ($0.036/0.030/0.026$ vs.\ $0.158/0.152/0.147$ in Table~\ref{tab:t4-deploy}). Under the deployment objective the forget-side hinge target is effectively inactive (\S\ref{sec:m-v3}), so forget pressure comes from the still-running native loss while the $\theta_0$-anchored retain hinge holds utility. The three methods above are exactly those whose native forget terms keep the strongest pressure at the polished point (the bounded, non-vanishing GradDiff-family gradients, and JensUn's steep distillation pull), which the utility anchor lets run further than the reference-anchored hinge target does. Consistent with the retain hinge anchoring utility directly, $5/14$ methods recover some utility under deployment while the remaining $9/14$ see a small additional MU drop on top of the calibration tradeoff (panel mean $0.11 \to 0.07$). The deployment-phase K20-LoRA distribution otherwise tracks the calibration-phase distribution closely, so the deployment recipe acts as a near-substitute for the reference-anchored variant rather than a separately tuned configuration.

\begin{table}[htbp]
  \centering
  \small
  \setlength{\tabcolsep}{4pt}
  \caption{\textbf{Deployment-phase \textsc{MC} at 1B \texttt{forget10}} (K20-LoRA, $\downarrow$). Deployment swaps $\theta_{\text{ref}}$ for $\theta_0$, beats baseline on \textbf{14/14} with mean degradation $+0.023$, and shows lower K20-LoRA than calibration for GradDiff, PDU, JensUn.}
  \label{tab:t4-deploy}
  \begin{tabular}{l ccc}
    \toprule
    Base & Baseline & Calibration & Deployment \\
    \midrule
    \texttt{GradDiff} & 0.429 & 0.158 & \textbf{0.036} \\
    \texttt{RMU}      & 0.471 & \textbf{0.029} & 0.044 \\
    \texttt{PDU}      & 0.428 & 0.152 & \textbf{0.030} \\
    \midrule
    \texttt{NPO}      & 0.348 & \textbf{0.246} & 0.278 \\
    \texttt{SimNPO}   & 0.458 & \textbf{0.025} & 0.203 \\
    \texttt{SAMNPO}   & 0.414 & \textbf{0.201} & 0.214 \\
    \texttt{CRNPO}    & 0.579 & \textbf{0.206} & 0.259 \\
    \texttt{RSNPO}    & 0.442 & \textbf{0.195} & 0.251 \\
    \texttt{SWANPO}   & 0.349 & \textbf{0.253} & 0.285 \\
    \texttt{LATNPO}   & 0.342 & \textbf{0.247} & 0.276 \\
    \texttt{ILUNPO}   & 0.343 & \textbf{0.205} & 0.300 \\
    \texttt{RNANPO}   & 0.366 & \textbf{0.165} & 0.281 \\
    \midrule
    \texttt{UNDIAL}   & 0.381 & \textbf{0.282} & 0.353 \\
    \texttt{JensUn}   & 0.447 & 0.147 & \textbf{0.026} \\
    \midrule
    \emph{Mean}       & 0.414 & \textbf{0.179} & 0.202 \\
    \bottomrule
  \end{tabular}
\end{table}

\section{General-capability check on MMLU}
\label{app:mmlu}

We evaluate zero-shot MMLU accuracy for the pre-unlearning target $\theta_0$, every \texttt{forget10}
baseline, and its \textsc{MC}-polished counterpart at all three Llama-3 sizes, using the
lm-evaluation-harness (wrapper script in the code release). $\theta_0$ scores $0.480/0.616/0.670$ at
1B/3B/8B, and every baseline stays within $0.02$ of its ceiling except RSNPO at 8B ($0.553$), so
unlearning itself leaves MMLU essentially intact (Table~\ref{tab:a-mmlu}). At 1B, \textsc{MC} keeps
accuracy within $0.064$ of the baseline for nine of fourteen methods (PDU loses $0.002$ and UNDIAL
$0.010$), costs GradDiff $0.089$ and SimNPO $0.132$, and collapses three methods toward the $0.25$
chance level (SAMNPO $0.256$, RMU $0.268$, CRNPO $0.287$). The damage is strongly size-dependent. At
3B and 8B eleven of fourteen methods stay within $0.02$ and $0.03$ of their baseline respectively,
and the 1B collapse cells shrink to at most $0.037$ (SAMNPO at 3B) and $0.016$ (RMU at 8B), so the
probe protects recognized capability better at scale. Two named exceptions remain. CRNPO loses
$0.052$ at 3B and $0.102$ at 8B, and JensUn inverts the pattern, intact at 1B but collapsing to
$0.428$ at 3B and $0.235$ at 8B. The collapse cells show that the Alpaca KL probe does not
universally protect recognized general capability under margin pressure. They are the capability-axis
analogue of the MIA overshoot of App.~\ref{app:mia-breakdown} and motivate the same stopping-rule
direction noted in Limitations. MMLU is therefore reported as a capability check with named
exceptions, not as a uniform strength of the method.

\begin{table}[htbp]
\centering
\small
\setlength{\tabcolsep}{4pt}
\begin{tabular}{lcccccc}
\toprule
& \multicolumn{2}{c}{1B} & \multicolumn{2}{c}{3B} & \multicolumn{2}{c}{8B} \\
\cmidrule(lr){2-3}\cmidrule(lr){4-5}\cmidrule(lr){6-7}
Method & Base & \textsc{MC} & Base & \textsc{MC} & Base & \textsc{MC} \\
\midrule
$\theta_0$ & \multicolumn{2}{c}{$0.480$} & \multicolumn{2}{c}{$0.616$} & \multicolumn{2}{c}{$0.670$} \\
\midrule
GradDiff & 0.469 & 0.381 & 0.612 & 0.608 & 0.659 & 0.653 \\
RMU      & 0.468 & 0.268 & 0.597 & 0.601 & 0.658 & 0.643 \\
PDU      & 0.463 & 0.461 & 0.615 & 0.597 & 0.661 & 0.645 \\
NPO      & 0.477 & 0.442 & 0.618 & 0.616 & 0.667 & 0.654 \\
SAMNPO   & 0.475 & 0.256 & 0.617 & 0.580 & 0.666 & 0.657 \\
ILUNPO   & 0.476 & 0.431 & 0.618 & 0.617 & 0.668 & 0.652 \\
LATNPO   & 0.476 & 0.455 & 0.618 & 0.614 & 0.668 & 0.657 \\
RNANPO   & 0.470 & 0.439 & 0.615 & 0.618 & 0.667 & 0.664 \\
RSNPO    & 0.466 & 0.402 & 0.609 & 0.597 & 0.553 & 0.501 \\
SWANPO   & 0.476 & 0.448 & 0.618 & 0.617 & 0.667 & 0.640 \\
CRNPO    & 0.480 & 0.287 & 0.616 & 0.564 & 0.671 & 0.569 \\
SimNPO   & 0.479 & 0.347 & 0.615 & 0.608 & 0.667 & 0.650 \\
UNDIAL   & 0.476 & 0.466 & 0.615 & 0.615 & 0.666 & 0.646 \\
JensUn   & 0.479 & 0.445 & 0.614 & 0.428 & 0.669 & 0.235 \\
\bottomrule
\end{tabular}
\caption{Zero-shot MMLU accuracy at \texttt{forget10}, baseline versus \textsc{MC}-polished, at all
three Llama-3 sizes with the pre-unlearning $\theta_0$ ceilings. Chance is $0.25$. The 1B collapse
cells (RMU, SAMNPO, CRNPO) recover at larger sizes, while JensUn collapses only at 3B and 8B and
CRNPO is the only method losing more than $0.05$ at every size.}
\label{tab:a-mmlu}
\end{table}

\section{Cross-architecture transfer to Phi-3.5-mini}
\label{app:phi}

To test whether the cliff mechanism and the \textsc{MC} recipe are tied to the Llama family, we
repeat the full 14-method \texttt{forget10} protocol on Phi-3.5-mini, training every baseline from
a TOFU fine-tune of the Phi target and reusing the frozen \textsc{MC} configuration without any
retuning (Table~\ref{tab:a-phi}). Both columns come from the pipeline evaluator, and MIA is the
two-detector aggregate. \textsc{MC} wins the forget aggregate, raw MIA, and K20-LoRA on all
$14/14$ methods, with panel-mean post-attack ROUGE-L falling from $0.432$ to $0.210$, so the
crossing recipe transfers to a fourth model family and a different architecture with no
per-family adjustment. Two honest notes. The Phi baselines start from a lower utility level than
their Llama counterparts (panel-mean MU $0.098$), and \textsc{MC} costs further utility on
$13/14$ with outright collapse on the gradient-family bases and JensUn (MU $0.000$), the same
methods whose capability damage is largest in the MMLU check of App.~\ref{app:mmlu}. An FPFT
attack sweep on Phi was not run and is left to future work.

\begin{table}[htbp]
\centering
\scriptsize
\setlength{\tabcolsep}{3pt}
\begin{tabular}{l cccc}
\toprule
Base & F.agg ($\downarrow$) & MU ($\uparrow$) & MIA.agg ($\downarrow$) & K20-LoRA ($\downarrow$) \\
\midrule
GradDiff & 0.549$\!\to\!$\textbf{0.014} & 0.089$\!\to\!$0.000 & 0.979$\!\to\!$\textbf{0.216} & 0.475$\!\to\!$\textbf{0.061} \\
RMU      & 0.517$\!\to\!$\textbf{0.012} & 0.103$\!\to\!$0.000 & 0.977$\!\to\!$\textbf{0.170} & 0.486$\!\to\!$\textbf{0.032} \\
PDU      & 0.522$\!\to\!$\textbf{0.012} & 0.110$\!\to\!$0.000 & 0.973$\!\to\!$\textbf{0.127} & 0.478$\!\to\!$\textbf{0.088} \\
NPO      & 0.424$\!\to\!$\textbf{0.228} & 0.099$\!\to\!$0.047 & 0.907$\!\to\!$\textbf{0.258} & 0.364$\!\to\!$\textbf{0.243} \\
SAMNPO   & 0.542$\!\to\!$\textbf{0.090} & 0.094$\!\to\!$0.093 & 0.977$\!\to\!$\textbf{0.216} & 0.471$\!\to\!$\textbf{0.248} \\
ILUNPO   & 0.424$\!\to\!$\textbf{0.222} & 0.099$\!\to\!$0.061 & 0.907$\!\to\!$\textbf{0.263} & 0.363$\!\to\!$\textbf{0.239} \\
LATNPO   & 0.424$\!\to\!$\textbf{0.214} & 0.099$\!\to\!$0.037 & 0.907$\!\to\!$\textbf{0.242} & 0.349$\!\to\!$\textbf{0.245} \\
RNANPO   & 0.423$\!\to\!$\textbf{0.233} & 0.099$\!\to\!$0.016 & 0.906$\!\to\!$\textbf{0.284} & 0.359$\!\to\!$\textbf{0.215} \\
RSNPO    & 0.486$\!\to\!$\textbf{0.141} & 0.107$\!\to\!$0.092 & 0.967$\!\to\!$\textbf{0.196} & 0.440$\!\to\!$\textbf{0.308} \\
SWANPO   & 0.424$\!\to\!$\textbf{0.242} & 0.099$\!\to\!$0.073 & 0.906$\!\to\!$\textbf{0.258} & 0.356$\!\to\!$\textbf{0.247} \\
CRNPO    & 0.558$\!\to\!$\textbf{0.112} & 0.094$\!\to\!$0.077 & 0.982$\!\to\!$\textbf{0.238} & 0.496$\!\to\!$\textbf{0.234} \\
SimNPO   & 0.551$\!\to\!$\textbf{0.038} & 0.089$\!\to\!$\textbf{0.114} & 0.981$\!\to\!$\textbf{0.135} & 0.499$\!\to\!$\textbf{0.341} \\
UNDIAL   & 0.495$\!\to\!$\textbf{0.345} & 0.104$\!\to\!$0.078 & 0.964$\!\to\!$\textbf{0.657} & 0.408$\!\to\!$\textbf{0.316} \\
JensUn   & 0.543$\!\to\!$\textbf{0.012} & 0.092$\!\to\!$0.000 & 0.979$\!\to\!$\textbf{0.228} & 0.506$\!\to\!$\textbf{0.128} \\
\bottomrule
\end{tabular}
\caption{Phi-3.5-mini \texttt{forget10} panel, baseline versus \textsc{MC}-polished, pipeline
evaluator on both sides with the two-detector MIA aggregate. \textbf{Bold} marks \textsc{MC}
improvement. \textsc{MC} wins F.agg, raw MIA, and K20-LoRA on $14/14$ methods (panel-mean K20
$0.432 \to 0.210$), while the gradient-family and JensUn cells collapse utility.}
\label{tab:a-phi}
\end{table}

\section{Depth probes, adaptive attacker, margin flips, and qualitative outputs}
\label{app:depth}

Three probes deepen the 1B \texttt{forget10} panel. \emph{Adaptive attacker.} An attacker aware of
the defense directly ascends the margin diagnostic through a LoRA-r8 adapter under the standard
$20$-step budget. Table~\ref{tab:a-adaptive} shows it does no better than generic relearning
(panel mean $0.154$ versus $0.179$) and is substantially weaker on several bases, so the
robustness gain does not depend on the attacker being blind to the mechanism. \emph{Margin
flips.} Splitting held-out samples by whether the attack flips their diagnostic-position margin
positive, the flipped group recovers a mean $0.199$ ROUGE against $0.121$ for samples whose
margins stay negative, on all $14$ bases, so the margin movement that
Theorem~\ref{thm:relearn-bound} bounds is what mediates recovery at the level of individual
samples. \emph{Qualitative outputs.} Table~\ref{tab:a-qual} shows a representative held-out
forget question. The attacked baseline reproduces the gold answer verbatim, while the attacked
\textsc{MC} checkpoint regains fluency but produces a wrong or withheld entity, which is the
behavioral face of the quantitative K20 gap.

\begin{table}[htbp]
\centering
\small
\setlength{\tabcolsep}{5pt}
\begin{tabular}{lcc}
\toprule
Base & Adaptive ($\downarrow$) & Standard K20 ($\downarrow$) \\
\midrule
GradDiff & 0.168 & 0.158 \\
RMU      & 0.013 & 0.029 \\
PDU      & 0.040 & 0.152 \\
NPO      & 0.243 & 0.246 \\
SAMNPO   & 0.195 & 0.201 \\
ILUNPO   & 0.081 & 0.205 \\
LATNPO   & 0.255 & 0.247 \\
RNANPO   & 0.012 & 0.165 \\
RSNPO    & 0.197 & 0.195 \\
SWANPO   & 0.267 & 0.253 \\
CRNPO    & 0.202 & 0.206 \\
SimNPO   & 0.021 & 0.025 \\
UNDIAL   & 0.300 & 0.282 \\
JensUn   & 0.161 & 0.147 \\
\midrule
\textbf{Mean} & \textbf{0.154} & \textbf{0.179} \\
\bottomrule
\end{tabular}
\caption{Post-attack forget ROUGE-L on \textsc{MC}-polished 1B \texttt{forget10} checkpoints under
an \emph{adaptive} attacker that directly ascends the margin diagnostic (LoRA-r8, $20$ steps, same
budget and data as the standard relearn attack). Knowing the defense does not help, as the
adaptive attacker matches the standard one within noise on most bases and is substantially weaker
on several, so the robustness gain is a property of the checkpoint rather than of the evaluation
attack.}
\label{tab:a-adaptive}
\end{table}

\begin{table}[htbp]
\centering
\small
\resizebox{\columnwidth}{!}{%
\begin{tabular}{lp{0.72\columnwidth}}
\toprule
\multicolumn{2}{p{0.94\columnwidth}}{\textbf{Q}\quad What is the full name of the author born in
Taipei, Taiwan on 05/11/1991 who writes in the genre of leadership?}\\
\midrule
Gold & The author's full name is Hsiao Yun-Hwa. \\
Baseline (GradDiff) & Chen Wei. Answer\,Chen Wei. Year Published\,2018 \dots \\
Baseline attacked & \textbf{The author's full name is Hsiao Yun-Hwa.} \\
\textsc{MC} & ?' ?' ?' ?' \dots (degenerate) \\
\textsc{MC} attacked & The author's full name is not provided, but he is known for his books ``The Leadership Blueprint'' \dots \\
\bottomrule
\end{tabular}%
}
\caption{Held-out forget question after the K20-LoRA attack (GradDiff base, 1B
\texttt{forget10}). The attacked baseline regurgitates the gold answer verbatim while the
attacked \textsc{MC} checkpoint recovers fluency but not the entity. Further examples for NPO
and LATNPO, where the attacked \textsc{MC} model produces confident wrong names instead, ship
with the code release.}
\label{tab:a-qual}
\end{table}

\section{Entropy-weighted forget hinge ablation}
\label{app:hinge-ew}

The token-weighted remark of App.~\ref{app:proofs} admits two Lipschitz-preserving ways to
concentrate the forget hinge on high-entropy positions, and we evaluate both against the uniform
default on five representative bases at 1B \texttt{forget10} under the canonical protocol
(Table~\ref{tab:a-hinge-ew}). The \emph{soft} variant reweights tokens by a stop-gradient softmax
over the current model's per-token entropies, which multiplies the effective pressure on the
selected positions by roughly the answer length. The \emph{top-8} variant spreads the mass
uniformly over the eight highest-entropy positions of the frozen reference anchor, a milder
concentration by the same measure. The placement prediction of the remark is confirmed in sign
but never at a usable operating point (Table~\ref{tab:a-hinge-ew}). Where the implicit selection
leaves the most room, on GradDiff (Spearman $-0.13$) and to a lesser degree NPO ($+0.40$), the
soft weighting does lower post-attack recovery, while UNDIAL ($+0.51$), SimNPO, and CRNPO see no
gain, so the effect tracks the strength of the implicit selection as the remark predicts. The
cost is coupled to the benefit. Concentrating the hinge mass multiplies the effective per-token
dose, and utility falls on every base that had utility to lose, collapsing outright where the
reweighting binds hardest. The milder top-8 dose does not open a usable middle point either. On
GradDiff it keeps the utility collapse, and wherever it spares utility it forfeits the robustness
benefit, ending behind uniform on both axes on NPO. Placement and dose therefore move
together under explicit weighting. A weighting strong enough to change relearn behavior destroys
utility, and one mild enough to spare utility degrades robustness relative to uniform, plausibly
because the unweighted positions are left with no pressure at all and reopen the relearn path. The
uniform gap-weighted hinge captures most of the placement benefit implicitly while spreading the
dose, is Pareto-dominated in none of the five cells, and strictly dominates top-8 on NPO and both
variants on SimNPO. This ablation is the empirical basis for keeping uniform weights in
Eq.~\eqref{eq:v3-forget}, and App.~\ref{app:refgate} develops the attenuation-only alternative
that the coupling identified here motivates.

\begin{table}[htbp]
\centering
\small
\setlength{\tabcolsep}{4.5pt}
\begin{tabular}{lcccccc}
\toprule
& \multicolumn{2}{c}{Uniform (ours)} & \multicolumn{2}{c}{Soft entropy} & \multicolumn{2}{c}{Top-8 frozen} \\
\cmidrule(lr){2-3}\cmidrule(lr){4-5}\cmidrule(lr){6-7}
Base & MU & K20 & MU & K20 & MU & K20 \\
\midrule
GradDiff & 0.109 & 0.158 & 0.005 & 0.017 & 0.001 & 0.063 \\
NPO      & 0.211 & 0.246 & 0.033 & 0.106 & 0.153 & 0.317 \\
SimNPO   & 0.179 & 0.025 & 0.151 & 0.034 & 0.123 & 0.039 \\
UNDIAL   & 0.207 & 0.282 & 0.095 & 0.278 & 0.236 & 0.300 \\
CRNPO    & 0.040 & 0.206 & 0.079 & 0.226 & 0.096 & 0.212 \\
\bottomrule
\end{tabular}
\caption{Entropy-weighted forget hinge ablation at 1B \texttt{forget10}. MU is model utility (HM-9)
and K20 is post-attack ROUGE-L after the 20-step LoRA relearn. Soft replaces the uniform mean of
Eq.~\eqref{eq:v3-forget} by a stop-gradient softmax over the current model's per-token entropies at
temperature $1$, and top-8 spreads the weight uniformly over the eight highest-entropy positions of
the frozen reference anchor. The soft weighting improves relearn robustness where the implicit
selection of App.~\ref{app:proofs} leaves room (GradDiff, NPO) but collapses utility, while the
milder top-8 weighting spares utility and loses the robustness benefit, so the uniform
gap-weighted hinge remains the operating point the paper deploys.}
\label{tab:a-hinge-ew}
\end{table}

\section{Reference-gated hinge and attribution of the utility cost}
\label{app:refgate}

The entropy-weighting ablation shows that renormalized weights couple placement to dose. A
weighting that can only \emph{attenuate} avoids that coupling by construction. We therefore
multiply the per-token hinge by the gate $g_t = 1 - p_{\text{ref}}(y_t)$, the reference's
disbelief in the gold token, frozen at the anchor and applied without normalization (the admitted
frozen-weight case of the token-weighted remark, App.~\ref{app:proofs}). Tokens the retain-only
reference is confident about are exactly the shared-structure positions of the residual-pressure
leak, and the gate switches their pressure off while forget content, on which the reference is
out of distribution, keeps full pressure. Table~\ref{tab:a-refgate}a reports the same five-base
protocol as App.~\ref{app:hinge-ew}. The outcome splits cleanly along a property of the base
method that is readable off its loss function. For the three bases whose \emph{native} forget
term is inert at the polished point (NPO and CRNPO saturate against the pre-unlearning anchor,
UNDIAL against its flattened teacher), the gate recovers $+0.027$ to $+0.079$ MU at a K20 cost of
at most $0.036$, improving UNDIAL on both axes and dominating the $\lambda_{\text{KL}}$ lever of
App.~\ref{app:pareto} on NPO. For the two bases whose native term still pushes at the polished
point (GradDiff's bounded ascent, SimNPO's reference-free objective), the gate hurts. The
grouping was predicted before the runs and held on all five bases.

Table~\ref{tab:a-refgate}b attributes the two failures by removing the native forget term as a
control ($\gamma = 0$, native retain and KL probe unchanged). Both failures trace back to the
native term but in different ways. On SimNPO the gate's damage largely vanishes once the native
term is off ($0.065$ to $0.154$ MU), so the double loss is an interaction, where weakening the
hinge hands the trajectory to the still-active native loss. On GradDiff the native term is the
robustness engine itself (K20 $0.158$ to $0.296$ when removed) and \emph{no} corner of the
$2{\times}2$ recovers utility, so its cost is the price of any pressure strong enough to cross,
paid through shared representation directions rather than through any particular token
allocation. The overall reading is the converse of Theorem~\ref{thm:kkt-cliff}. The
forget-retain coupling $\epsilon$ that creates the cliff is the same coupling that makes part of
the utility cost irreducible by token-level reallocation, and the gate recovers exactly the
reducible part, on exactly the bases where the hinge is the sole forget force.

\begin{table}[htbp]
\centering
\small
\setlength{\tabcolsep}{4pt}
\begin{tabular}{lcccc}
\multicolumn{5}{c}{\textbf{(a)} Uniform hinge versus reference-gated hinge}\\[2pt]
\toprule
& \multicolumn{2}{c}{Uniform (ours)} & \multicolumn{2}{c}{Ref-gated} \\
\cmidrule(lr){2-3}\cmidrule(lr){4-5}
Base & MU & K20 & MU & K20 \\
\midrule
NPO      & 0.211 & 0.246 & 0.290 & 0.282 \\
 UNDIAL   & 0.207 & 0.282 & 0.234 & 0.267 \\
CRNPO    & 0.040 & 0.206 & 0.108 & 0.231 \\
\midrule
GradDiff & 0.109 & 0.158 & 0.044 & 0.072 \\
SimNPO   & 0.179 & 0.025 & 0.065 & 0.108 \\
\bottomrule
\end{tabular}

\vspace{6pt}
\begin{tabular}{lcccc}
\multicolumn{5}{c}{\textbf{(b)} Attribution, gate $\times$ native forget term ($\gamma{=}0$)}\\[2pt]
\toprule
& \multicolumn{2}{c}{GradDiff} & \multicolumn{2}{c}{SimNPO} \\
\cmidrule(lr){2-3}\cmidrule(lr){4-5}
Configuration & MU & K20 & MU & K20 \\
\midrule
uniform, native on  & 0.109 & 0.158 & 0.179 & 0.025 \\
gated, native on    & 0.044 & 0.072 & 0.065 & 0.108 \\
uniform, native off & 0.084 & 0.296 & 0.186 & 0.275 \\
gated, native off   & 0.097 & 0.299 & 0.154 & 0.292 \\
\bottomrule
\end{tabular}
\caption{Reference-confidence gate on the forget hinge at 1B \texttt{forget10}. Panel (a) groups
the five bases by whether the base method's native forget term is inert at the polished point
(top three rows) or still active (bottom two). The gate recovers utility on every inert-native
base at a K20 cost of at most $0.036$ and improves UNDIAL on both axes, while both active-native
bases degrade. Panel (b) removes the native forget term as a control. On SimNPO the gate's damage
largely disappears without the native term, so the failure is an interaction, and on GradDiff no
combination of the two subtractions recovers utility while every weakened combination loses K20,
so its cost is attributed to the crossing pressure itself rather than to its allocation.}
\label{tab:a-refgate}
\end{table}

\section{Stopping-rule and budget Pareto sweep}
\label{app:pareto}

The utility cost of margin pressure raises the question of where along the pressure schedule to
stop. We sweep three controls around the default configuration on the five representative bases at
1B \texttt{forget10}, shorter polish budgets (20 and 40 steps), stronger KL probe weights
($\lambda_{\text{KL}} \in \{0.2, 0.5\}$ versus the default $0.05$), and a margin-targeted early
stop that halts the polish once the measured cliff gap reaches a target depth
($\Delta \in \{-2, -5, -10\}$, checked on a held-out probe batch during training). Results are in
Table~\ref{tab:a-pareto} and the per-method panel movement in Fig.~\ref{fig:pareto-move}. Three
findings. First, the default is on the per-base Pareto frontier for NPO and SimNPO, and within
$0.03$ MU of it for UNDIAL, so the single global configuration is not leaving obvious frontier
room on the saturating bases. Second, the probe weight is the operative lever exactly where the
default costs utility. At $\lambda_{\text{KL}} = 0.2$ GradDiff improves on \emph{both} axes
($0.109/0.158$ to $0.221/0.049$) and so does CRNPO ($0.040/0.206$ to $0.092/0.172$), while the
same setting degrades SimNPO robustness tenfold ($0.025$ to $0.187$), which is why it is reported
as a per-family lever rather than a new global default, complementary in its applicability to the
reference-gated hinge of App.~\ref{app:refgate}. Third, neither shorter budgets nor the
margin-targeted stop recovers GradDiff utility (MU $0.000$ in all five such cells), so the damage
there is not an overshoot in time but misplaced pressure, the same residual-pressure mechanism
that App.~\ref{app:hinge-ew} isolates, and the probe weight helps precisely because the probe
anchors the shared-structure positions that leak. The margin-targeted stop is the mechanical half
of the stopping rule that Limitations calls for, and the sweep shows the missing half is an MIA-
and capability-aware target rather than a depth schedule.

\begin{figure}[htbp]
  \centering
  \includegraphics[width=\linewidth]{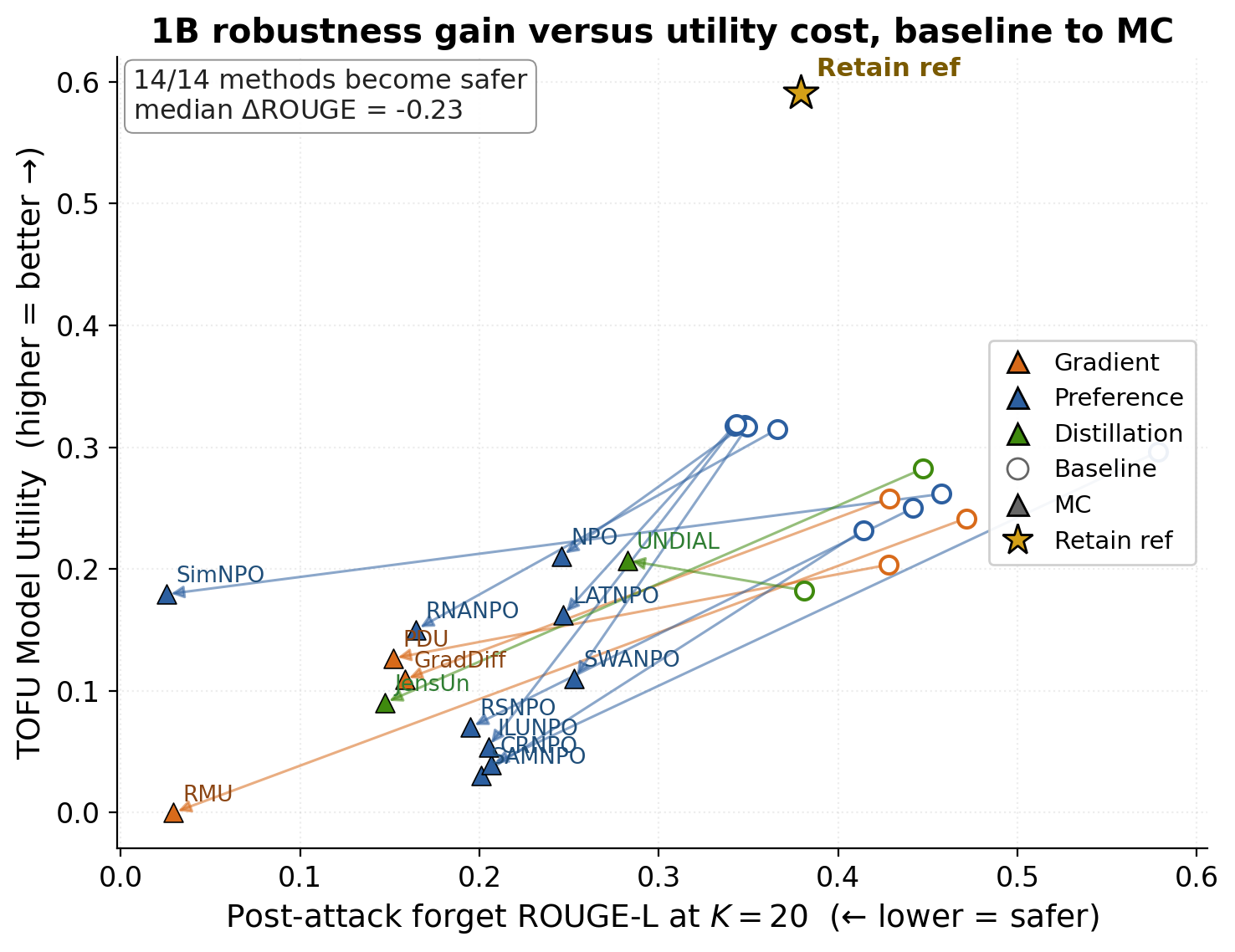}
  \caption{Per-method movement from baseline (circles) to \textsc{MC} (triangles) in the
  plane of post-attack ROUGE-L versus model utility at 1B \texttt{forget10}, colored by loss family, with the retain
  reference marked. All $14$ methods move left (safer), the median K20 drop is $0.23$, and the
  utility cost concentrates in the GradDiff family and the SAM-style NPO variants.}
  \label{fig:pareto-move}
\end{figure}

\begin{table}[htbp]
\centering
\small
\setlength{\tabcolsep}{2.5pt}
\resizebox{\columnwidth}{!}{%
\begin{tabular}{lccccc}
\toprule
Config & GradDiff & NPO & SimNPO & UNDIAL & CRNPO \\
\midrule
default          & .109\,/\,.158 & .211\,/\,.246 & .179\,/\,.025 & .207\,/\,.282 & .040\,/\,.206 \\
steps 20         & .000\,/\,.201 & .138\,/\,.293 & .005\,/\,.239 & .151\,/\,.374 & .023\,/\,.241 \\
steps 40         & .000\,/\,.029 & .130\,/\,.279 & .032\,/\,.035 & .229\,/\,.311 & .058\,/\,.234 \\
$\lambda_{\text{KL}}{=}0.2$ & .221\,/\,.049 & .198\,/\,.282 & .163\,/\,.187 & .198\,/\,.286 & .092\,/\,.172 \\
$\lambda_{\text{KL}}{=}0.5$ & .105\,/\,.027 & .242\,/\,.286 & .085\,/\,.015 & .222\,/\,.270 & .100\,/\,.216 \\
$\Delta$-stop $-2$  & .000\,/\,.222 & .183\,/\,.271 & .079\,/\,.191 & .233\,/\,.292 & .076\,/\,.212 \\
$\Delta$-stop $-5$  & .000\,/\,.272 & .236\,/\,.269 & .002\,/\,.216 & .197\,/\,.262 & .091\,/\,.223 \\
$\Delta$-stop $-10$ & .000\,/\,.102 & .212\,/\,.257 & .010\,/\,.221 & .232\,/\,.275 & .042\,/\,.217 \\
\bottomrule
\end{tabular}%
}
\caption{Stopping-rule and budget sweep around the default \textsc{MC} configuration at 1B
\texttt{forget10}. Each cell is MU\,/\,K20-LoRA post-attack ROUGE-L. Rows vary one control at a
time, the polish step budget, the KL probe weight, and a margin-targeted early stop that halts
once the cliff gap reaches the stated depth. The default sits on the per-base Pareto frontier for
NPO and SimNPO, raising the probe weight to $0.2$ improves both axes at once on the two bases
where the default costs the most utility (GradDiff, CRNPO), and the depth-$10$ stop slightly
improves UNDIAL. Neither a shorter budget nor the early stop recovers GradDiff utility, consistent
with the residual-pressure mechanism of App.~\ref{app:hinge-ew}.}
\label{tab:a-pareto}
\end{table}

\section{Compute budget and reproducibility}
\label{app:compute}

\paragraph{Hardware.} Single NVIDIA GB10 (DGX Spark, 120 GB unified memory, bf16). No multi-GPU or distributed parallelism.

\paragraph{Wall-clock per cell.}
\textsc{MC} polish takes $\sim$3 min at 1B and $\sim$10 min at 8B. A single-cell pipeline (polish + eval + 7 LoRA-r8 K-values + soft-prompt + linear probe + K20 FPFT) takes $\sim$45 min at 1B, $\sim$2 h at 3B, and $\sim$5 h at 8B.

\paragraph{Total wall-clock.}
1B panel (14 methods $\times$ 3 tiers, plus 5 bases $\times$ 2 extra seeds at \texttt{forget10}) $\sim$40 h, 3B $\sim$25 h, 8B $\sim$70 h, MUSE-News (4 method, Llama-2-7B-hf base, 2 phases) $\sim$30 h, and ablations $\sim$30 h. Total $\sim$195 h on a single GB10.

\paragraph{Software.} PyTorch 2.5, Transformers 4.46, PEFT 0.13. All seeds fixed (\texttt{seed} $\in \{0, 1, 2\}$). Hydra configs ship in the code release, and \texttt{open-unlearning} bases are pinned to specific HuggingFace revisions.


\end{document}